\documentclass[nohyperref]{article}

\usepackage{microtype}
\usepackage{graphicx}
\usepackage{subfigure}
\usepackage{booktabs} % for professional tables
\usepackage{color}
\usepackage{multirow}
\usepackage{wrapfig}

\usepackage{algorithm}
\usepackage{algorithmic}

\usepackage{hyperref}

\usepackage[accepted]{icml2023}

\usepackage{amsmath}
\usepackage{amssymb}
\usepackage{mathtools}
\usepackage{amsthm}

\usepackage[capitalize,noabbrev]{cleveref}

\theoremstyle{plain}
\newtheorem{theorem}{Theorem}[section]
\newtheorem{proposition}[theorem]{Proposition}

\theoremstyle{definition}

\theoremstyle{remark}
\newtheorem{remark}[theorem]{Remark}

\usepackage{color}
\usepackage{enumerate}
\usepackage[shortlabels]{enumitem}

\usepackage[textsize=tiny]{todonotes}

\icmltitlerunning{Improving Adversarial Robustness Through the Contrastive-Guided Diffusion Process}

\begin{document}

\twocolumn[
% \icmltitle{TSVGD: Tabular Stein Variational Gradient Descent \\for Failure Case Simulation}
\icmltitle{Improving Adversarial Robustness Through the \\ Contrastive-Guided Diffusion Process}

% It is OKAY to include author information, even for blind
% submissions: the style file will automatically remove it for you
% unless you've provided the [accepted] option to the icml2023
% package.

% List of affiliations: The first argument should be a (short)
% identifier you will use later to specify author affiliations
% Academic affiliations should list Department, University, City, Region, Country
% Industry affiliations should list Company, City, Region, Country

% You can specify symbols, otherwise they are numbered in order.
% Ideally, you should not use this facility. Affiliations will be numbered
% in order of appearance and this is the preferred way.
\icmlsetsymbol{equal}{*}

\begin{icmlauthorlist}
\icmlauthor{Yidong Ouyang}{y}
\icmlauthor{Liyan Xie}{y}
\icmlauthor{Guang Cheng}{sch}
% % \icmlauthor{Firstname3 Lastname3}{comp}
% % \icmlauthor{Firstname4 Lastname4}{sch}
% % \icmlauthor{Firstname5 Lastname5}{yyy}
% % \icmlauthor{Firstname6 Lastname6}{sch,yyy,comp}
% % \icmlauthor{Firstname7 Lastname7}{comp}
% %\icmlauthor{}{sch}
% % \icmlauthor{Firstname8 Lastname8}{sch}
% % \icmlauthor{Firstname8 Lastname8}{yyy,comp}
% %\icmlauthor{}{sch}
% %\icmlauthor{}{sch}
\end{icmlauthorlist}

\icmlaffiliation{y}{School of Data Science, The Chinese University of Hong Kong, Shenzhen, China}
\icmlaffiliation{sch}{Department of Statistics, University of California, Los Angeles, USA}
% \icmlaffiliation{comp}{Company Name, Location, Country}
% \icmlaffiliation{sch}{School of ZZZ, Institute of WWW, Location, Country}

% \icmlcorrespondingauthor{Yidong Ouyang}{yidongouyang@link.cuhk.edu.cn}
\icmlcorrespondingauthor{Liyan Xie}{xieliyan@cuhk.edu.cn}
% \icmlcorrespondingauthor{Guang Cheng}{guangcheng@ucla.edu}

% You may provide any keywords that you
% find helpful for describing your paper; these are used to populate
% the "keywords" metadata in the PDF but will not be shown in the document
\icmlkeywords{Machine Learning, ICML}

\vskip 0.3in
]

% this must go after the closing bracket ] following \twocolumn[ ...

% This command actually creates the footnote in the first column
% listing the affiliations and the copyright notice.
% The command takes one argument, which is text to display at the start of the footnote.
% The \icmlEqualContribution command is standard text for equal contribution.
% Remove it (just {}) if you do not need this facility.

%\printAffiliationsAndNotice{} % leave blank if no need to mention equal contribution
\printAffiliationsAndNotice{} % otherwise use the standard text.

\begin{abstract}
Synthetic data generation has become an emerging tool to help improve the adversarial robustness in classification tasks, since robust learning requires a significantly larger amount of training samples compared with standard classification. Among various deep generative models, the diffusion model has been shown to produce high-quality synthetic images and has achieved good performance in improving the adversarial robustness. However, diffusion-type methods are generally slower in data generation as compared with other generative models. Although different acceleration techniques have been proposed recently, it is also of great importance to study how to improve the sample efficiency of synthetic data for the downstream task. In this paper, we first analyze the optimality condition of synthetic distribution for achieving improved robust accuracy. We show that enhancing the distinguishability among the generated data is critical for improving adversarial robustness. Thus, we propose the Contrastive-Guided Diffusion Process (Contrastive-DP), which incorporates the contrastive loss to guide the diffusion model in data generation. We validate our theoretical results using simulations and demonstrate the good performance of Contrastive-DP on image datasets.
\end{abstract}

\section{Introduction}\label{sec:intro}

The success of most deep learning methods relies heavily on a massive amount of training data, which can be expensive to acquire in practice. For example, in applications like autonomous driving \citep{OKelly2018ScalableEA} and medical diagnosis \citep{Das2022ConditionalSD}, the number of rare scenes is usually very limited in the training dataset. Moreover, the number of labeled data for supervised learning could also be limited in some applications since it may be expensive to label the data. These challenges call for methods that can produce additional data that are easy to generate and can help improve downstream task performance. Synthetic data generation based on deep generative models has shown promising performance recently to tackle these challenges \citep{sehwag2021improving,gowal2021improving,Das2022ConditionalSD}.

In synthetic data generation, one aims to learn a {\it synthetic distribution} (from which we generate synthetic data) that is close to the true date-generating distribution and, most importantly, can help improve the downstream task performance. Synthetic data generation is highly related to generative models. Among various kinds of generative models, the score-based model and diffusion type models have gained much success in image generation recently \citep{song2019generative, Song2021ScoreBasedGM,song2020sliced,song2020improved,sohl2015deep,nichol2021improved,Bao2022AnalyticDPMAA,rombach2022high,Nie2022DiffusionMF,Sun2022PointDPDP}. As validated in image datasets, the prototype of diffusion models, the Denoising Diffusion Probabilistic Model (DDPM) \citep{Ho2020DDPM}, and many variants can generate high-quality images as compared with classical generative models such as generative adversarial networks \citep{Dhariwal2021DiffusionMB}. 

This paper mainly focuses on the adversarial robust classification of image data, which typically requires more training data than standard classification tasks \cite{Carmon2019UnlabeledDI}. In \cite{gowal2021improving}, 100M high-quality synthetic images are generated by DDPM and achieve the state-of-the-art performance on adversarial robustness on the CIFAR-10 dataset, which demonstrates the effectiveness of diffusion models in improving adversarial robustness. However, a major drawback of diffusion-type methods is the slow computational speed. More specifically, DDPM is usually 1000 times slower than GAN \citep{Song2021DenoisingDI}, and this drawback is more serious when generating a large number of samples, e.g., it takes more than 99 GPU days \footnote{Running on a RTX 4x2080Ti GPU cluster.} for generating 100M image data according to \cite{gowal2021improving}. Moreover, the computational cost will increase dramatically when the resolution of images increases, which inspires a plentiful of works studying how to accelerate the diffusion models \citep{Song2021DenoisingDI,Watson2022LearningFS,ma2022accelerating, Salimans2022ProgressiveDF, Bao2022AnalyticDPMAA,cao2022survey,yang2022diffusion}. In this paper, we aim to study the aforementioned problem from a different perspective -- ``how to generate effective synthetic data that are most helpful for the downstream task?'' We analyze the {\it optimal synthetic distribution} for the downstream tasks to improve the sample efficiency of the generative model. 

We first study the theoretical insights for finding the optimal synthetic distributions for achieving adversarial robustness. Following the setting considered in \cite{Carmon2019UnlabeledDI}, we introduce a family of synthetic distributions controlled by the distinguishability of the representation from different classes. Our theoretical results show that the more distinguishable the representation is for the synthetic data, the higher the classification accuracy we will get. Motivated by the theoretical insights, we propose the Contrastive-Guided Diffusion Process (Contrastive-DP) for efficient synthetic data generation, incorporating the gradient of the contrastive learning loss \citep{Oord2018RepresentationLW, Chuang2020DebiasedCL, Robinson2021ContrastiveLW} into the diffusion process. We conduct comprehensive simulations and experiments on real image datasets to demonstrate the effectiveness of the proposed Contrastive-DP method.

The remainder of the paper is organized as follows. Section~\ref{sec:formulation} presents the problem formulation and preliminaries on diffusion models. Section~\ref{sec:theory} contains the theoretical insights of optimal synthetic distribution under the Gaussian setting. Motivated by the theoretical insights, Section~\ref{sec:contr} proposes a new type of synthetic data generation procedure that combines contrastive learning with diffusion models. Finally, Section~\ref{sec:numerical} conducts extensive numerical experiments to validate the good performance of the proposed generation method on simulation and image datasets. 

\section{Problem Formulation and Preliminaries} \label{sec:formulation}

We first give a brief overview of adversarial robust classification, which is our main focus in this work. It is worth mentioning that the whole framework can be applied to other downstream tasks in general. Denote the feature space as $\mathcal{X}$, the corresponding label space as $\mathcal{Y}$, and the true (joint) data distribution as $\mathcal{D}=\mathcal{D}_{\mathcal{X}\times\mathcal{Y}}$. 

Assume we have labeled training data $\mathcal{D}_{\text {train}}:=\{(\boldsymbol{x}_i,y_i)\}_{i=1}^n$ sampled from $\mathcal{D}$. We aim to learn a robust classifier $f_{\boldsymbol{\theta}}: \mathcal{X} \mapsto \mathcal{Y}$, parameterized by a learnable $\boldsymbol\theta$, that can achieve the minimum adversarial loss:
\begin{equation}\label{eq:robust_classifier}
\min_{\boldsymbol{\theta}} \mathcal L_{adv}(\boldsymbol{\theta}):= \mathbb{E}_{(\boldsymbol{x},y)\sim \mathcal{D}}
\Big(\max_{\boldsymbol{\delta}\in\Delta}\ell(\boldsymbol{x}+\boldsymbol{\delta},y,\boldsymbol{\theta})
\Big),
\end{equation}
where $\ell(\boldsymbol{x},y,\boldsymbol{\theta}) = \mathbf{1}\{y \neq f_\theta(\boldsymbol{x})\}$ is the 0-1 loss function, $\mathbf{1}\{\cdot\}$ is the indicator function, and $\Delta=\{\boldsymbol{\delta}:\left\Vert \boldsymbol{\delta} \right\Vert_\infty\leq\epsilon\}$ is the adversarial set defined using $\ell_\infty$-norm. Intuitively, the solution to problem \eqref{eq:robust_classifier} is a robust classifier that minimizes the worst-case loss within an $\epsilon$-neighborhood of the input features.

In the canonical form of adversarial training, we train the robust classifier $f_{\boldsymbol\theta}$ on the training set $\mathcal{D}_{\text {train}}:=\{(\boldsymbol{x}_i,y_i)\}_{i=1}^n$ by solving the following sample average approximation of the population loss in \eqref{eq:robust_classifier}:
\begin{equation}\label{eq:adv_train}
 \min_{\boldsymbol{\theta}} \widehat{ \mathcal L}_{adv}(\boldsymbol{\theta}) :=\frac1n \sum_{i=1}^n \max_{\boldsymbol{\delta}_i\in\Delta} \ell(\boldsymbol{x}_i+\boldsymbol{\delta}_i,y_i,\boldsymbol{\theta}).
\end{equation}

\subsection{Adversarial Training Using Synthetic Data}\label{sec:2.1}

As shown in \cite{Carmon2019UnlabeledDI}, adversarial training requires more training data in order to achieve the desired accuracy. 
Synthetic data generation has been used as a method to artificially increase the size of the training set by generating a sufficient amount of additional data, thus helping improve the learning algorithm's performance \citep{gowal2021improving}. 

The mainstream generation procedures can be categorized into two types: {\it unconditional} and {\it conditional} generation. In the unconditional generation, we first generate the features ($\boldsymbol{x}$) and then assign pseudo labels to them. In the conditional generation, we generate the features conditioned on the desired label. %The former pipeline requires using the self-training paradigm, which is harder to analyze.
Our analysis is mainly based on the former paradigm, which can be easily generalized to the conditional generation procedure, and our proposed algorithm is also flexible enough for both pipelines.

Denote the distribution of the generated features as $\widetilde{\mathcal{D}}_{\mathcal{X}}$ and the generated synthetic data as $\mathcal{D}_{\text {syn}}:=\{(\tilde{\boldsymbol{x}}_i,\tilde{y}_i)\}_{i=1}^{\tilde{n}}$. Here the feature values $\tilde{\boldsymbol{x}}_i$ are generated from the synthetic distribution $\widetilde{\mathcal{D}}_{\mathcal{X}}$, and $\tilde{y}_i$ are pseudo labels assigned by a classifier learned on the training data $\mathcal{D}_{\text {train}}$. Combining the synthetic and real data, we will learn the robust classifier using a larger training set $\mathcal{D}_{\text{all}}:=\mathcal{D}_{\text {train}}\cup\mathcal{D}_{\text {syn}}$ which now contains $n+\tilde{n}$ samples:
\begin{equation}\label{eq:adv_train_syn}
\begin{split}
 \min_{\boldsymbol{\theta}}\left\{ \right.&\eta \left(\frac1n \sum_{i=1}^n \max_{\boldsymbol{\delta}_i\in\Delta} \ell(\boldsymbol{x}_i+\boldsymbol{\delta}_i,y_i,\boldsymbol{\theta})\right) + \\
 &\left.(1-\eta) \left(\frac{1}{ \tilde n} \sum_{i=1}^n \max_{\boldsymbol{\delta}_i\in\Delta} \ell(\tilde{\boldsymbol{x}}_i+\boldsymbol{\delta}_i,\tilde y_i,\boldsymbol{\theta})\right) \right\},
 \end{split}
\end{equation}
where $\eta\in(0,1)$ is a parameter controlling the weights of synthetic data.

\subsection{Diffusion Model for Data Generation}\label{sec:2.2}

We build our proposed generation procedure based on the Denoising Diffusion Probabilistic Model (DDPM) \citep{Ho2020DDPM} and its accelerated variant Denoising Diffusion Implicit Model (DDIM) \citep{Song2021DenoisingDI}. In the following, we briefly review the key components of DDPM. 
%The can accelerate the generating speed of DDPM by 10 to 100 times, with only a slight deterioration of the sample quality. Therefore, we adopt DDIM as the backbone to generate synthetic data. 

The core of DDPM is a forward Markov chain with Gaussian transitions distributions $q(\boldsymbol{x}_t|\boldsymbol{x}_{t-1})$ to inject Gaussian noise to the original data distribution $q(\boldsymbol{x}_0)$. More specifically, \cite{Ho2020DDPM} model the forward Gaussian transition as:
\[
q\left(\boldsymbol{x}_{t}| \boldsymbol{x}_{t-1}\right):=\mathcal{N}\left(\sqrt{\alpha_{t}} \boldsymbol{x}_{t-1},\left(1-\alpha_{t}\right) \mathbb{I}\right),
\]
where $\alpha_t, t = 1,2,\ldots,T$ is a decreasing sequence to control the variance of injected noise, and $\mathbb{I}$ is the identity covariance matrix. The joint likelihood for the above Markov chain can be written as $q\left(\boldsymbol{x}_{0:T}\right)=q\left(\boldsymbol{x}_{0}\right) \prod_{t=1}^T q\left(\boldsymbol{x}_{t} | \boldsymbol{x}_{t-1}\right)$. 
DDPM then propose to use $p_\theta\left(\boldsymbol{x}_{0:T}\right)=p_\theta\left(\boldsymbol{x}_{T}\right) \prod_{t=1}^T p_\theta\left(\boldsymbol{x}_{t-1}| \boldsymbol{x}_{t}\right)$
to model the reverse process, where $p_\theta(\boldsymbol{x}_{t-1}|\boldsymbol{x}_t)$ is parameterized using a neural network. 
The training objective is to minimize the Kullback–Leibler (KL) divergence between the forward and reverse process: $\operatorname{D_{KL}}(q\left(\boldsymbol{x}_{0:T}\right), p_\theta\left(\boldsymbol{x}_{0:T}\right))$, which can be simplified as:
$$
\min_\theta \mathbb{E}_{t, \boldsymbol{x}_{0}, \epsilon}\left[\left\|\boldsymbol{\epsilon}-\boldsymbol{\epsilon}_{\theta}\left(\sqrt{\bar{\alpha}_{t}} \boldsymbol{x}_{0}+\sqrt{1-\bar{\alpha}_{t}} \boldsymbol{\epsilon}, t\right)\right\|^{2}\right],
$$
where the expectation is taken with respect to $\boldsymbol{x}_{0}\sim q(\boldsymbol{x}_{0})$, $\boldsymbol{\epsilon} \sim \mathcal{N}(\boldsymbol{0},\mathbb{I})$, and $t$ uniformly distributed in $\{1,\ldots,T\}$. Here $\bar{\alpha}_{t}=\prod_{s=1}^{t} \alpha_{s}$ and $\boldsymbol{\epsilon}_{{\theta}}(\boldsymbol{x},t)$ denotes the neural network parameterized by ${\theta}$. We refer to \cite{Ho2020DDPM} for the detailed derivation and learning algorithms.

After learning the time-reversed process parameterized by $\theta$, the original generation process in \cite{Ho2020DDPM} is a time-reversed Markov chain as follows:
$$
\begin{aligned}
  \boldsymbol{x}_{t-1}=\frac{1}{\sqrt{\alpha_{t}}} &\left(\boldsymbol{x}_{t}-\frac{1-\alpha_{t}}{\sqrt{1-\bar{\alpha}_{t}}} \boldsymbol{\epsilon}_{{\theta}}\left(\boldsymbol{x}_{t}, t\right)\right)+\sigma_{t} \boldsymbol{z}_t,  
\end{aligned}
$$
starting from $\boldsymbol{x}_T \sim \mathcal{N}(\boldsymbol{0},\mathbb{I})$ and calculating for $t=T,T-1,\ldots,1$. The output value $\boldsymbol{x}_0$ is the generated synthetic data. Here $\boldsymbol{z}_t\sim \mathcal{N}(\boldsymbol{0},\mathbb{I})$ if $t>1$ and $\boldsymbol{z}_t=\boldsymbol{0}$ if $t=1$. 
DDIM \citep{Song2021DenoisingDI} speeds up the above procedure by generalizing the diffusion process to a non-Markovian process, leading to a sampling trajectory much shorter than $T$. DDIM carefully designs the forward transition such that $q\left(\boldsymbol{x}_{t}| \boldsymbol{x}_0\right)=\mathcal{N}\left(\sqrt{\alpha_{t}} \boldsymbol{x}_0,\left(1-\alpha_{t}\right) \mathbb{I}\right)$ for all $t= 1, \ldots,T$. The great advantage of DDIM is that it admits the same training objective as DDPM, which means we can adapt the pre-trained model of DDPM and accelerate the sampling process without additional cost. The key sample-generating step in DDIM is as follows:
\begin{equation}\label{eq:ddpm}
\begin{split}
\boldsymbol{x}_{t-1}=&\sqrt{\alpha_{t-1}} \underbrace{\left(\frac{\boldsymbol{x}_{t}-\sqrt{1-\alpha_{t}} \boldsymbol{\epsilon}_{\theta}\left(\boldsymbol{x}_{t}, t\right)}{\sqrt{\alpha_{t}}}\right)}_{\text {predicted } \boldsymbol{x}_0}\\
&+\underbrace{\sqrt{1-\alpha_{t-1}} \cdot \boldsymbol{\epsilon}_{\theta}\left(\boldsymbol{x}_{t},t\right)}_{\text {pointing to } \boldsymbol{x}_t},
\end{split}
\end{equation}
in which we can generate $\boldsymbol{x}_{t-1}$ using $\boldsymbol{x}_{t}$ and $\boldsymbol{x}_{0}$. Also, the generating process becomes deterministic.

% \begin{algorithm}[!ht]
% 	\caption{Training}
% 	\begin{algorithmic}[1]
% 	\REPEAT
% 		\STATE$\boldsymbol{x}_{0} \sim q\left(\boldsymbol{x}_{0}\right)$ \\
% \STATE $t \sim$ Uniform $(\{1, \ldots, T\})$ \\
% \STATE $\boldsymbol{\epsilon} \sim \mathcal{N}(\mathbf{0}, \mathbb{I})$ \\
% \STATE Take gradient descent step on $\nabla_{\boldsymbol{\theta}}\left\|\boldsymbol{\epsilon}-\boldsymbol{\epsilon}_{\boldsymbol{\theta}}\left(\sqrt{\bar{\alpha}_{t}} \boldsymbol{x}_{0}+\sqrt{1-\bar{\alpha}_{t}} \boldsymbol{\epsilon}, t\right)\right\|^{2}$ \\
%   \UNTIL{converged}.
% 	\end{algorithmic}
% 	\label{alg1}
% \end{algorithm}

% 	\caption{Sampling}
% 	\begin{algorithmic}[1]
% 		\STATE $\boldsymbol{x}_{T} \sim \mathcal{N}(\mathbf{0}, \mathbb{I})$
% 		\STATE $t=T$
%     \WHILE{$t\neq1$}
%       \STATE $\mathbf{z} \sim \mathcal{N}(\mathbf{0}, \mathbb{I})$
%       \STATE $\boldsymbol{x}_{t-1}=\frac{1}{\sqrt{\alpha_{t}}}\left(\boldsymbol{x}_{t}-\frac{1-\alpha_{t}}{\sqrt{1-\bar{\alpha}_{t}}} \boldsymbol{\epsilon}_{\boldsymbol{\theta}}\left(\boldsymbol{x}_{t}, t\right)\right)+\sigma_{t} \mathbf{z}$
%       \STATE $t=t-1$
%     \ENDWHILE
% \RETURN $\boldsymbol{x}_{0}$
% 	\end{algorithmic}
% 	\label{alg1}
% \end{algorithm}

\section{Theoretical Insights: Optimal Synthetic Distribution}\label{sec:theory}

In this section, we consider a concrete distributional model as used in \cite{Carmon2019UnlabeledDI,Schmidt2018AdversariallyRG}, and demonstrate the advantage of refining the synthetic data generation process -- using the optimal distribution for synthetic data generation can help reduce the sample complexity needed for robust classification. This provides theoretical insights and motivates the proposed Contrastive-DP method to be introduced in Section~\ref{sec:contr}.

% \vspace{-0.1in}
\subsection{Theoretical Setup}\label{sec:theo_setup}
%\paragraph{True data generating process.}
%\paragraph{Learning a linear classifier.}
%\paragraph{Self-learning .}\label{self-learning}

Consider a binary classification task where $\mathcal{X}=\mathbb{R}^d, \mathcal{Y}=\{-1,1\}$. The true data distribution $\mathcal{D}$ is specified as follows. The marginal distribution for label $y$ is uniform in $\mathcal{Y}$, and the conditional distribution of features is $\boldsymbol{x}|y\sim \mathcal{N}(y\boldsymbol{\mu},\sigma^2 \mathbb{I}_d)$, where $\boldsymbol{\mu}\in \mathbb{R}^d$ is non-zero, and $\mathbb{I}_d$ is the $d$ dimensional identity covariance matrix. 
Thus the marginal feature distribution $\mathcal{D}_{\mathcal{X}}$ is a Gaussian mixture, for convenience we denote as $0.5\mathcal{N}({\boldsymbol{\mu}},\sigma^2 \mathbb{I})+0.5 \mathcal{N}(-{\boldsymbol{\mu}},\sigma^2 \mathbb{I})$. Suppose we also generate a set of synthetic data from another synthetic distribution $\widetilde{\mathcal{D}}$ which could be different from $\mathcal D$.

We focus on learning a robust linear classifier under the above setting. The family of linear classifiers is represented as $f_{\boldsymbol{\theta}}(\boldsymbol{x})=\operatorname{sign}(\boldsymbol{\theta}^\top\boldsymbol{x})$.
%and the optimal classifier that minimizes the classification error is \citep{Carmon2019UnlabeledDI}
% \begin{equation}
% \boldsymbol{\theta}^*=\mathbb{E}_{(\boldsymbol{x},y) \sim \mathcal{D}} [y\boldsymbol{x}].\label{eq-population}
% \end{equation}
%We would like to analyze the optimal choice of $\widetilde{\mathcal{D}}$ for learning a robust linear classifier. 
Recall that we first generate features and then assign pseudo labels to the features. Therefore, a self-learning paradigm is adopted here \citep{wei2020theoretical}. Given a set of unlabeled synthetic features $\{\tilde{\boldsymbol{x}}_1, \tilde{\boldsymbol{x}}_2, \ldots, \tilde{\boldsymbol{x}}_{\tilde{n}}\}$, we apply an intermediate linear classifier parameterized by $\hat{\boldsymbol{\theta}}_{\text{inter}}=\frac{1}{n}\sum_{i=1}^{n} y_i\boldsymbol{x}_i$, learned from real data $\mathcal{D}_{\text {train}}$, to assign the pseudo-label. 
Then, the synthetic data $\mathcal{D}_{\text {syn}}=\{(\tilde{\boldsymbol{x}}_1,\tilde{y}_1),\ldots, (\tilde{\boldsymbol{x}}_{\tilde{n}},\tilde{y}_{\tilde{n}})\}$, where $\tilde{y}_i=\operatorname{sign}(\hat{\boldsymbol{\theta}}_{\text{inter}}^\top\boldsymbol{x}_i)$, $i = 1,\ldots,\tilde{n}$. We combine the real data and synthetic data $\mathcal{D}_{\text{all}}:=\mathcal{D}_{\text {train}}\cup\mathcal{D}_{\text {syn}} = \{\{(\boldsymbol{x}_i,y_i)\}_{i=1}^n,\{(\tilde{\boldsymbol{x}}_i,\tilde{y}_i)\}_{i=1}^{\tilde{n}} \}$ to obtain an approximate optimal solution $\hat{\boldsymbol{\theta}}_{\text{final}}$ as \cite{Carmon2019UnlabeledDI}:
\begin{equation}\label{eq-empirical-all}
\hat{\boldsymbol{\theta}}_{\text{final}} = \frac{1}{n+\tilde{n}}(\sum_{i=1}^n y_i\boldsymbol{x}_i + \sum_{j=1}^{\tilde{n}} \tilde{y}_j\tilde{\boldsymbol{x}}_j ).
\end{equation} 
%Given $n$ labeled data $\mathcal{D}_{\text {n}}=\{(\boldsymbol{x}_1,y_1),\dots, (\boldsymbol{x}_n,y_n)\}$, we can approximate the optimal classifier using the estimated optimal classifier as follows \begin{equation}\label{eq-empirical}
% \hat{\boldsymbol{\theta}}= \frac{1}{n}\sum_{i=1}^n y_i\boldsymbol{x}_i.
%\end{equation} 
%
Note that the final linear classifier $\hat{\boldsymbol{\theta}}_{\text{final}}$ depends on the synthetic data generated from $\widetilde{\mathcal{D}}$. We aim to study which synthetic distribution $\widetilde{\mathcal{D}}$ can help reduce the adversarial classification error (also called robust error) 
$$
\mathrm{err}_{\text {robust }}(f_{\hat{\boldsymbol{\theta}}_{\text{final}}}):=\mathbb{P}_{(\boldsymbol{x}, y) \sim \mathcal{D}} (\exists \boldsymbol{\delta}\in\Delta,f_{\hat{\boldsymbol{\theta}}_{\text{final}}}\left(\boldsymbol{x}+\boldsymbol{\delta}\right) \neq y),
$$ 
where $\Delta=\{\boldsymbol{\delta}:\left\Vert \boldsymbol{\delta} \right\Vert_\infty\leq\epsilon\}$. 
And we similarly define the standard error as $\mathrm{err}_{\text {standard }}(f_{\hat{\boldsymbol{\theta}}_{\text{final}}}):=\mathbb{P}_{(\boldsymbol{x}, y) \sim \mathcal{D}}(f_{\hat{\boldsymbol{\theta}}_{\text{final}}}(\boldsymbol{x}) \neq y)$ which will be used later.

\begin{remark}[Comparison with existing literature]
In \cite{Carmon2019UnlabeledDI,Deng2021ImprovingAR}, sample complexity results are analyzed based on the same Gaussian setting. The major difference is that they all assume the learned linear classifier $\hat{\boldsymbol{\theta}}_{\text{final}}$ is only learned from synthetic data $\mathcal{D}_{\text {syn}}$ rather than the combination of the real and synthetic data $\mathcal{D}_{\text {all}}$. In general, our theoretical setup matches well with the practical algorithms.
\end{remark}

\subsection{Theoretical Insights for Optimal Synthetic Distribution}\label{sec:insights}

We first study the desired properties of the synthetic distribution $\widetilde{\mathcal{D}}$ that can lead to a better adversarial classification accuracy when the additional synthetic sample $\mathcal{D}_{\text{syn}}$ is used in the training stage. In \cite{Carmon2019UnlabeledDI}, the standard case $\widetilde{\mathcal{D}}=\mathcal{D}$ is studied, i.e., they consider the case that additional unlabeled data from the true distribution $\mathcal{D}$ is available, and they characterize the usefulness of those additional training data. Compared with \cite{Carmon2019UnlabeledDI}, we consider general distributions $\widetilde{\mathcal{D}}$ which does not necessarily equal to $\mathcal{D}$.

%and delve into which kinds of synthetic distribution $\widetilde{\mathcal{D}}$ can lead to a better solution or smaller standard and robust errors. %, even different from the true distribution $\mathcal{D}$,

First note that by the Bayes rule, the optimal decision boundary for the true data distribution is given by $\boldsymbol{\mu}^\top\boldsymbol{x}=0$. Therefore, we restrict our attention to synthetic data distributions that satisfy: (i) the marginal distribution of the label $\tilde{y}$ is also uniform in $\mathcal{Y}$, same as $\mathcal{D}$; (ii) the conditional probability densities $p(\boldsymbol{\tilde{x}}|\tilde{y}=1)$ and $p(\boldsymbol{\tilde{x}}|\tilde{y}=-1)$ of the synthetic data distribution are symmetric around the true optimal decision boundary $\boldsymbol{\mu}^\top\boldsymbol{x}=0$. 
More specifically, we start with a special case of the synthetic data distribution $\widetilde{\mathcal{D}}_\mathcal{X}=0.5\mathcal{N}(\tilde{\boldsymbol{\mu}},\sigma^2 \mathbb{I})+0.5 \mathcal{N}(-\tilde{\boldsymbol{\mu}},\sigma^2 \mathbb{I})$ (note that when $\tilde{\boldsymbol{\mu}}=c\boldsymbol{\mu}$ for some constant $c$, the above two conditions are all satisfied).

In the following proposition, we present several representative scenarios of synthetic distributions in terms of how they may contribute to the downstream classification task. Figure \ref{theorem3_demo} gives a pictorial demonstration for different cases.

\begin{proposition}\label{theorem3}
Consider a special form of synthetic distributions $\widetilde{\mathcal{D}}_\mathcal{X}=0.5\mathcal{N}(\tilde{\boldsymbol{\mu}},\sigma^2 \mathbb{I})+0.5 \mathcal{N}(-\tilde{\boldsymbol{\mu}},\sigma^2 \mathbb{I})$ and assume $\{\tilde{\boldsymbol{x}}_1, \ldots, \tilde{\boldsymbol{x}}_{\tilde{n}}\}$ are samples from $\widetilde{\mathcal{D}}_\mathcal{X}$. 
We follow the self-learning paradigm described in Section \ref{sec:theo_setup} to learn the classifier $f_{\hat{\boldsymbol{\theta}}_{\text{final}}}$, when $\tilde{n}$ is sufficiently large we have:
\begin{enumerate}[start=1,label={ Case \arabic*:},wide=\parindent]
 \item Inefficient $\widetilde{\mathcal{D}}_\mathcal{X}$. When $\langle \tilde{\boldsymbol{\mu}},\boldsymbol{\mu} \rangle$ = 0, the standard error $\mathrm{err}_{\text {standard }}(f_{\hat{\boldsymbol{\theta}}_{\text{final}}})$ achieves the maximum and when $\langle \tilde{\boldsymbol{\mu}},\boldsymbol{\mu}-\varepsilon \mathbf{1}_d \rangle$ = 0, the robust error $\mathrm{err}_{\text {robust }}(f_{\hat{\boldsymbol{\theta}}_{\text{final}}})$ achieves the maximum. 

\item Optimal $\widetilde{\mathcal{D}}_\mathcal{X}$ for clean accuracy. When $\tilde{\boldsymbol{\mu}}=c \boldsymbol{\mu}$ for $c>0$, $\mathrm{err}_{\text {standard }}(f_{\hat{\boldsymbol{\theta}}_{\text{final}}})$ achieves the minimum, and the larger the $c$ is, the smaller the $\mathrm{err}_{\text {standard }}(f_{\hat{\boldsymbol{\theta}}_{\text{final}}})$.

\item Optimal $\widetilde{\mathcal{D}}_\mathcal{X}$ for robust accuracy. When $\tilde{\boldsymbol{\mu}}=c (\boldsymbol{\mu} - \varepsilon \mathbf{1}_d)$ for $c>0$, the robust error $\mathrm{err}_{\text {robust }}(f_{\hat{\boldsymbol{\theta}}_{\text{final}}})$ achieves the minimum, and the larger the $c$ is, the smaller the $\mathrm{err}_{\text {robust }}(f_{\hat{\boldsymbol{\theta}}_{\text{final}}})$.
\end{enumerate}
\end{proposition}

%\vspace{-0.1in}
\begin{figure}
\centering
  \includegraphics[scale=0.7]{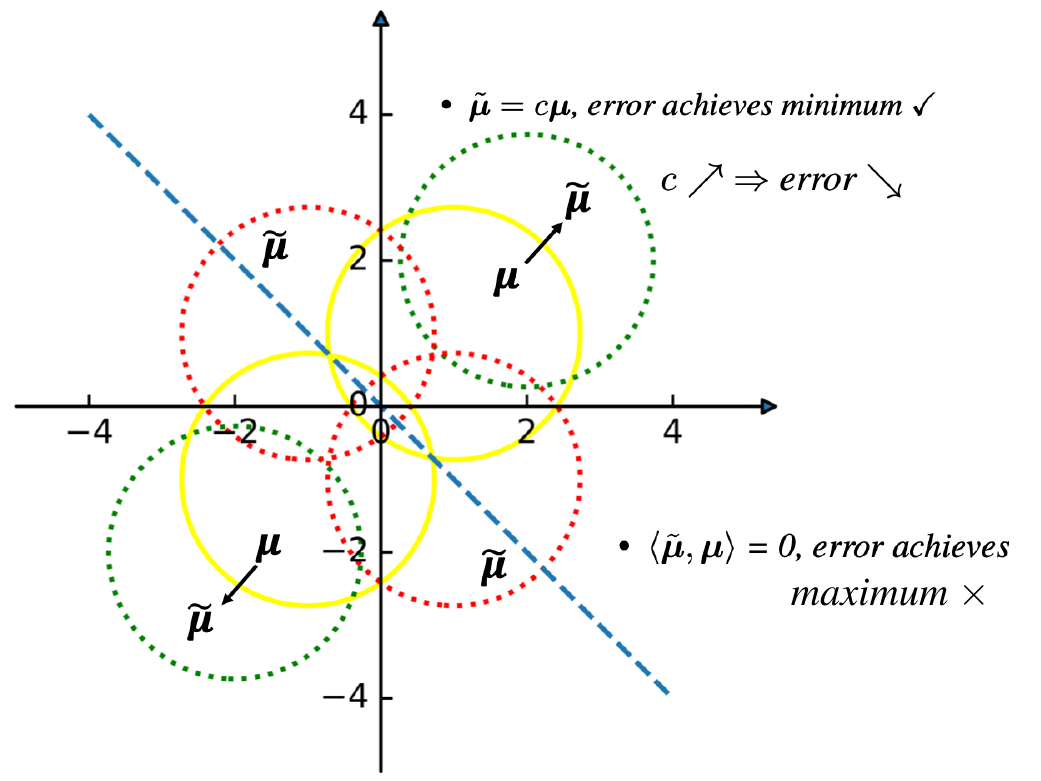}
  \caption{Demonstration of Proposition \ref{theorem3}. We refer to the main text for a detailed explanation.}\label{theorem3_demo}
  % \vspace{-0.1in}
\end{figure}

\begin{remark}[Comparison with the existing characterization of the synthetic distribution]
We briefly comment on the main differences and similarities with \cite{Deng2021ImprovingAR}, in which a similar result was presented in Theorem 4 therein.
In \cite{Deng2021ImprovingAR}, the final solution of $\boldsymbol{\theta}^*$ was given for minimizing robust error $\mathrm{err}_{\text {robust }}(f_{\hat{\boldsymbol{\theta}}_{\text{final}}})$ and they provides a specific unlabeled distribution $\tilde{\boldsymbol{\mu}}=\boldsymbol{\mu}-\varepsilon \mathbf{1}_d$ that achieves better performance under certain condition. In this paper, we propose a general family of optimal distribution controlled by a scalar $c$, which represents the distinguishability of the feature. The final $\boldsymbol{\theta}^*$ used in \cite{Deng2021ImprovingAR} can be viewed as a special case of $c=1$. 
%recovered when $c \to \infty$.
Therefore, our conclusion points out the optimality condition of unlabeled distribution and inspires a line of work to improve the performance of $\hat{\boldsymbol{\theta}}_{\text{final}}$ by making the feature of unlabeled distribution distinguishable.
\end{remark}

% Moreover, we can also quantify the difference between the empirical classifier \ref{eq-empirical-all} and the optimal classifier \ref{eq-population} in the population version using standard concentration inequality:
% $$
% \mathbb{P}(\|\hat{\boldsymbol{\theta}}_{\text{final}}-\boldsymbol{\theta}^*\|>\epsilon)\leq \frac{\operatorname{Var}(\hat{\boldsymbol{\theta}}_{\text{final}})}{\epsilon^2},
% $$
% where $\operatorname{Var}(\hat{\boldsymbol{\theta}}_{\text{final}})$ depends on the synthetic distribution $\tilde{\mathcal{D}}$ chosen. 

% Any class conditional Gaussian distribution, sub-Gaussian distribution, and distribution in the exponential family satisfies the optimality condition. A generalized version of Lemma \ref{theorem1} can be found in Appendix \ref{theorem1-a}.

% comment out tentatively
We also study the sample complexity for the synthetic distributions satisfying the condition in Proposition~\ref{theorem3}. The results below show that for larger $c$, we typically need fewer synthetic samples to achieve the desired robust accuracy.
% Here, we consider the sample complexity needed to achieve non-trivial adversarial robust accuracy rather than guarantee the generalization gap by using the tools like Rademacher Complexity and so on.
\begin{theorem}\label{theorem2}
Under the parameter setting $\epsilon < \frac12, \ \sigma = (nd)^{1/4}, \ \|\boldsymbol{\mu}\|^2 = d$, there exists a universal constant $\tilde C$ such that for $\epsilon^2 \sqrt{d/n} \geq \tilde C$, where $n$ is the number of labeled real data used to construct the intermediate classifier, and additional $\tilde n$ synthetic feature generated with mean vector $\pm c\boldsymbol{\mu}$ and pseudo labels, if 
\[\tilde n \geq \frac{288n}{c} \epsilon^2 \sqrt{\frac{d}{n}},
\]
then
$$
\mathbb{E}_{\hat{\boldsymbol{\theta}}_{\text {final }}} [\text { err }_{\text {robust }}(f_{\hat{\boldsymbol{\theta}}_{\text {final }}}) ]\leq 10^{-3}.
$$
\end{theorem}

\begin{table}[ht!]
\centering
\caption{Simulation results validating findings in Proposition \ref{theorem3} and Theorem \ref{theorem2}. We use ``Real'' to denote the real data distribution and $n$ to denote the number of data from the real distribution, and we use ``$c$'' to denote different synthetic distributions with mean $\tilde{\boldsymbol{\mu}}=c \boldsymbol{\mu}$ and use $\tilde{n}$ to denote the number of synthetic data. The average accuracy and the standard deviation in the bracket are obtained from 50 repetitions.}
%\vskip 0.15in
\begin{center}
%\begin{small}
\resizebox{0.48\textwidth}{!}{
  \begin{tabular}{c|lrr}
   \toprule  
     &     & \multicolumn{1}{l}{clean acc} & \multicolumn{1}{l}{rob acc} \\
       \midrule
  \multirow{2}[0]{*}{Real}& $n=10$ &    0.9023 (0.0192) & 0.6843 (0.0359) \\
   & $n=100$ &    {0.9682 (0.0014)} & {0.8239 (0.0028)} \\
  \midrule
  \multirow{2}[0]{*}{$c=0.5$} &  $\tilde{n}=10$ &    0.7562 (0.0564) & 0.4611 {(0.0694)} \\
   & $\tilde{n}=100$ &    0.9505 (0.0047) & 0.7848 {(0.0111)} \\
  \midrule
  \multirow{2}[0]{*}{$c=1$}  & $\tilde{n}=10$ &    0.8866 (0.0273) & 0.6557 {(0.0487)}  \\
     &$\tilde{n}=100$ &    0.9695 (0.0012) & 0.8239 {(0.0031)} \\
     \midrule
  \multirow{2}[0]{*}{$c=1.5$} & $\tilde{n}=10$ &    \textbf{0.9400} (0.0100) &\textbf{ 0.7603} {(0.0233)} \\
       & $\tilde{n}=100$ &   \textbf{0.9743} (0.0008) & \textbf{0.8343} {(0.0011)} \\
  
  \bottomrule
 \end{tabular}
 }%  
%\end{small}
\end{center}
% \vskip -0.1in
\label{tab:simulation}%
\end{table}%

% \vspace{-0.1in}
\paragraph{Simulation Results.} %\label{2d-experiments}
To verify the findings in Proposition \ref{theorem3} and Theorem \ref{theorem2}, we conduct extensive simulation experiments under Gaussian distributions with varying data dimensions, sample sizes, and the mean vector $\tilde{\boldsymbol{\mu}}$. In most cases, we find increasing $c$ for synthetic distribution can lead to better clean and robust accuracy. A detailed description of the experimental setting can be found in Appendix \ref{simu:setting}. In Table \ref{tab:simulation}, we demonstrate the clean and robust accuracy learned on synthetic distribution with the angle between $\boldsymbol{\mu}$ and $\mathbf{\epsilon}\mathbf{1}_d$ equals $0^{\circ}$.  Remarkably, the classifier learned only from the synthetic distribution with $\tilde{\boldsymbol{\mu}}=c\boldsymbol{\mu}$ with $c>1$ achieves better performance even than the iid samples (denoted as ``Real'' in Table). 

The closed form of optimal synthetic distribution is $\tilde{\boldsymbol{\mu}}=\boldsymbol{\mu}-\varepsilon \mathbf{1}_d$ as stated in Proposition \ref{theorem3}. It is interesting to see when $\tilde{\boldsymbol{\mu}}=c \boldsymbol{\mu}$ and $\boldsymbol{\mu}$ is not aligned on  $\mathbf{\epsilon}\mathbf{1}_d$, whether increasing $c$ provides better data distribution for learning a robust classifier. The answer is still true, demonstrated by the experiment results (with $30^{\circ}$, $60^{\circ}$, and $90^{\circ}$) in Table \ref{t1}, \ref{t2}, \ref{t3} and \ref{t4} in Appendix \ref{simu-gaussian}.
%
%which verifies Lemma \ref{theorem3}.

\section{Contrastive-Guided Diffusion Process}\label{sec:contr}

It has been shown in Proposition \ref{theorem3} and Theorem \ref{theorem2} that the synthetic data can help improve the classification task, especially when the representation of different classes is more distinguishable in the synthetic distribution. Meanwhile, the contrastive loss \citep{Oord2018RepresentationLW} can be adopted to explicitly control the distances of the representation of different classes. 
Therefore, we propose a variant of the classical diffusion model, named {\it Contrastive-Guided Diffusion Process (Contrastive-DP)}, to enhance the sample efficiency of the generative model. In this section, we first present the overall algorithm of the proposed Contrastive-DP procedure in Section~\ref{contr-cgdf}, then we describe the detailed design of the contrastive loss in Section~\ref{contr-loss}.

\subsection{Algorithm for Contrastive-Guided DP} \label{contr-cgdf}

The detailed generation procedure of Contrastive-DP is given in Algorithm \ref{alg1}. We highlight below some major differences between the proposed Contrastive-DP and the vanilla DDIM algorithm. In each time step $t$ of the generation procedure, given the current value $\boldsymbol{x}^{(i)}_t$, we add the gradient of the contrastive loss $\ell_{\text{contra}}(\boldsymbol{x}_{t}^{(i)}, \boldsymbol{x}_p^{(i)};\tau)$ with respect to $\boldsymbol{x}_{t}^{(i)}$ to the original diffusion generative process, here $\boldsymbol{x}_p^{(i)}$ is the positive pair of $\boldsymbol{x}_t^{(i)}$ (will be explained in detail later), $\tau$ is the temperature for softmax, and $\lambda$ is the hyperparameter balancing the contrastive loss within the diffusion process. 

This modification ensures that the generated data will be distinguishable among data in the same batch. The construction of the contrastive loss $\ell_{\text{contra}}(\cdot)$ is very flexible -- we can adopt multiple forms of contrastive loss together with different selection strategies of positive and negative pairs, which will be discussed in detail in the following.

% \vspace{-0.1in}
\subsection{Contrastive Loss for Diffusion Process} \label{contr-loss}
\label{subsection-contrastive}

%CDGP
% \vspace{-0.1in}
\begin{algorithm}[tb]
	\caption{Generation in Contrastive-guided Diffusion Process (Contrastive-DP)}
	\begin{algorithmic}[1]
      \REQUIRE Contrastive loss temperature $\tau$, diffusion process hyperparameter $\sigma_t$
		\STATE $\mathbb{X}_T = \{\boldsymbol{x}_{T}^{(i)}\}_{i=1}^m \sim \mathcal{N}(\mathbf{0}, \mathbb{I})$
		\STATE $t=T$
    \WHILE{$t\neq 1$}
      \FOR{$i = 1$ \textbf{to} $m$}
        \STATE Sampling $\boldsymbol{\epsilon}_t \sim \mathcal{N}(\mathbf{0}, \mathbb{I})$
        \STATE Choosing $\boldsymbol{x}_p^{(i)}$ as the positive pair of $\boldsymbol{x}_{t}^{(i)}$
        \STATE $\Delta\boldsymbol{x}_t^{(i)} = \lambda \cdot \nabla_{\boldsymbol{x}_t^{(i)}} \ell_{\text{contra}}(\boldsymbol{x}_{t}^{(i)}, \boldsymbol{x}_p^{(i)};\tau) +\boldsymbol{\epsilon}_{\theta}(\boldsymbol{x}_{t}^{(i)},t)$ \\
        \STATE 
        $\boldsymbol{x}_{t-1}^{(i)}=\sqrt{\alpha_{t-1}}(\frac{\boldsymbol{x}_{t}^{(i)}-\sqrt{1-\alpha_{t} \Delta\boldsymbol{x}_t^{(i)}}}{\sqrt{\alpha_{t}}})+\sqrt{1-\alpha_{t-1}-\sigma_t^2} \cdot \Delta\boldsymbol{x}_t^{(i)}+\sigma_t \boldsymbol{\epsilon}_t$
        \STATE $t=t-1$
      \ENDFOR
    \ENDWHILE
    \STATE Return $\mathbb{X}_0 = \{\boldsymbol{x}_{0}^{(i)}\}_{i=1}^m$
	\end{algorithmic}
	\label{alg1}
\end{algorithm}

Let $\mathbb{X} = \{\boldsymbol{x}_{1},...,\boldsymbol{x}_{m}\}$ be a minibatch of training data. We apply the contrastive loss to the embedding space. Assume $f(\cdot)$ is the feature extractor that maps the input data in $\mathbb{X}$ onto the embedding space.
%where $\boldsymbol{x}_i \in \mathbb{R}^{W \times H}$ is the $i$th image in that minibatch. 
In general, we adopt two forms of the contrastive loss $\ell_{\text{contra}}(\boldsymbol{x}_{t}^{(i)}, \boldsymbol{x}_p^{(i)};\tau)$ which will be used in Algorithm \ref{alg1}.

% \begin{enumerate}[font=\bfseries,align=left,leftmargin=*]
% \item[InfoNCE loss:] 

First is the InfoNCE loss: 
$$
\ell_{\text{InfoNCE}}\left(\boldsymbol{x}_{a}, \boldsymbol{x}_{p} ; \tau \right)=-\log (\frac{g_\tau(\boldsymbol{x}_{a}, \boldsymbol{x}_{p})}{{\sum_{k=1}^{m} \mathbf{1}_{\mathrm{k} \neq \mathrm{a}} g_\tau(\boldsymbol{x}_{a}, \boldsymbol{x}_{k})}}),$$
% $\ell_{\text{InfoNCE}}\left(\boldsymbol{x}_{a}, \boldsymbol{x}_{p} ; \tau \right)=-\log (g_\tau(\boldsymbol{x}_{a}, \boldsymbol{x}_{p}){\sum_{k=1}^{m} \mathbf{1}_{\mathrm{k} \neq \mathrm{a}} g_\tau(\boldsymbol{x}_{a}, \boldsymbol{x}_{k})}),$
where $m$ is the batch size, $\tau$ is the temperature for softmax, $\boldsymbol{x}_{a}$, $\boldsymbol{x}_{p}$ denote the anchor and the positive pair, respectively, $g_\tau(\boldsymbol{x},\boldsymbol{x}') =\exp(f(\boldsymbol{x})^{\top} f(\boldsymbol{x}') / \tau)$, and all images except the anchor $\boldsymbol{x}_{a}$ in the minibatch $\mathbb{X}$ is negative pairs. 
%
%The InfoNCE loss is defined as:
% $$
% \ell_{\text{InfoNCE}}\left(\boldsymbol{x}_{a}, \boldsymbol{x}_{p} ; \tau \right)=-\log \left(\frac{\exp \left(f\left(\boldsymbol{x}_{a}\right)^{\top} f\left(\boldsymbol{x}_{p}\right) / \tau\right)}{\sum_{k=1}^{m} \mathbf{1}_{\mathrm{k} \neq \mathrm{a}} \cdot \exp \left(f\left(\boldsymbol{x}_{a}\right)^{\top} f\left(\boldsymbol{x}_{k}\right) / \tau\right)}\right),
% $$
InfoNCE loss is an unsupervised learning metric and does not explicitly distinguish the representation from different classes, which implicitly regards the representation from the same class as negative pair. 

% \paragraph{Hard negative mining contrastive loss.} 
% $\ell_{\text{HNM}}\left(\boldsymbol{x}_{a}, \boldsymbol{x}_{p} ; \tau\right) = -\log(g_\tau(\boldsymbol{x}_{a}, \boldsymbol{x}_{p})/(g_\tau(\boldsymbol{x}_{a}, \boldsymbol{x}_{p}) + m/\tau^-(\mathbb{E}_{\boldsymbol{x}_{n} \sim q_{\beta}}[(g_\tau(\boldsymbol{x}_{a}, \boldsymbol{x}_{n})]-\tau^{+} \mathbb{E}_{\boldsymbol{v} \sim q_{\beta}^{+}}[(g_\tau(\boldsymbol{x}_{a}, \boldsymbol{v})]))),$

Second is the hard negative mining loss: 
\[
\ell_{\text{HNM}}\left(\boldsymbol{x}_{a}, \boldsymbol{x}_{p} ; \tau\right) = -\log \frac{g_\tau(\boldsymbol{x}_{a}, \boldsymbol{x}_{p})}{ g_\tau(\boldsymbol{x}_{a}, \boldsymbol{x}_{p}) + \frac{m}{\tau^-} h_\tau(\boldsymbol{x}_{a}) } ,
\]
where
\[
h_\tau(\boldsymbol{x}_{a}) = \mathbb{E}_{\boldsymbol{x}_{n} \sim q_{\beta}}[g_\tau(\boldsymbol{x}_{a}, \boldsymbol{x}_{n})]-\tau^{+} \mathbb{E}_{\boldsymbol{v} \sim q_{\beta}^{+}}[g_\tau(\boldsymbol{x}_{a}, \boldsymbol{v})],
\]
and $m$ denotes the batch size, $\tau^{-}=1-\tau^{+}$ denotes the probability of observing any different class with $\boldsymbol{x}_{a}$ and $q_\beta$ is an unnormalized von Mises–Fisher distribution \citep{Jammalamadaka2011DirectionalSI}, with
mean direction $f(\boldsymbol{x})$ and ``concentration parameter'' $\beta$ to control the hardness of negative mining; $q_\beta$ and $q_\beta^{+}$ can be easily approximated by Monte-Carlo importance sampling techniques. We refer to \cite{Chuang2020DebiasedCL,Robinson2021ContrastiveLW} for detailed descriptions of hard negative mining contrastive loss. Compared with the InfoNCE loss that does not consider class/label information, the hard negative mining (HNM) loss enhances the discriminative ability of different classes in the feature space. 

It is worth mentioning that the Contrastive-DP enjoys the plugin-type property -- it does not modify the original training procedure of diffusion processes and can be easily adopted to various kinds of diffusion models. 

Moreover, it has a close relationship with Stein Variational Gradient Descent \citep{Liu2016SteinVG,Liu2017SteinVG}. By utilizing the existing theoretical tools \citep{Shi2022AFC} for analyzing the convergence rate for SVGD, we may also get the convergence guarantee for Contrastive-DP, which is left to future work.

% \begin{small}
% \begin{align}
% &\ell_{\text{HNM}}\left(\boldsymbol{x}_{a}, \boldsymbol{x}_{p} ; \tau\right)= \notag\\&-\log \frac{\exp{\left(\frac{f(\boldsymbol{x}_{a})^{T} f\left(\boldsymbol{x}_{p}\right)}{\tau}\right)}}{\exp{\left(\frac{f(\boldsymbol{x}_{a})^{T} f\left(\boldsymbol{x}_{p}\right)}{\tau}\right)}+\frac{N}{\tau^{-}}\left(\mathbb{E}_{\boldsymbol{x}_{n} \sim q_{\beta}}\left[\exp{\left(\frac{f(\boldsymbol{x}_{a})^{T} f\left(\boldsymbol{x}_{n}\right)}{\tau}\right)}\right]-\tau^{+} \mathbb{E}_{\boldsymbol{v} \sim q_{\beta}^{+}}\left[\exp{\left(\frac{f(\boldsymbol{x}_{a})^{T} f(\boldsymbol{v})}{\tau}\right)}\right]\right)},\notag
% \end{align}
% \end{small}

\begin{remark} [Connection with Stein Variational Gradient Descent]
Recall the updating step in Contrastive-DP is $$\Delta\boldsymbol{x}_t^{(i)} = \underbrace{\boldsymbol{\epsilon}_{\theta}(\boldsymbol{x}_{t}^{(i)},t)}_{\text {score}}+\lambda \underbrace{\cdot \nabla_{\boldsymbol{x}_t^{(i)}} \ell_{\text{contra}}(\boldsymbol{x}_{t}^{(i)}, \boldsymbol{x}_p^{(i)};\tau)}_{\text {repulsive force}}.$$
And the updating step $\Delta\boldsymbol{x}_t^{(i)}$ in SVGD equals \cite{Liu2017SteinVG}
    $$ \frac{1}{n}\sum_{j=1}^n\Big[\underbrace{s_p\big(\boldsymbol{x}_t^{(j)}\big) k\big(\boldsymbol{x}_t^{(j)}, \boldsymbol{x}_t^{(i)}\big)}_{\text {weighted sum of score}}+\underbrace{\nabla_{\boldsymbol{x}_t^{(j)}} k\big(\boldsymbol{x}_t^{(j)}, \boldsymbol{x}_t^{(i)}\big)}_{\text {repulsive force}}\Big],$$
where $\boldsymbol{x}_t^{(j)}$ is other samples in the batch, $s_p(\boldsymbol{x}):= \nabla_{\boldsymbol{x}}\log p(\boldsymbol{x})$ is the score function, and $k(\boldsymbol{x},\boldsymbol{x}')$ is a positive definite kernel. Contrastive-DP is similar to the SVGD with the score that pulls the particles to high-density region and the repulsive force that pulls the particles away from each other, but the score in Contrastive-DP is easier to calculate as it does not require a weighted sum over the current batch.
\end{remark}

% \vspace{-0.1in}
\paragraph{Numerical Validations.} We first demonstrate the effectiveness of Contrastive-DP in Figure \ref{effectiveness-contrastive-2d} using a simulation example. Consider the binary classification problem as in Section \ref{sec:theo_setup}, and the real data for each class are generated from a Gaussian distribution. Figure \ref{effectiveness-contrastive-2d}(a) demonstrates the synthetic data generated by the vanilla diffusion model, which recovers the ground-truth Gaussian distribution well. When using the contrastive-DP procedure with HNM loss, we obtain the generated synthetic data as shown in Figure \ref{effectiveness-contrastive-2d}(b), which is more distinguishable with a much smaller variance. 
%Training a model on that synthetic distribution has smaller sample complexity than under true distributions.

\begin{figure}[th]
\centering
% \vspace{-0.2in}
\subfigure[DDPM]{	\includegraphics[width=.23\textwidth]{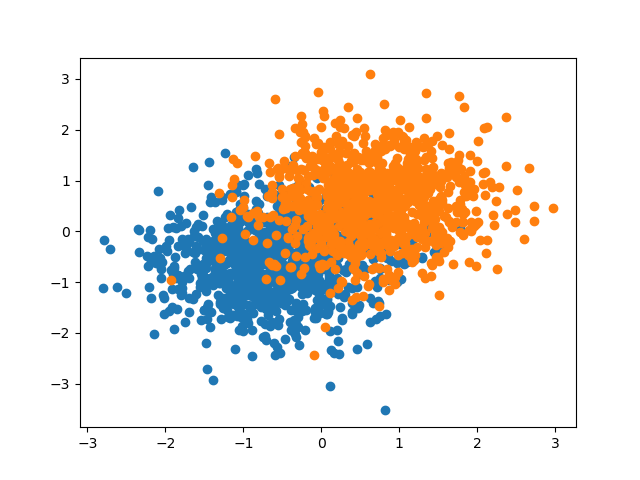}}
\subfigure[ Contrastive-DP]{\includegraphics[width=.23\textwidth]{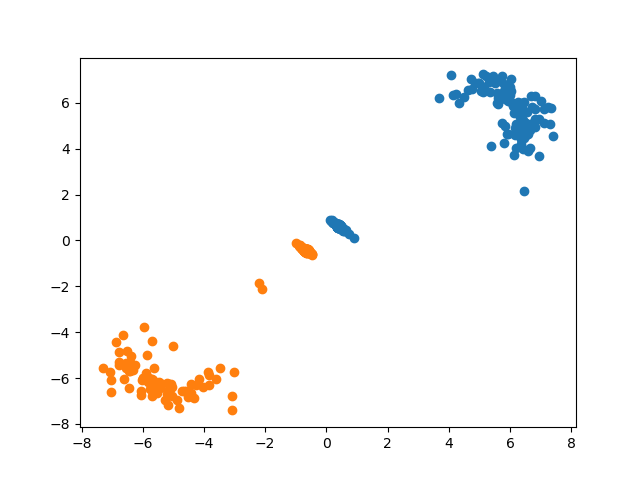}}
% \vspace{-0.1in}
\caption{An illustration of the effectiveness of synthetic distribution guided by contrastive loss. More details can be found in Appendix \ref{2d-simulation}.
%Diffusion guided by conditional contrastive loss has larger $c$ and small variance, which is more data efficient for training a robust model.
}
\label{effectiveness-contrastive-2d}
% \vspace{-0.1in}
\end{figure}

In addition, Figure \ref{contrastive-simulation1} and Figure \ref{contrastive-simulation2} in Appendix \ref{2d-simulation} demonstrate the synthetic data distribution guided by different kinds of contrastive loss mentioned above. It can be shown that InfoNCE loss and hard negative mining method cannot explicitly distinguish the data within the same class and thus form a circle within each class to maximize the distance between samples, while the conditional version of contrastive loss (given the oracle class information) can make two classes more separable.

% \vspace{-0.1in}
\section{Real Data Examples}\label{sec:numerical}

%We first describe the pipeline of synthetic data generation for adversarial robustness and a corresponding setting in Section~\ref{experimental-setup}.
In this section, we demonstrate the effectiveness of the proposed contrastive guided diffusion process for synthetic data generation in adversarial classification tasks. We first compare the performance of Contrastive-DP with the vanilla diffusion models in Section~\ref{sec:exp}. Then, we present a comprehensive ablation study on the performance of Contrastive-DP to shed insights on how to adopt the contrastive loss functions and further data selection methods on the diffusion model in Section~\ref{sec:ablation}. 
%\yd{I comment it}
% In particular, we investigate the performance of seven contrastive loss functions, the effect of the hyperparameter that controls the strength of the guidance of the contrastive loss, and four proposed data selection criteria.

% \vspace{-0.1in}
\subsection{Experimental Results} \label{sec:exp}
We test the contrastive-DP algorithm on three image datasets, the MNIST dataset \citep{LeCun1998GradientbasedLA}, CIFAR-10 dataset \citep{Krizhevsky2009LearningML}, and the Traffic Signs dataset \citep{Houben-IJCNN-2013}.
%\paragraph{Comparison studies.}
% We test the contrastive-DP algorithm on three image datasets, the MNIST dataset \citep{LeCun1998GradientbasedLA}, CIFAR-10 dataset \citep{Krizhevsky2009LearningML} and Traffic Signs dataset \citep{Houben-IJCNN-2013}. MNIST dataset CIFAR-10 dataset contains 50k training images in 10 classes and 10K images for testing, while the Traffic signs dataset contains 39252 training images in 43 classes and 12629 images for testing. For the CIFAR-10 dataset, we generate 50K, 200K, and 1M additional images together with the original training images for adversarial training, while for the Traffic signs dataset, we synthetic 50k images. To demonstrate our Contrastive-DP algorithm is flexible to be adopted to various kinds of diffusion models and make use of the existing pre-trained models, we establish our Contrastive-DP algorithm on the unconditional DDIM for CIFAR-10 datasets and on the conditional DDPM for Traffic signs dataset. 
A detailed description of the pipeline for generating data and the corresponding hyperparameter can be found in Appendix \ref{experimental-setup}.

Figure \ref{mnist} shows the efficacy of synthetic data in terms of improving adversarial robustness on MNIST data. We fix the total number of images for training the classifier as 5K (e.g., when we use 1K real data and 4K synthetic data, to avoid information leakage, we only use these 1K real data to train the diffusion model. The same case for 2K, 3K, and 4K.) \footnote{The case with 5K training images is a special scenario: (i) the accuracy resulted from using all of the 5K real data is the baseline (dashed line shown in Figure \ref{mnist}); (ii) while the accuracy resulted from using 5K synthetic data generated by the diffusion model trained on 5K real data is shown in the scatter plot.} 
Notably, synthetic data improves the clean and robust accuracy even without using any real data, which is consistent with the results in Proposition \ref{theorem3} and Theorem \ref{theorem2}. Moreover, lots of proposed methods \citep{Madry2017TowardsDL, Tsipras2018RobustnessMB, Zhang2019TheoreticallyPT} improves the robust accuracy at the sacrifice of clean accuracy, while adding contrastive guidance increase both the clean and the robust accuracy at the same time under the majority of the settings in Figure \ref{mnist}, which shows the potential of Contrastive-DP in achieving a better trade-off between clean and robust accuracy.

\begin{figure}[h]
\centering
% \vspace{-0.2in}
\subfigure[acc]{	\includegraphics[width=.23\textwidth]{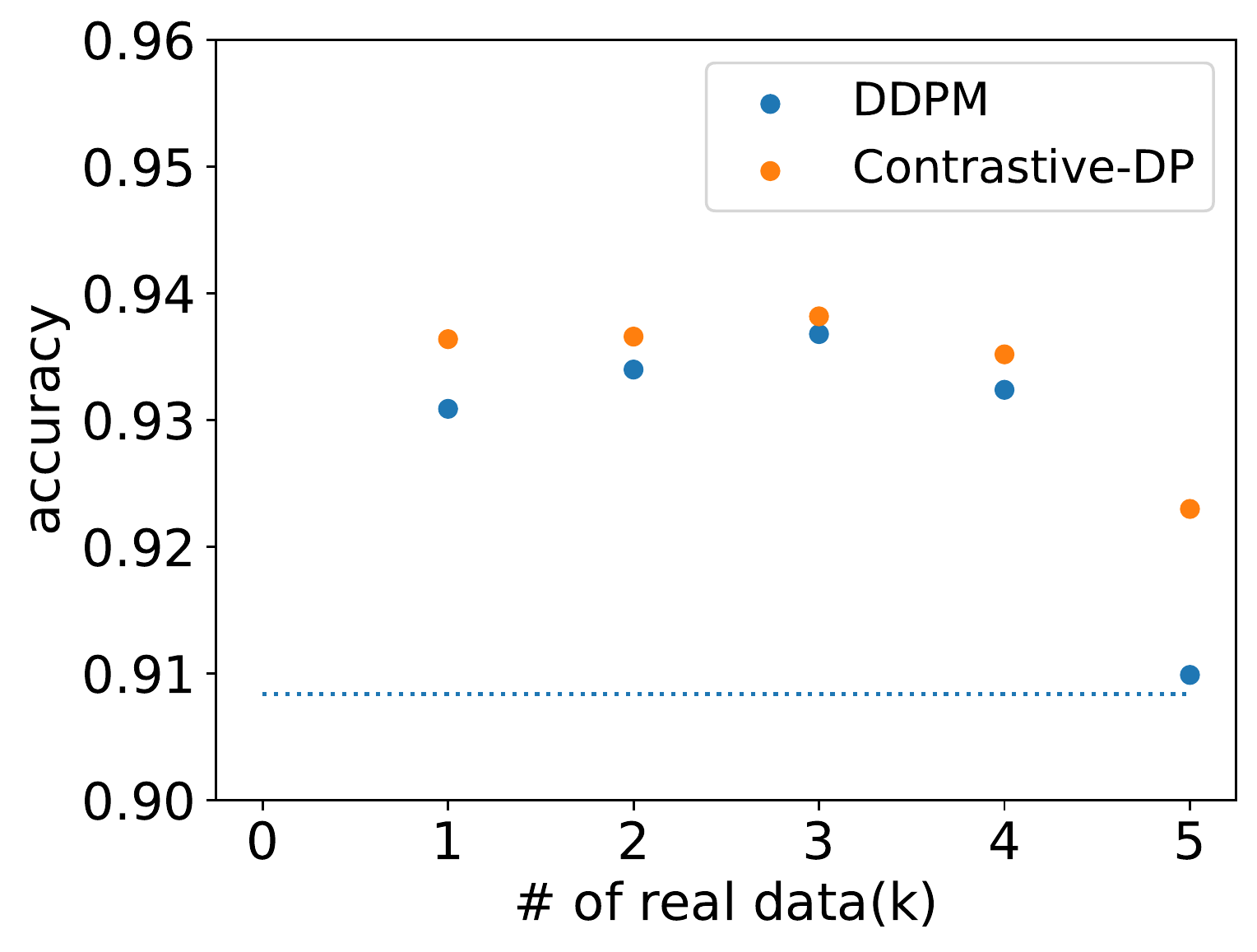}}
\subfigure[ rob acc]{\includegraphics[width=.23\textwidth]{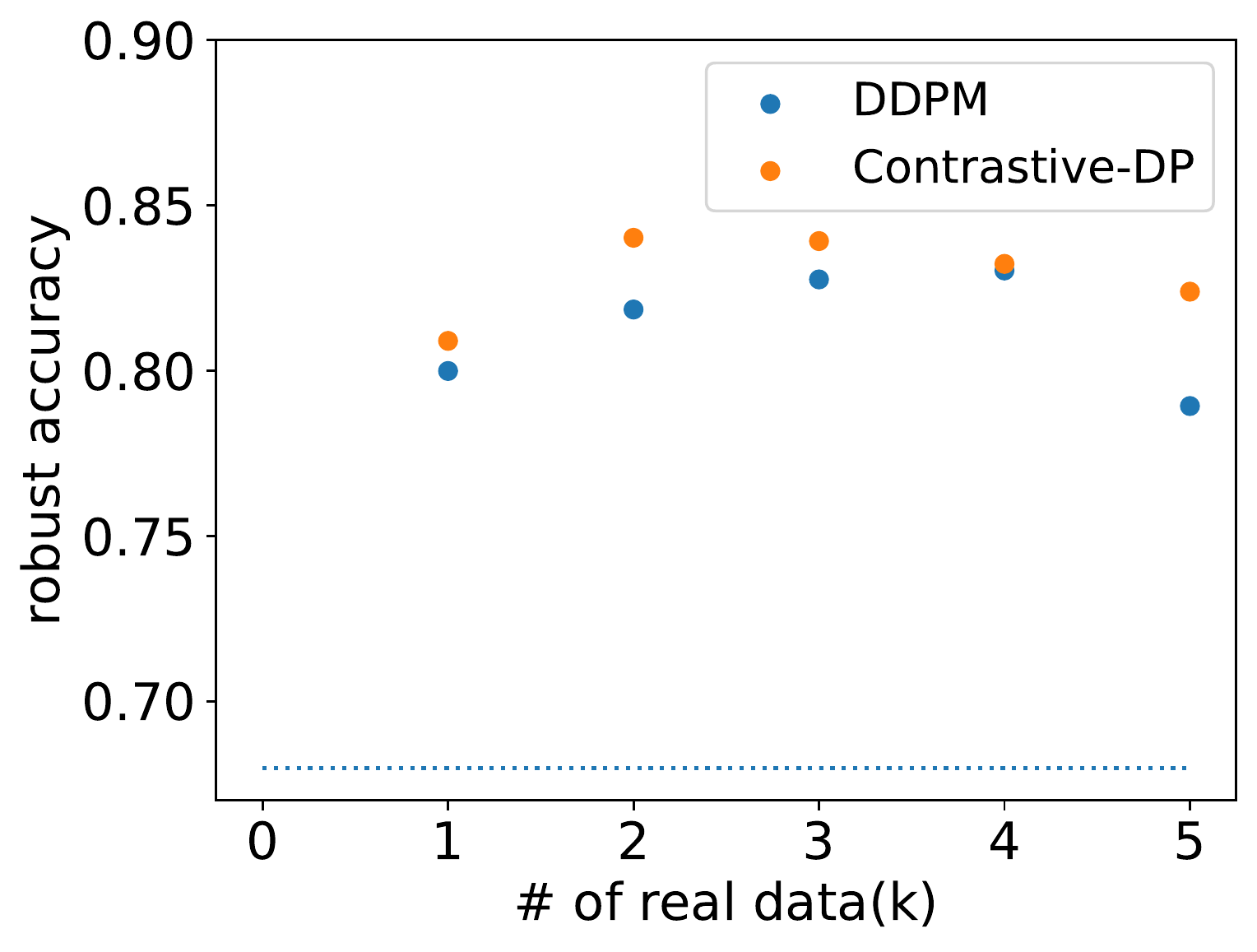}}
% \vspace{-0.1in}
\caption{The efficacy of the synthetic data on MNIST dataset. All the results are trained with 5K data (with different proportions of real and synthetic data). Subfigure (a) and (b) show the clean and adversarial accuracy on the MNIST dataset, respectively. The dashed line is the performance training only on 5K real data.}
\label{mnist}
% \vspace{-0.1in}
\end{figure}

Table \ref{tab:cifar} demonstrates the effectiveness of our contrastive-DP algorithm on the CIFAR-10 dataset \footnote{Since the Pytorch Implementation of \cite{gowal2021improving} is not open source, we utilize the best \href{https://github.com/imrahulr/adversarial_robustness_pytorch}{unofficial implementation} to reconduct all the experiments for a fair comparison.}, which achieves better robust accuracy on all data regimes than the vanilla DDPM and DDIM. All of the results are higher than the baseline result without synthetic data ($81.98\%\pm 0.58\%$ for clean accuracy and $50.42\%\pm0.35\%$ for robust accuracy) by a large margin (+6.18\% in 50K setting and +9.57\% in 1M setting). Table \ref{tab:gtsrb} demonstrates the effectiveness of our contrastive-DP algorithm on the Traffic Signs dataset. Our contrastive-DP achieves better clean and robust accuracy than the vanilla DDPM model and is also higher than the baseline result without synthetic data by a large margin (+10.24\%).

\begin{table*}[htbp]
 \centering
 \caption{The clean and adversarial accuracy on CIFAR-10 dataset. The robust accuracy is reported by AUTOATTACK \citep{Croce2020ReliableEO} with $\epsilon_\infty=8/255$ and WRN-28-10. 50k, 200k, and 1M denote the number of synthetic used for adversarial training. The results and the standard deviation in the bracket are obtained from 3 repetitions. }
% \vskip 0.15in
  \resizebox{\textwidth}{!}{\begin{tabular}{l|rr|rr|rr}
\toprule 
      & \multicolumn{2}{c|}{50K} & \multicolumn{2}{c|}{200K} & \multicolumn{2}{c}{1M}  \\
      & \multicolumn{1}{l}{clean acc} & \multicolumn{1}{l|}{rob acc} & \multicolumn{1}{l}{clean acc} & \multicolumn{1}{l|}{rob acc} & \multicolumn{1}{l}{clean acc} & \multicolumn{1}{l}{rob acc}\\
\midrule
  DDIM  & \textbf{84.05\%(0.06\%)} & 56.29\%(0.15\%) & 84.86\%(0.43\%) & 57.83\%(0.28\%)& 85.73\%(0.51\%) & 59.85\%(0.26\%)\\
  Contrastive-DP &83.66\%(0.21\%) & \textbf{56.60\%}(0.17\%) & \textbf{85.71\%(0.18\%)} & \textbf{58.24\%}(0.20\%) & \textbf{86.30\%}(0.09\%) & \textbf{59.99\%}(0.23\%) \\
  \midrule
  DDPM  & \textbf{84.84\%(0.37\%)} & \textbf{56.30\%(0.06\%)} & 85.23\%(0.25\%) & 58.28\%(0.09\%)& \textbf{86.86\%(0.04\%)} & 59.03\%(0.16\%) \\
  Contrastive-DP &84.70\%(0.43\%) & 56.18\%(0.10\%) & \textbf{85.61\%(0.14\%)} & \textbf{58.62\%}(0.12\%) & 86.30\%(0.10\%) & \textbf{59.74\%}(0.26\%)  \\
   
    \bottomrule
  \end{tabular}}%
 \label{tab:cifar}%
\end{table*}%

\begin{table}[htbp]
 \centering
 \caption{The clean and adversarial accuracy on the Traffic Signs dataset. The results and the standard deviation in the bracket are obtained from 3 repetitions.}  
\resizebox{0.48\textwidth}{!}{
    \begin{tabular}{lcc}
    \toprule  & \multicolumn{1}{l}{clean acc} & \multicolumn{1}{l}{rob acc} \\
     \midrule
  No additional data & 78.52\% (0.16\%)	& 46.03\% (0.85\%) \\
  DDPM & 86.79\% (0.12\%) & 56.01\% (0.14\%) \\
  Contrastive-DP & \textbf{86.94\% (0.32\%)}	& \textbf{56.27\% (0.25\%)} \\
  \bottomrule
  \end{tabular}
  }%
 \label{tab:gtsrb}%
\end{table}%

To visualize the effectiveness of the guidance, we use the same initialization to generate images by DDIM and Contrastive-DP. We find the guidance of the contrastive loss changes the category of the synthetic images or makes the synthetic images realistic (colorful). 

% \subsection{Comparison of the image generated by Contrastive-DP with the vanilla DDIM}
% In this subsection, we visualize the image generated by Contrastive-DP and the vanilla DDIM on the CIFAR-10 dataset. We find the guidance of the contrastive loss changes the category of the synthetic images or makes the synthetic images realistic (colorful). 
%\yd{I try to demonstrate 1.Contrastive-DP algorithm change the category of the images 2. the images generated by Contrastive-DP is more distinguish---how to show the second advantage}

\begin{figure}
\centering
  \subfigure[DDIM]{	\includegraphics[width=.23\textwidth]{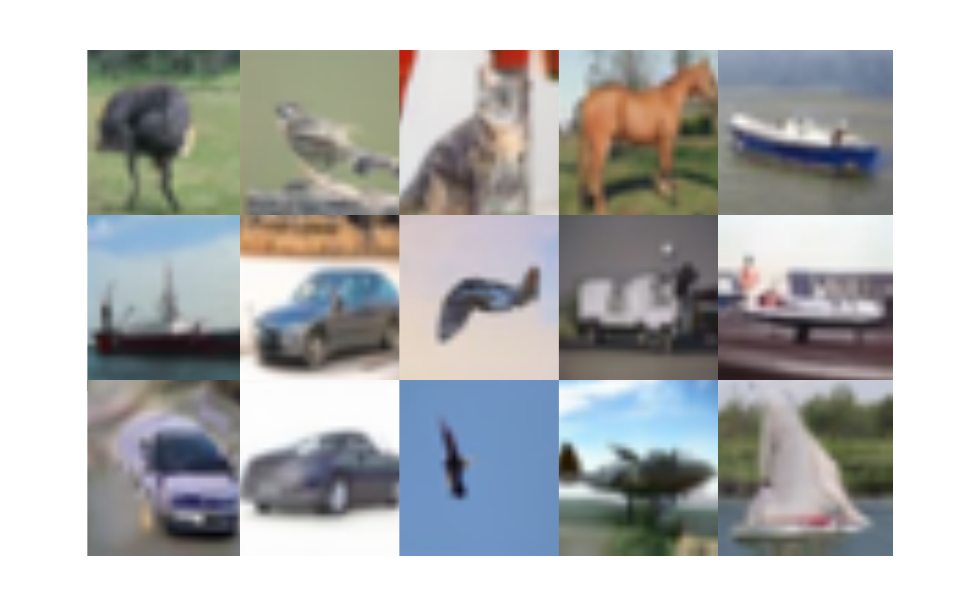}}
\subfigure[Contrasrive-DP]{\includegraphics[width=.23\textwidth]{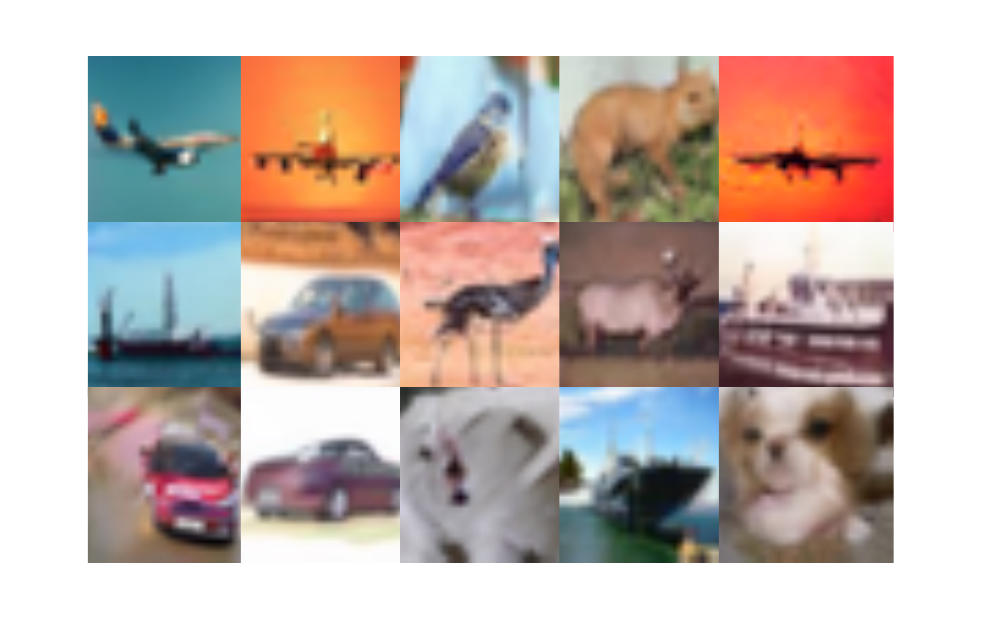}}
  \caption{Comparison of the Contrastive-DP with the vanilla DDIM. The image in the same position on subfigures (a) and (b) has the same initialization. 
  % With the guidance of the contrastive loss, the category of the synthetic images changes, or the synthetic images become more realistic (colorful), which demonstrates the effectiveness of our Contrastive-DP algorithm.
  %\yd{I comment them}
  }
  \label{visualization}
\end{figure}

Moreover, we visualize the t-SNE of the finial classifier learned on different synthetic data. We find with the guidance of the contrastive loss, the final classifier learns a better representation that makes the feature of the images from different classes more separable than the final classifier learned on the images generated by the vanilla DDIM.

% \begin{figure}[h]
% 	\centering
% 	\subfigure[DDIM]{	\includegraphics[width=.45\textwidth]{fig/cifar_.pdf}}
% \subfigure[Contrasrive-DP]{\includegraphics[width=.45\textwidth]{fig/cifar_HNM_SNIP.pdf}}

% 	\caption{A comparison of the t-SNE of different synthetic data. }
% 	\label{contrastive-tsne}
% \end{figure}

% \yd{I will regenerate these figures --- changing the transparency}
\begin{figure}[h]
	\centering
% 	\subfigure[CIFAR (DDIM)]{	\includegraphics[width=.45\textwidth]{fig/cifar.pdf}}
% \subfigure[DDIM]{\includegraphics[width=.45\textwidth]{fig/ddim1M_latex_entropy_50k.pdf}}
% \\
% \subfigure[DDIM + InfoNCE (Our model)]{\includegraphics[width=.45\textwidth]{fig/infoNCE_true_2M_lanet_entropy_50k_trained_model.pdf}}
% \subfigure[DDIM + InfoNCE]{	\includegraphics[width=.45\textwidth]{fig/infoNCE_true_2M_lanet_entropy_50k.pdf}}
% \\
% \subfigure[conditional (Our model)]{	\includegraphics[width=.45\textwidth]{fig/conditional_contrastive_hard_negative_alpha=2000000_with_x_0_2M_lanet_entropy_50k_trained_model.pdf}}
% \subfigure[conditional]{	\includegraphics[width=.45\textwidth]{fig/conditional_contrastive_hard_negative_alpha=2000000_with_x_0_2M_lanet_entropy_50k.pdf}}
\subfigure[DDIM]{	\includegraphics[width=.23\textwidth]{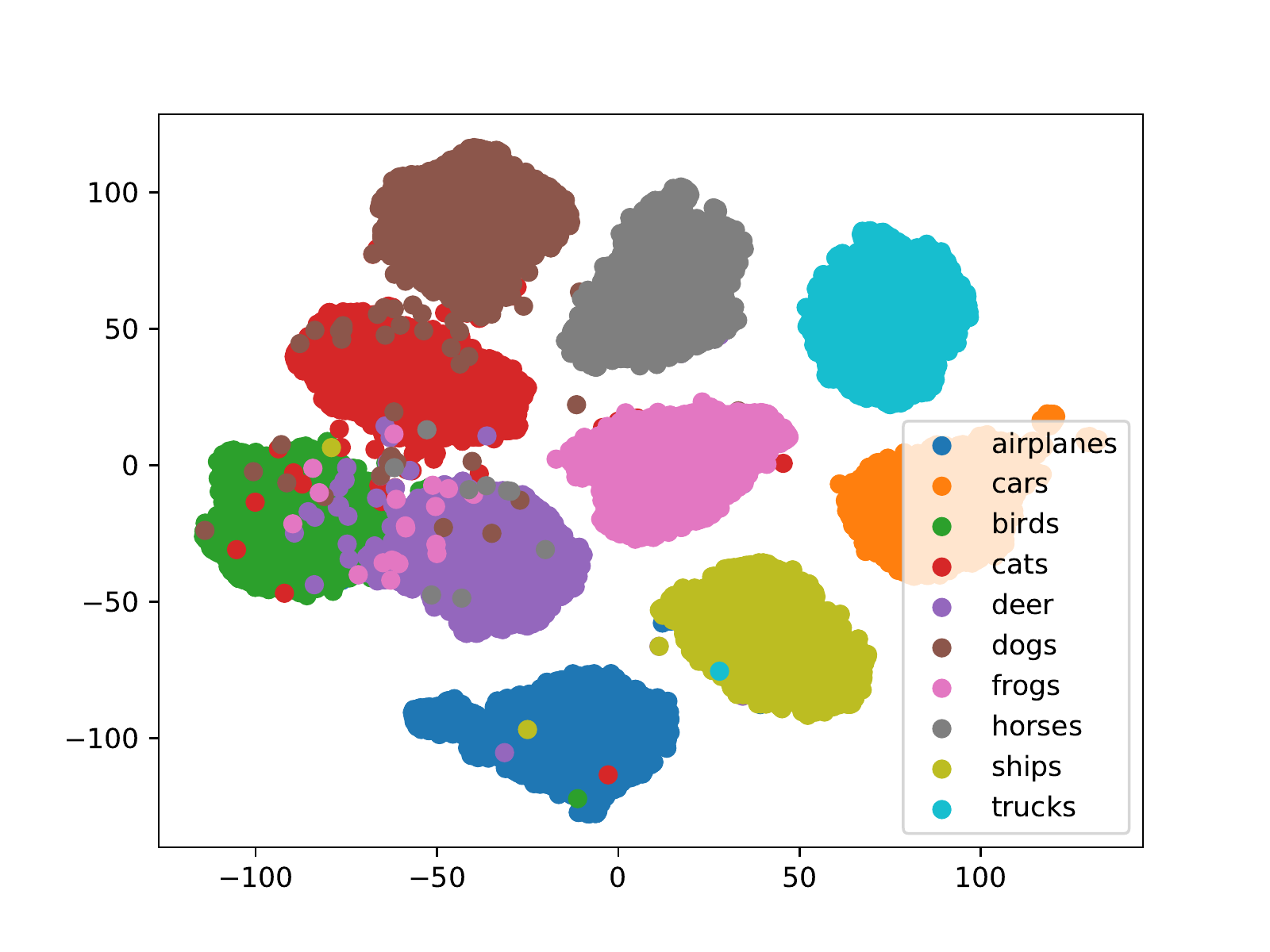}}
\subfigure[Contrastive-DP]{	\includegraphics[width=.23\textwidth]{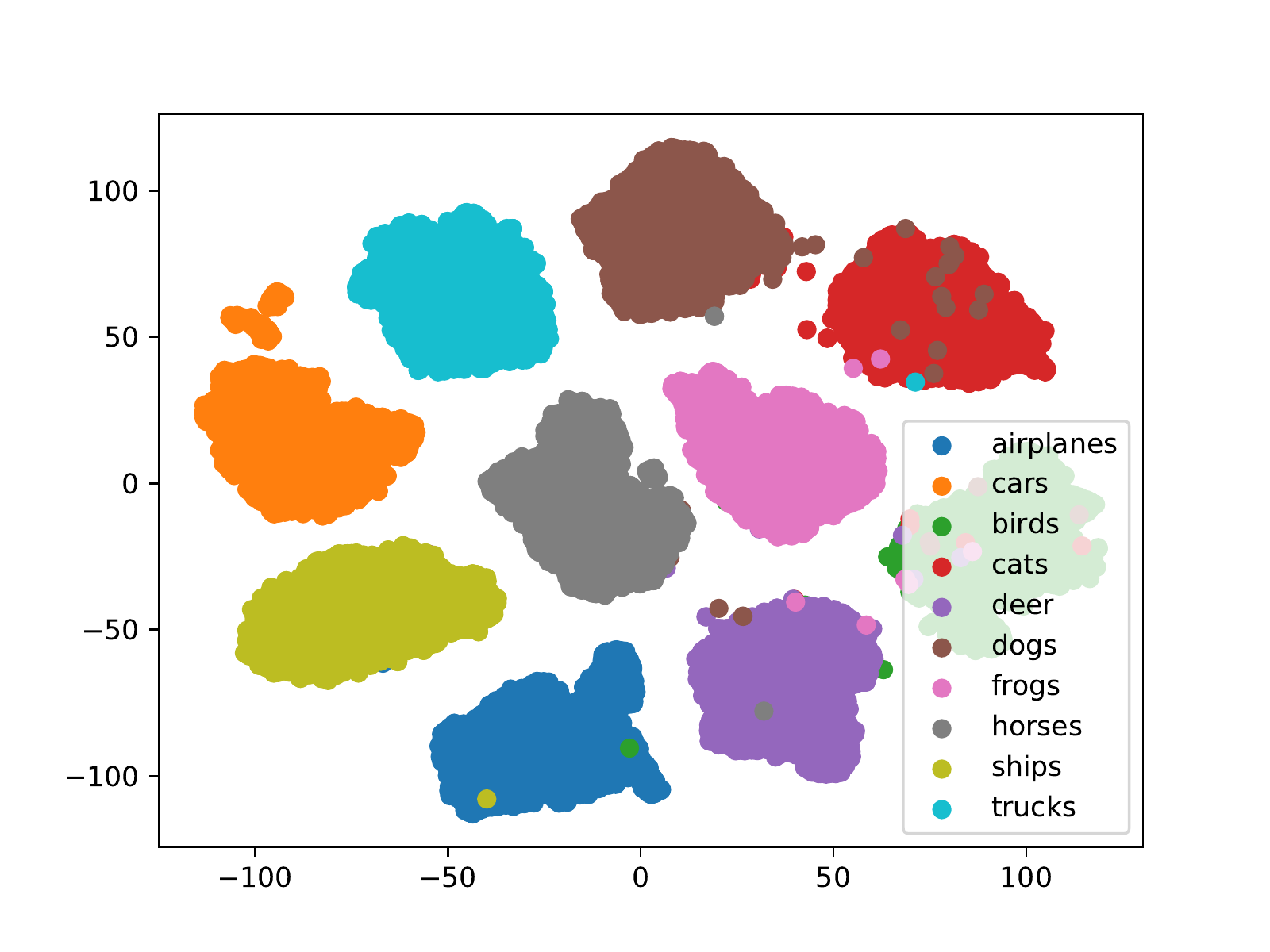}}

	\caption{A comparison of the T-SNE of the final classifier learned on different synthetic data on the CIFAR-10 dataset.}
	\label{contrastive-tsne}
\end{figure} 

\subsection{Ablation Studies} \label{sec:ablation}

In this subsection, we investigate the effectiveness of seven kinds of contrastive loss, the effect of strength of the contrastive loss, and four proposed selection criterion for choosing more informative data from synthetic data. Due to the space limit, we refer to Appendix \ref{ab-contrastive}, \ref{ab-strength}, and \ref{data-selection} for the detailed numerical results, respectively.

% \begin{table}[htbp]
%  \centering
%  \caption{The experimental results on DDPM model.}
%  \vskip 0.15in
%   \begin{tabular}{lrr}
% \toprule     & \multicolumn{1}{l}{clean acc} & \multicolumn{1}{l}{adv acc} \\
% \midrule
%   DDPM 50k 0.7 deepmind & 81.02\% & 53.62\% \\
%   DDPM 200K deepmind & 84.93\% & 57.61\% \\
%   DDPM 1M deepmind & 85.14\% & 58.67\% \\
%   DDPM 1M deepmind reported & 85.98\% & 60.73\% \\
%   \midrule
%   DDPM 50k 0.3 regenerate & \textbf{84.28\%} & 55.08\% \\
%   DDPM 200K regenerate & 85.45\%(0.56\%) & \textbf{55.90\%} \\
%   DDPM 1M regenerate & \textbf{86.86\%(0.04\%)} & \textbf{56.94\%(0.16\%)} \\
%   DDPM+hard negative mining 50k 0.3 & 84.07\% & \textbf{55.54\%} \\
%   DDPM+hard negative mining 200K 0.7 & \textbf{86.58\%} & 55.44\% \\
%   DDPM+hard negative mining 1M 0.7 & 86.22\% & 56.93\% \\
%     \bottomrule
%   \end{tabular}%
%  \label{tab:ablation-ddpm}%
% \end{table}%
 
\section{Related Work} 

Using generative models to improve adversarial robustness has attracted increasing attention recently. \cite{gowal2021improving} uses 100M high-quality images generated by DDPM together with the original training set to achieve state-of-the-art performance on the CIFAR-10 dataset. They propose to use Complementary as an important metric for measuring the efficacy of the synthetic data. In \cite{sehwag2021improving}, it was claimed that the transferability of adversarial robustness between two data distributions is measured by conditional Wasserstein distance, which inspires us to use it as a criterion for selecting samples. Our work follows the same line, but we investigate how to generate the samples with high information rather than applying the selection to the data generated by the vanilla diffusion model. Below we also summarize some closely related work in different lines. 

% \vspace{-0.1in}
% \paragraph{Sample-efficient training.}
% Sample-efficient training algorithms aim to improve the performance of the model by using less training data. Data augmentation methods are used for generating new training data based on original data.
% Active learning methods \cite{Houlsby2011BayesianAL,Gal2017DeepBA,Sinha2019VariationalAA,Kim2021TaskAwareVA} use the acquisition function to evaluate the information of training samples and use this criterion to select training data. 
% \vspace{-0.1in}
\paragraph{Sample-Efficient Generation.} We can view the sample-efficient generation problem as a Bi-level optimization problem. We can regard how to synthesize data as the meta objective and the performance of the model trained on the synthetic data as the inner objective. % We can use the performance on the downstream tasks to guide how to generate data. 
For data-augmentation based methods, \cite{Ruiz2019LearningTS} adopt a reinforcement learning based method for optimizing the generator in order to maximize the training accuracy. 
For active learning based methods, \cite{Tran2019BayesianGA} use an Auto-Encoder to generate new samples based on the informative training data selected by the acquisition function. Besides, \cite{Kim2020LADALD} combines the active learning criterion with data augmentation methods. They use the gradient of acquisition function after one-step augmentation as guidance for training the augmentation policy network.

% Contrastive learning for GAN 
% However, as far as we know, there is no work showing why we should use the contrastive learning loss during the diffusion process and how to design the contrastive loss for the diffusion model.
% \vspace{-0.1in}
\paragraph{Theoretical Analysis of Adversarial Robustness.}
In \cite{Schmidt2018AdversariallyRG}, the sample complexity of adversarial robustness has been shown to be substantially larger than standard classification tasks in the Gaussian setting. \cite{Carmon2019UnlabeledDI} bridges this gap by using the self-training paradigm and corresponding unlabeled data. \cite{Deng2021ImprovingAR} further extends the aforementioned conclusion by leveraging out-of-domain unlabeled data. However, there still lacks analysis on the optimal distribution for synthetic data and the corresponding generation algorithm. %\yd{Rightness}
% \vspace{-0.1in}
\paragraph{Contrastive Learning.}
Contrastive learning algorithms have been widely used for representation learning \citep{Chen2020ASF,He2020MomentumCF,Grill2020BootstrapYO}. The vanilla contrastive learning loss, InfoNCE \citep{Oord2018RepresentationLW}, aims to draw the distance between positive pairs and push the negative pairs away. To mitigate the problem that not all negative pairs may be true negatives, the negative hard mining criterion was proposed in \citep{Chuang2020DebiasedCL, Robinson2021ContrastiveLW}.

\section{Conclusion and Discussion}

We delve into which kind of synthetic distribution is optimal for the downstream task, especially for achieving adversarial robustness in image data classification. We derive the optimality condition under the Gaussian setting and propose the Contrastive-guided Diffusion Process (Contrastive-DP), a plug-in algorithm suitable for various types of diffusion models. We verify our theoretical results on the simulated Gaussian example and demonstrate the superiority of the Contrastive-DP algorithm on real image datasets. 

It would also be interesting to study the theoretical guarantee of the contrastive-guided diffusion process from the perspective of optimal control. We believe that the proposed plug-in type algorithm can also be generalized to loss functions other than contrastive loss, such as the acquisition function in active learning, for other downstream tasks. 

% \section{Acknowledgement}
% Yidong Ouyang and Liyan Xie are partially supported by UDF01002142 through The Chinese University of Hong Kong, Shenzhen, and Grant J00220220004 through Shenzhen Research Institute of Big Data.

\section{Acknowledgement}
Yidong Ouyang and Liyan Xie are partially supported by UDF01002142 through The Chinese University of Hong Kong, Shenzhen, and Grant J00220220004 through Shenzhen Research Institute of Big Data. Guang Cheng is partially supported by Office of Naval Research, ONR (N00014-22-1-2680), NSF – SCALE MoDL (2134209), and Meta gift fund.

%\clearpage
% In the unusual situation where you want a paper to appear in the
% references without citing it in the main text, use \nocite
% \nocite{langley00}
\bibliographystyle{icml2023}
\bibliography{icml2023}

\begin{thebibliography}{54}
\providecommand{\natexlab}[1]{#1}
\providecommand{\url}[1]{\texttt{#1}}
\expandafter\ifx\csname urlstyle\endcsname\relax
  \providecommand{\doi}[1]{doi: #1}\else
  \providecommand{\doi}{doi: \begingroup \urlstyle{rm}\Url}\fi

\bibitem[Bao et~al.(2022)Bao, Li, Zhu, and Zhang]{Bao2022AnalyticDPMAA}
Bao, F., Li, C., Zhu, J., and Zhang, B.
\newblock Analytic-dpm: an analytic estimate of the optimal reverse variance in
  diffusion probabilistic models.
\newblock In \emph{ICLR}, 2022.

\bibitem[Cao et~al.(2022)Cao, Tan, Gao, Chen, Heng, and Li]{cao2022survey}
Cao, H., Tan, C., Gao, Z., Chen, G., Heng, P.-A., and Li, S.~Z.
\newblock A survey on generative diffusion model.
\newblock \emph{arXiv preprint arXiv:2209.02646}, 2022.

\bibitem[Carmon et~al.(2019)Carmon, Raghunathan, Schmidt, Liang, and
  Duchi]{Carmon2019UnlabeledDI}
Carmon, Y., Raghunathan, A., Schmidt, L., Liang, P., and Duchi, J.~C.
\newblock Unlabeled data improves adversarial robustness.
\newblock In \emph{NeurIPS}, 2019.

\bibitem[Chen et~al.(2020)Chen, Kornblith, Norouzi, and Hinton]{Chen2020ASF}
Chen, T., Kornblith, S., Norouzi, M., and Hinton, G.~E.
\newblock A simple framework for contrastive learning of visual
  representations.
\newblock \emph{ArXiv}, abs/2002.05709, 2020.

\bibitem[Chuang et~al.(2020)Chuang, Robinson, Lin, Torralba, and
  Jegelka]{Chuang2020DebiasedCL}
Chuang, C.-Y., Robinson, J., Lin, Y.-C., Torralba, A., and Jegelka, S.
\newblock Debiased contrastive learning.
\newblock In \emph{NeurIPS}, 2020.

\bibitem[Croce \& Hein(2020)Croce and Hein]{Croce2020ReliableEO}
Croce, F. and Hein, M.
\newblock Reliable evaluation of adversarial robustness with an ensemble of
  diverse parameter-free attacks.
\newblock \emph{ArXiv}, abs/2003.01690, 2020.

\bibitem[Das et~al.(2022)Das, Tran, Singh, Yue, Tison, Sangiovanni-Vincentelli,
  and Spanos]{Das2022ConditionalSD}
Das, H.~P., Tran, R., Singh, J., Yue, X., Tison, G.~H.,
  Sangiovanni-Vincentelli, A.~L., and Spanos, C.~J.
\newblock Conditional synthetic data generation for robust machine learning
  applications with limited pandemic data.
\newblock In \emph{AAAI}, 2022.

\bibitem[Deng et~al.(2021)Deng, Zhang, Ghorbani, and Zou]{Deng2021ImprovingAR}
Deng, Z., Zhang, L., Ghorbani, A., and Zou, J.~Y.
\newblock Improving adversarial robustness via unlabeled out-of-domain data.
\newblock In \emph{AISTATS}, 2021.

\bibitem[Dhariwal \& Nichol(2021)Dhariwal and Nichol]{Dhariwal2021DiffusionMB}
Dhariwal, P. and Nichol, A.
\newblock Diffusion models beat gans on image synthesis.
\newblock \emph{ArXiv}, abs/2105.05233, 2021.

\bibitem[Fan et~al.(2021)Fan, Liu, Chen, Zhang, and Gan]{Fan2021WhenDC}
Fan, L., Liu, S., Chen, P.-Y., Zhang, G., and Gan, C.
\newblock When does contrastive learning preserve adversarial robustness from
  pretraining to finetuning?
\newblock In \emph{Neural Information Processing Systems}, 2021.

\bibitem[Gowal et~al.(2020)Gowal, Qin, Uesato, Mann, and
  Kohli]{Gowal2020UncoveringTL}
Gowal, S., Qin, C., Uesato, J., Mann, T.~A., and Kohli, P.
\newblock Uncovering the limits of adversarial training against norm-bounded
  adversarial examples.
\newblock \emph{ArXiv}, abs/2010.03593, 2020.

\bibitem[Gowal et~al.(2021)Gowal, Rebuffi, Wiles, Stimberg, Calian, and
  Mann]{gowal2021improving}
Gowal, S., Rebuffi, S.-A., Wiles, O., Stimberg, F., Calian, D.~A., and Mann,
  T.~A.
\newblock Improving robustness using generated data.
\newblock In \emph{NeurIPS}, 2021.

\bibitem[Grill et~al.(2020)Grill, Strub, Altch'e, Tallec, Richemond,
  Buchatskaya, Doersch, Pires, Guo, Azar, Piot, Kavukcuoglu, Munos, and
  Valko]{Grill2020BootstrapYO}
Grill, J.-B., Strub, F., Altch'e, F., Tallec, C., Richemond, P.~H.,
  Buchatskaya, E., Doersch, C., Pires, B.~{\'A}., Guo, Z.~D., Azar, M.~G.,
  Piot, B., Kavukcuoglu, K., Munos, R., and Valko, M.
\newblock Bootstrap your own latent: A new approach to self-supervised
  learning.
\newblock \emph{ArXiv}, abs/2006.07733, 2020.

\bibitem[He et~al.(2020)He, Fan, Wu, Xie, and Girshick]{He2020MomentumCF}
He, K., Fan, H., Wu, Y., Xie, S., and Girshick, R.~B.
\newblock Momentum contrast for unsupervised visual representation learning.
\newblock \emph{2020 IEEE/CVF Conference on Computer Vision and Pattern
  Recognition (CVPR)}, pp.\  9726--9735, 2020.

\bibitem[Hendrycks \& Gimpel(2016)Hendrycks and
  Gimpel]{Hendrycks2016GaussianEL}
Hendrycks, D. and Gimpel, K.
\newblock Gaussian error linear units (gelus).
\newblock \emph{arXiv: Learning}, 2016.

\bibitem[Ho et~al.(2020)Ho, Jain, and Abbeel]{Ho2020DDPM}
Ho, J., Jain, A., and Abbeel, P.
\newblock Denoising diffusion probabilistic models.
\newblock In \emph{NeurIPS}, 2020.

\bibitem[Houben et~al.(2013)Houben, Stallkamp, Salmen, Schlipsing, and
  Igel]{Houben-IJCNN-2013}
Houben, S., Stallkamp, J., Salmen, J., Schlipsing, M., and Igel, C.
\newblock Detection of traffic signs in real-world images: The {G}erman
  {T}raffic {S}ign {D}etection {B}enchmark.
\newblock In \emph{International Joint Conference on Neural Networks}, number
  1288, 2013.

\bibitem[Izmailov et~al.(2018)Izmailov, Podoprikhin, Garipov, Vetrov, and
  Wilson]{Izmailov2018AveragingWL}
Izmailov, P., Podoprikhin, D., Garipov, T., Vetrov, D.~P., and Wilson, A.~G.
\newblock Averaging weights leads to wider optima and better generalization.
\newblock \emph{ArXiv}, abs/1803.05407, 2018.

\bibitem[Jammalamadaka(2011)]{Jammalamadaka2011DirectionalSI}
Jammalamadaka, S.~R.
\newblock Directional statistics, i.
\newblock 2011.

\bibitem[Kim et~al.(2020{\natexlab{a}})Kim, Tack, and
  Hwang]{Kim2020AdversarialSC}
Kim, M., Tack, J., and Hwang, S.~J.
\newblock Adversarial self-supervised contrastive learning.
\newblock \emph{ArXiv}, abs/2006.07589, 2020{\natexlab{a}}.

\bibitem[Kim et~al.(2020{\natexlab{b}})Kim, Song, Jang, and
  Moon]{Kim2020LADALD}
Kim, Y.-Y., Song, K., Jang, J., and Moon, I.-C.
\newblock Lada: Look-ahead data acquisition via augmentation for active
  learning.
\newblock \emph{ArXiv}, abs/2011.04194, 2020{\natexlab{b}}.

\bibitem[Krizhevsky(2009)]{Krizhevsky2009LearningML}
Krizhevsky, A.
\newblock Learning multiple layers of features from tiny images.
\newblock 2009.

\bibitem[LeCun et~al.(1998)LeCun, Bottou, Bengio, and
  Haffner]{LeCun1998GradientbasedLA}
LeCun, Y., Bottou, L., Bengio, Y., and Haffner, P.
\newblock Gradient-based learning applied to document recognition.
\newblock \emph{Proc. IEEE}, 86:\penalty0 2278--2324, 1998.

\bibitem[Liu(2017)]{Liu2017SteinVG}
Liu, Q.
\newblock Stein variational gradient descent as gradient flow.
\newblock In \emph{NIPS}, 2017.

\bibitem[Liu \& Wang(2016)Liu and Wang]{Liu2016SteinVG}
Liu, Q. and Wang, D.
\newblock Stein variational gradient descent: A general purpose bayesian
  inference algorithm.
\newblock In \emph{NIPS}, 2016.

\bibitem[Ma et~al.(2022)Ma, Zhang, Zhu, Zhang, and Feng]{ma2022accelerating}
Ma, H., Zhang, L., Zhu, X., Zhang, J., and Feng, J.
\newblock Accelerating score-based generative models for high-resolution image
  synthesis.
\newblock \emph{arXiv preprint arXiv:2206.04029}, 2022.

\bibitem[Madry et~al.(2017)Madry, Makelov, Schmidt, Tsipras, and
  Vladu]{Madry2017TowardsDL}
Madry, A., Makelov, A., Schmidt, L., Tsipras, D., and Vladu, A.
\newblock Towards deep learning models resistant to adversarial attacks.
\newblock \emph{ArXiv}, abs/1706.06083, 2017.

\bibitem[Nichol \& Dhariwal(2021)Nichol and Dhariwal]{nichol2021improved}
Nichol, A.~Q. and Dhariwal, P.
\newblock Improved denoising diffusion probabilistic models.
\newblock In \emph{International Conference on Machine Learning}, pp.\
  8162--8171. PMLR, 2021.

\bibitem[Nie et~al.(2022)Nie, Guo, Huang, Xiao, Vahdat, and
  Anandkumar]{Nie2022DiffusionMF}
Nie, W., Guo, B., Huang, Y., Xiao, C., Vahdat, A., and Anandkumar, A.
\newblock Diffusion models for adversarial purification.
\newblock In \emph{International Conference on Machine Learning}, 2022.

\bibitem[O'Kelly et~al.(2018)O'Kelly, Sinha, Namkoong, Duchi, and
  Tedrake]{OKelly2018ScalableEA}
O'Kelly, M., Sinha, A., Namkoong, H., Duchi, J.~C., and Tedrake, R.
\newblock Scalable end-to-end autonomous vehicle testing via rare-event
  simulation.
\newblock In \emph{NeurIPS}, 2018.

\bibitem[Robinson et~al.(2021)Robinson, Chuang, Sra, and
  Jegelka]{Robinson2021ContrastiveLW}
Robinson, J., Chuang, C.-Y., Sra, S., and Jegelka, S.
\newblock Contrastive learning with hard negative samples.
\newblock In \emph{ICLR}, 2021.

\bibitem[Rombach et~al.(2022)Rombach, Blattmann, Lorenz, Esser, and
  Ommer]{rombach2022high}
Rombach, R., Blattmann, A., Lorenz, D., Esser, P., and Ommer, B.
\newblock High-resolution image synthesis with latent diffusion models.
\newblock In \emph{Proceedings of the IEEE/CVF Conference on Computer Vision
  and Pattern Recognition}, pp.\  10684--10695, 2022.

\bibitem[Ruiz et~al.(2019)Ruiz, Schulter, and Chandraker]{Ruiz2019LearningTS}
Ruiz, N., Schulter, S., and Chandraker, M.
\newblock Learning to simulate.
\newblock \emph{ArXiv}, abs/1810.02513, 2019.

\bibitem[Salimans \& Ho(2022)Salimans and Ho]{Salimans2022ProgressiveDF}
Salimans, T. and Ho, J.
\newblock Progressive distillation for fast sampling of diffusion models.
\newblock In \emph{ICLR}, 2022.

\bibitem[Schmidt et~al.(2018)Schmidt, Santurkar, Tsipras, Talwar, and
  Madry]{Schmidt2018AdversariallyRG}
Schmidt, L., Santurkar, S., Tsipras, D., Talwar, K., and Madry, A.
\newblock Adversarially robust generalization requires more data.
\newblock In \emph{NeurIPS}, 2018.

\bibitem[Sehwag et~al.(2022)Sehwag, Mahloujifar, Handina, Dai, Xiang, Chiang,
  and Mittal]{sehwag2021improving}
Sehwag, V., Mahloujifar, S., Handina, T., Dai, S., Xiang, C., Chiang, M., and
  Mittal, P.
\newblock Robust learning meets generative models: Can proxy distributions
  improve adversarial robustness?
\newblock In \emph{ICLR}, 2022.

\bibitem[Shi \& Mackey(2022)Shi and Mackey]{Shi2022AFC}
Shi, J.~H. and Mackey, L.~W.
\newblock A finite-particle convergence rate for stein variational gradient
  descent.
\newblock \emph{ArXiv}, abs/2211.09721, 2022.

\bibitem[Sohl-Dickstein et~al.(2015)Sohl-Dickstein, Weiss, Maheswaranathan, and
  Ganguli]{sohl2015deep}
Sohl-Dickstein, J., Weiss, E., Maheswaranathan, N., and Ganguli, S.
\newblock Deep unsupervised learning using nonequilibrium thermodynamics.
\newblock In \emph{International Conference on Machine Learning}, pp.\
  2256--2265. PMLR, 2015.

\bibitem[Song et~al.(2021{\natexlab{a}})Song, Meng, and
  Ermon]{Song2021DenoisingDI}
Song, J., Meng, C., and Ermon, S.
\newblock Denoising diffusion implicit models.
\newblock In \emph{ICLR}, 2021{\natexlab{a}}.

\bibitem[Song \& Ermon(2019)Song and Ermon]{song2019generative}
Song, Y. and Ermon, S.
\newblock Generative modeling by estimating gradients of the data distribution.
\newblock \emph{Advances in Neural Information Processing Systems}, 32, 2019.

\bibitem[Song \& Ermon(2020)Song and Ermon]{song2020improved}
Song, Y. and Ermon, S.
\newblock Improved techniques for training score-based generative models.
\newblock \emph{Advances in neural information processing systems},
  33:\penalty0 12438--12448, 2020.

\bibitem[Song et~al.(2020)Song, Garg, Shi, and Ermon]{song2020sliced}
Song, Y., Garg, S., Shi, J., and Ermon, S.
\newblock Sliced score matching: A scalable approach to density and score
  estimation.
\newblock In \emph{Uncertainty in Artificial Intelligence}, pp.\  574--584.
  PMLR, 2020.

\bibitem[Song et~al.(2021{\natexlab{b}})Song, Sohl-Dickstein, Kingma, Kumar,
  Ermon, and Poole]{Song2021ScoreBasedGM}
Song, Y., Sohl-Dickstein, J.~N., Kingma, D.~P., Kumar, A., Ermon, S., and
  Poole, B.
\newblock Score-based generative modeling through stochastic differential
  equations.
\newblock \emph{ArXiv}, abs/2011.13456, 2021{\natexlab{b}}.

\bibitem[Sun et~al.(2022)Sun, Nie, Yu, Mao, and Xiao]{Sun2022PointDPDP}
Sun, J., Nie, W., Yu, Z., Mao, Z.~M., and Xiao, C.
\newblock Pointdp: Diffusion-driven purification against adversarial attacks on
  3d point cloud recognition.
\newblock \emph{ArXiv}, abs/2208.09801, 2022.

\bibitem[Tran et~al.(2019)Tran, Do, Reid, and Carneiro]{Tran2019BayesianGA}
Tran, T., Do, T.-T., Reid, I.~D., and Carneiro, G.
\newblock Bayesian generative active deep learning.
\newblock \emph{ArXiv}, abs/1904.11643, 2019.

\bibitem[Tsipras et~al.(2018)Tsipras, Santurkar, Engstrom, Turner, and
  Madry]{Tsipras2018RobustnessMB}
Tsipras, D., Santurkar, S., Engstrom, L., Turner, A., and Madry, A.
\newblock Robustness may be at odds with accuracy.
\newblock \emph{arXiv: Machine Learning}, 2018.

\bibitem[van~den Oord et~al.(2018)van~den Oord, Li, and
  Vinyals]{Oord2018RepresentationLW}
van~den Oord, A., Li, Y., and Vinyals, O.
\newblock Representation learning with contrastive predictive coding.
\newblock \emph{ArXiv}, abs/1807.03748, 2018.

\bibitem[Wang et~al.(2019)Wang, Xie, Li, Fonseca, and
  Tian]{Wang2019SampleEfficientNA}
Wang, L., Xie, S., Li, T., Fonseca, R., and Tian, Y.
\newblock Sample-efficient neural architecture search by learning action space.
\newblock \emph{ArXiv}, abs/1906.06832, 2019.

\bibitem[Watson et~al.(2022)Watson, Chan, Ho, and
  Norouzi]{Watson2022LearningFS}
Watson, D., Chan, W., Ho, J., and Norouzi, M.
\newblock Learning fast samplers for diffusion models by differentiating
  through sample quality.
\newblock In \emph{ICLR}, 2022.

\bibitem[Wei et~al.(2020)Wei, Shen, Chen, and Ma]{wei2020theoretical}
Wei, C., Shen, K., Chen, Y., and Ma, T.
\newblock Theoretical analysis of self-training with deep networks on unlabeled
  data.
\newblock In \emph{International Conference on Learning Representations}, 2020.

\bibitem[Yang et~al.(2022)Yang, Zhang, and Hong]{yang2022diffusion}
Yang, L., Zhang, Z., and Hong, S.
\newblock Diffusion models: A comprehensive survey of methods and applications.
\newblock \emph{arXiv preprint arXiv:2209.00796}, 2022.

\bibitem[Zagoruyko \& Komodakis(2016)Zagoruyko and
  Komodakis]{Zagoruyko2016WideRN}
Zagoruyko, S. and Komodakis, N.
\newblock Wide residual networks.
\newblock \emph{ArXiv}, abs/1605.07146, 2016.

\bibitem[Zhang et~al.(2022)Zhang, Zhang, Zhang, Niu, Feng, Yoo, and
  Kweon]{Zhang2022DecoupledAC}
Zhang, C., Zhang, K., Zhang, C., Niu, A., Feng, J., Yoo, C.~D., and Kweon,
  I.-S.
\newblock Decoupled adversarial contrastive learning for self-supervised
  adversarial robustness.
\newblock \emph{ArXiv}, abs/2207.10899, 2022.

\bibitem[Zhang et~al.(2019)Zhang, Yu, Jiao, Xing, Ghaoui, and
  Jordan]{Zhang2019TheoreticallyPT}
Zhang, H.~R., Yu, Y., Jiao, J., Xing, E.~P., Ghaoui, L.~E., and Jordan, M.~I.
\newblock Theoretically principled trade-off between robustness and accuracy.
\newblock \emph{ArXiv}, abs/1901.08573, 2019.

\end{thebibliography}

\newpage
\appendix
\onecolumn

\section{Theoretical Details for Section \ref{sec:theory}}

\subsection{Error probabilities in closed form.}

Here, we briefly recapitulate the closed-form formulation for the standard and robust error probabilities as detailed in \cite{Carmon2019UnlabeledDI,Deng2021ImprovingAR}.

The standard error probability can be written as
\begin{align}
\mathrm{err}_{\text {standard }}(f_{\boldsymbol{\theta}})=\mathbb{P}\left(y \cdot \boldsymbol{x}^{\top} \boldsymbol{\theta}<0\right)=\mathbb{P}\left(\mathcal{N}\left(\frac{\boldsymbol{\mu}^{\top} \boldsymbol{\theta}}{\sigma\|\boldsymbol{\theta}\|}, 1\right)<0\right)= Q\left(\frac{\boldsymbol{\mu}^{\top} \boldsymbol{\theta}}{\sigma\|\boldsymbol{\theta}\|}\right),
\label{eq-standard-error}
\end{align}
where
$$
Q(x)=\frac{1}{\sqrt{2 \pi}} \int_{x}^{\infty} e^{-t^{2} / 2} d t
$$
is the Gaussian error function and is non-increasing. Clearly the standard error probability is minimized when $\frac{\boldsymbol{\theta}}{\left\Vert \boldsymbol{\theta}\right\Vert}=\frac{\boldsymbol{\mu}}{\left\Vert \boldsymbol{\mu}\right\Vert}$, i.e., $\boldsymbol{\theta} = c\boldsymbol{\mu}$ for some scalar $c>0$. We may impost $\left\Vert \boldsymbol{\theta}\right\Vert_2 = 1$ to ensure the unique solution $\boldsymbol{\theta} = \boldsymbol{\mu}/\left\Vert \boldsymbol{\mu}\right\Vert$.

The robust error probability under the $\ell_\infty$ adversarial set $\Delta=\{\delta:\left\Vert \boldsymbol{\delta} \right\Vert_\infty\leq\epsilon\}$ is
\begin{align}\label{eq-rob-error}
\mathrm{err}_{\text {robust }}^{\infty, \varepsilon}(f_{\boldsymbol{\theta}}) &=\mathbb{P}\left(\inf _{\|\boldsymbol{\nu}\|_{\infty} \leq \varepsilon}\left\{y \cdot(\boldsymbol{x}+\boldsymbol{\nu})^{\top} \boldsymbol{\theta}\right\}<0\right) \notag\\
&=\mathbb{P}\left(y \cdot \boldsymbol{x}^{\top} \boldsymbol{\theta}-\varepsilon\|\boldsymbol{\theta}\|_{1}<0\right)=\mathbb{P}\left(\mathcal{N}\left(\boldsymbol{\mu}^{\top} \boldsymbol{\theta},\sigma^2\|\boldsymbol{\theta}\|^{2}\right)<\varepsilon\|\boldsymbol{\theta}\|_{1}\right) \notag\\
&=Q\left(\frac{\boldsymbol{\mu}^{\top} \boldsymbol{\theta}}{\sigma\|\boldsymbol{\theta}\|}-\frac{\varepsilon\|\boldsymbol{\theta}\|_{1}}{\sigma\|\boldsymbol{\theta}\|}\right).
\end{align}

In the following part, we use a simpler notation $\mathrm{err}_{\text {robust }}(f_{\boldsymbol{\theta}})$ for the robust error $\mathrm{err}_{\text {robust }}^{\infty, \varepsilon}(f_{\boldsymbol{\theta}})$ without ambiguity. The closed-form of the optimal $\boldsymbol{\theta}^*$ that minimizes the above robust error $\mathrm{err}_{\text {robust }}$ can be shown to be \citep{Deng2021ImprovingAR}:
$$
\boldsymbol{\theta}^*=\frac{T_{\varepsilon}(\boldsymbol{\mu})}{\left\|T_{\varepsilon}(\boldsymbol{\mu})\right\|},
$$
where $T_{\varepsilon}(\boldsymbol{\mu}):\mathbb{R}^d\to \mathbb{R}^d$ is the hard-thresholding operator with $\left(T_{\varepsilon}(\boldsymbol{\mu})\right)_{j}=\operatorname{sign}\left(\boldsymbol{\mu}_{j}\right) \cdot \max \left\{\left|\boldsymbol{\mu}_{j}\right|-\varepsilon, 0\right\}$.
Under the mild assumption $\boldsymbol{\mu}_{j}>\varepsilon, \forall j \in \{1,2, \ldots, d\}$, the optimal solution can be simplified as: 
$$\boldsymbol{\theta}^*=\frac{\boldsymbol{\mu}-\varepsilon \mathbf{1}_d}{\|\boldsymbol{\mu}-\varepsilon \mathbf{1}_d\|}.$$

\begin{remark} \label{trade-off}
Note that when $\boldsymbol{\mu}=c\mathbf{1}_d$ for some constant $c>\epsilon$, the optimal solution $\boldsymbol{\theta}^*=\frac{\boldsymbol{\mu}-\varepsilon \mathbf{1}_d}{\|\boldsymbol{\mu}-\varepsilon \mathbf{1}_d\|}$ for minimizing the robust error is the same as the optimal solution $\frac{\boldsymbol{\mu}}{\left\Vert \boldsymbol{\mu}\right\Vert}$ for minimizing the standard error. Otherwise, these two solutions are different, representing a trade-off between robustness and accuracy.
\end{remark}

\subsection{Details for the theoretical analysis in Section \ref{sec:theory}}

%\begin{proof}[Proof of Lemma \ref{theorem1}]\label{Lemma1}

Overall, we would like to design an appropriate synthetic distribution $\widetilde{\mathcal{D}}$ that can help optimize the adversarial classification accuracy in the downstream task. 
First note that by Bayes rule, the optimal decision boundary for the true distribution $\boldsymbol{x}|y\sim \mathcal{N}(y\boldsymbol{\mu},\sigma^2 \mathbb{I})$ is given by $\boldsymbol{\mu}^\top\boldsymbol{x}=0$, i.e., the optimal classifier is parameterized by $\boldsymbol{\theta}=c \boldsymbol{\mu}$ for any $c>0$. Therefore, we restrict our attention to synthetic data distributions that satisfy the following two conditions:
\begin{enumerate}[leftmargin=*]
\item The marginal probability density $p(\tilde{y})$ of the synthetic distribution matches $p(y)$ of the real data distribution well.
\item The conditional probability densities $p(\boldsymbol{\tilde{x}}|\tilde{y}=1)$ and $p(\boldsymbol{\tilde{x}}|\tilde{y}=-1)$ of the synthetic data distribution are symmetric around the true optimal decision boundary $\boldsymbol{\mu}^\top\boldsymbol{x}=0$.
\end{enumerate}
More specifically, we consider a special case of the synthetic data distribution $\widetilde{\mathcal{D}}_\mathcal{X}=0.5\mathcal{N}(\tilde{\boldsymbol{\mu}},\sigma^2 \mathbb{I})+0.5 \mathcal{N}(-\tilde{\boldsymbol{\mu}},\sigma^2 \mathbb{I})$.

\begin{proof}[Proof of Proposition \ref{theorem3}]\label{proof-lemma}

We follow the proof strategy in \cite{Carmon2019UnlabeledDI}.
Let $b_{i}$ be the indicator that the $i$-th pseudo-label $\tilde{y}_{i}$ assigned to $\tilde{\boldsymbol{x}}_{i}$ is incorrect, so that we have $\tilde{\boldsymbol{x}}_{i} \sim \mathcal{N}\left(\left(1-2 b_{i}\right) \tilde{y}_{i} \tilde{\boldsymbol{\mu}}, \sigma^{2} \mathbb{I}\right)$. 
Let
$
\gamma:=\frac{1}{\tilde{n}} \sum_{i=1}^{\bar{n}}\left(1-2 b_{i}\right) \in[-1,1]
$ 
and 
$\alpha:=\frac{\tilde{n}}{\tilde{n}+n}$.
Note that the true data samples $\boldsymbol{x}_{i} \sim \mathcal{N}\left( y_{i} \boldsymbol{\mu}, \sigma^{2} \mathbb{I}\right)$, thus we may write the final estimator as
\[
\begin{aligned}
\hat{\boldsymbol{\theta}}_{\text {final }}
&=\frac{1}{n+\tilde{n}}( \sum_{j=1}^{\tilde{n}} \tilde{y}_j\tilde{\boldsymbol{x}}_j + \sum_{i=1}^n y_i\boldsymbol{x}_i )\\
& = \frac{1}{n+\tilde{n}}( \sum_{j=1}^{\tilde{n}} [\left(1-2 b_{j}\right) \tilde{y}_{j} \tilde{\boldsymbol{\mu}} + \tilde{\boldsymbol{\epsilon_j}}]\cdot \tilde{y}_j + \sum_{i=1}^n [\boldsymbol{\mu} y_{i} + \boldsymbol{\epsilon_i}]\cdot y_i ) \\
&=\alpha \gamma \tilde{\boldsymbol{\mu}}+\frac{1}{n+\tilde{n}} \sum_{i=1}^{\tilde{n}} \tilde{y}_{i} \tilde{\boldsymbol{\varepsilon}}_{i}+(1-\alpha)\boldsymbol{\mu}+\frac{1}{n+\tilde{n}}\sum_{i=1}^n y_i\boldsymbol{\varepsilon}_i\\
&=\alpha \gamma \tilde{\boldsymbol{\mu}}+(1-\alpha)\boldsymbol{\mu}+\frac{1}{n+\tilde{n}}(\sum_{i=1}^{n} y_i\boldsymbol{\varepsilon}_i+\sum_{i=1}^{\tilde{n}} \tilde{y}_i\tilde{\boldsymbol{\varepsilon}}_i),
\end{aligned}
\]
where $\boldsymbol{\varepsilon}_{i},\tilde{\boldsymbol{\varepsilon}}_{i} \sim \mathcal{N}\left(\mathbf{0}, \sigma^{2} \mathbb{I}\right)$ independent of each other, and the marginal probability density $p(\tilde{y})$ matches $p(y)$ well. 
Defining $\tilde{\boldsymbol{\delta}}:=\hat{\boldsymbol{\theta}}_{\text {final }}-\alpha\gamma \tilde{\boldsymbol{\mu}}-(1-\alpha)\boldsymbol{\mu}$. 

By (\ref{eq-standard-error}), we have that the standard error of $f_{\hat{\boldsymbol{\theta}}_{\text {final }}}$ is a non-increasing function of $\frac{\boldsymbol{\mu}^{\top} \hat{\boldsymbol{\theta}}_{\text {final }}}{\sigma\|\hat{\boldsymbol{\theta}}_{\text {final }}\|}$. Note that when $\tilde{n}$ is large enough, we have $\alpha\to 1$ and the direction of $\hat{\boldsymbol{\theta}}_{\text {final }}$ approach the direction of $\tilde{\boldsymbol{\mu}}$. Therefore, the statement in Case 1 holds as a consequence, and similarly for the robust error according to (\ref{eq-rob-error}). 

The remaining proof on Case 2 and Case 3 is based on a detailed discussion for the squared inverse of the term $\frac{\boldsymbol{\mu}^{\top} \hat{\boldsymbol{\theta}}_{\text {final }}}{\sigma\|\hat{\boldsymbol{\theta}}_{\text {final }}\|}$:
\begin{equation}\label{eq:insideQ}
\begin{aligned}
\frac{\|\hat{\boldsymbol{\theta}}_{\text {final }}\|^{2}}{(\boldsymbol{\mu}^{\top} \hat{\boldsymbol{\theta}}_{\text {final }})^{2}} 
&=\frac{\|\tilde{\boldsymbol{\delta}}+\alpha \gamma \tilde{\boldsymbol{\mu}}+(1- \alpha)\boldsymbol{\mu}\|^{2}}{(\alpha\gamma  \langle\boldsymbol{\mu},\tilde{\boldsymbol{\mu}}\rangle+\boldsymbol{\mu}^\top \tilde{\boldsymbol{\delta}}+(1-\alpha)\|\boldsymbol{\mu}\|^2)^{2}}.
%&=\frac{2\alpha(1-\alpha)\gamma \langle\boldsymbol{\mu},\tilde{\boldsymbol{\mu}}\rangle +\|\tilde{\boldsymbol{\delta}}\|^2+\alpha^2 \gamma^2 \|\tilde{\boldsymbol{\mu}}\|^2+(1- \alpha)^2\|\boldsymbol{\mu}\|^{2}+2\alpha \gamma \tilde{\boldsymbol{\mu}}^\top\tilde{\boldsymbol{\delta}}+2(1-\alpha)\boldsymbol{\mu}^\top\tilde{\boldsymbol{\delta}}}{\alpha^2 \gamma^2 \langle\boldsymbol{\mu},\tilde{\boldsymbol{\mu}}\rangle^2+2 \alpha\gamma (\boldsymbol{\mu}^\top\tilde{\boldsymbol{\delta}}+(1- \alpha)\|\boldsymbol{\mu}\|^{2})\langle\boldsymbol{\mu},\tilde{\boldsymbol{\mu}}\rangle+ (\boldsymbol{\mu}^\top \tilde{\boldsymbol{\delta}})^2+(1-\alpha)^2\|\boldsymbol{\mu}\|^4+2(1- \alpha)\boldsymbol{\mu}^\top \tilde{\boldsymbol{\delta}}\|\boldsymbol{\mu}\|^2}
\end{aligned}
\end{equation}
Note that the larger the quantity in (\ref{eq:insideQ}) is, the larger the standard error of $f_{\hat{\boldsymbol{\theta}}_{\text {final }}}$.

{Case 2.} Assume $\tilde{\boldsymbol{\mu}}=c\boldsymbol{\mu}$.
%Since the derivative of (\ref{eq:insideQ}) is negative when $\langle\boldsymbol{\mu},\tilde{\boldsymbol{\mu}}\rangle>0$, $\frac{\|\hat{\boldsymbol{\theta}}_{\text {final }}\|^{2}}{(\boldsymbol{\mu}^{\top} \hat{\boldsymbol{\theta}}_{\text {final }})^{2}}$ achieves its minimum when $\tilde{\boldsymbol{\mu}}=c\boldsymbol{\mu}$. Thus, the standard error $\mathrm{err}_{\text {standard }}(f_{\hat{\boldsymbol{\theta}}_{\text {final }}})$ achieves its minimum, which proves the first part of Case 2.
Then we have (\ref{eq:insideQ}) reduces to:
\begin{align}
\frac{\|\hat{\boldsymbol{\theta}}_{\text {final }}\|^{2}}{(\boldsymbol{\mu}^{\top} \hat{\boldsymbol{\theta}}_{\text {final }})^{2}} &=\frac{\|\tilde{\boldsymbol{\delta}}+(1-\alpha+c\gamma\alpha)\boldsymbol{\mu}\|^{2}}{\left((1-\alpha+c\gamma\alpha)\|\boldsymbol{\mu}\|^{2}+\boldsymbol{\mu}^{\top} \tilde{\boldsymbol{\delta}}\right)^{2}}\\
&=\frac{1}{\|\boldsymbol{\mu}\|^{2}}+\frac{\|\tilde{\boldsymbol{\delta}}+(1-\alpha+c\gamma\alpha)\boldsymbol{\mu}\|^{2}-\frac{1}{\|\boldsymbol{\mu}\|^{2}}\left((1-\alpha+c\gamma\alpha)\|\boldsymbol{\mu}\|^{2}+\boldsymbol{\mu}^{\top} \tilde{\boldsymbol{\delta}}\right)^{2}}{\left((1-\alpha+c\gamma\alpha)\|\boldsymbol{\mu}\|^{2}+\boldsymbol{\mu}^{\top} \tilde{\boldsymbol{\delta}}\right)^{2}} \notag\\
&=\frac{1}{\|\boldsymbol{\mu}\|^{2}}+\frac{\|\tilde{\boldsymbol{\delta}}\|^{2}-\frac{1}{\|\boldsymbol{\mu}\|^{2}}(\boldsymbol{\mu}^{\top} \tilde{\boldsymbol{\delta}})^{2}}{\left((1-\alpha+c\gamma\alpha)\|\boldsymbol{\mu}\|^{2}+\boldsymbol{\mu}^{\top} \tilde{\boldsymbol{\delta}}\right)^2}, \label{eq-tilde-standard}
 \end{align}
which demonstrates that the bigger the $c$ is, the smaller the standard error $\mathrm{err}_{\text {standard }}(f_{\hat{\boldsymbol{\theta}}_{\text{final}}})$ is, which verifies the second part of Case 2.

{Case 3.} Assume $\tilde{\boldsymbol{\mu}}=c (\boldsymbol{\mu} - \varepsilon \mathbf{1}_d)$. Similar to Case 2, we rewrite the term inside the robust error function (\ref{eq-rob-error}) as:
%Similar to the Case 2, the derivative of $\frac{\|\hat{\boldsymbol{\theta}}_{\text {final }}\|^{2}}{(\boldsymbol{\mu}-\varepsilon \mathbf{1}_d)^{\top} \hat{\boldsymbol{\theta}}_{\text {final }})^{2}}$ is negative when $\langle\boldsymbol{\mu}-\varepsilon \mathbf{1}_d,\tilde{\boldsymbol{\mu}}\rangle>0$, achieves its minimum when $\tilde{\boldsymbol{\mu}}=c(\boldsymbol{\mu}-\varepsilon \mathbf{1}_d)$. Thus, the robust error $\mathrm{err}_{\text {robust }}(f_{\hat{\boldsymbol{\theta}}_{\text {final }}})$ achieves its minimum, which proves the first part of Case 3.

%\begin{scriptsize}
\begin{align}
\frac{\|\hat{\boldsymbol{\theta}}_{\text {final }}\|^{2}}{\left((\boldsymbol{\mu}-\varepsilon \mathbf{1}_d)^{\top} \hat{\boldsymbol{\theta}}_{\text {final }}\right)^{2}}
=& \frac{\|\tilde{\boldsymbol{\delta}}+(1-\alpha)\boldsymbol{\mu}+c\gamma\alpha(\boldsymbol{\mu}-\varepsilon \mathbf{1}_d)\|^{2}}{\left( c\gamma\alpha\|\boldsymbol{\mu}-\varepsilon \mathbf{1}_d\|^{2}+ (1-\alpha)\boldsymbol{\mu}^\top(\boldsymbol{\mu}-\varepsilon \mathbf{1}_d) + (\boldsymbol{\mu}-\varepsilon \mathbf{1}_d)^{\top} \tilde{\boldsymbol{\delta}}\right)^{2}} \notag\\
%&= \frac{1}{\|\boldsymbol{\mu}-\varepsilon \mathbf{1}_d\|^{2}}+\frac{\|\tilde{\boldsymbol{\delta}}+(1-\alpha+c\gamma\alpha)(\boldsymbol{\mu}-\varepsilon \mathbf{1}_d)\|^{2}-\frac{1}{\|\boldsymbol{\mu}-\varepsilon \mathbf{1}_d\|^{2}}\left((1-\alpha+c\gamma\alpha)\|\boldsymbol{\mu}-\varepsilon \mathbf{1}_d\|^{2}+(\boldsymbol{\mu}-\varepsilon \mathbf{1}_d)^{\top} \tilde{\boldsymbol{\delta}}\right)^{2}}{\left((1-\alpha+c\gamma\alpha)\|\boldsymbol{\mu}-\varepsilon \mathbf{1}_d\|^{2}+(\boldsymbol{\mu}-\varepsilon \mathbf{1}_d)^{\top} \tilde{\boldsymbol{\delta}}\right)^{2}} \notag\\
\approx & \frac{1}{\|\boldsymbol{\mu}-\varepsilon \mathbf{1}_d\|^{2}}
+
\frac{\|\tilde{\boldsymbol{\delta}}\|^{2}-\frac{1}{\|\boldsymbol{\mu}-\varepsilon \mathbf{1}_d\|^{2}}\left((\boldsymbol{\mu}-\varepsilon \mathbf{1}_d)^{\top} \tilde{\boldsymbol{\delta}}\right)^{2}}{\left(c\gamma\alpha\|\boldsymbol{\mu}-\varepsilon \mathbf{1}_d\|^{2}+(\boldsymbol{\mu}-\varepsilon \mathbf{1}_d)^{\top} \tilde{\boldsymbol{\delta}}\right)^2}, %\\
% =& \frac{1}{\|\boldsymbol{\mu}-\varepsilon \mathbf{1}_d\|^{2}}
% +
% \frac{\|\tilde{\boldsymbol{\delta}}\|^{2}-\frac{1}{\|\boldsymbol{\mu}-\varepsilon \mathbf{1}_d\|^{2}}\left((\boldsymbol{\mu}-\varepsilon \mathbf{1}_d)^{\top} \tilde{\boldsymbol{\delta}}\right)^{2}}{\left((1-\alpha+c\gamma\alpha)\|\boldsymbol{\mu}-\varepsilon \mathbf{1}_d\|^{2}+(\boldsymbol{\mu}-\varepsilon \mathbf{1}_d)^{\top} \tilde{\boldsymbol{\delta}}\right)^2}, \label{eq-tilde-rob}
 \end{align}
%\end{scriptsize}
where the last approximation is due to $\tilde n$ sufficiently large and thus $\alpha\approx 1$. The above equation demonstrates the larger the $c$ is, the smaller the robust error $\mathrm{err}_{\text {robust }}(f_{\hat{\boldsymbol{\theta}}_{\text{final}}})$ is, which proves the second part of Case 3.

\end{proof}

% $\tilde{\mu}=c\mu$ case

\begin{proof}[Proof of Theorem \ref{theorem2}] \label{Lemma2}
We follow the proof strategy in \cite{Carmon2019UnlabeledDI}, with the main difference being twofold: (i) in our final estimate $ \hat{\boldsymbol{\theta}}_{\text {final }}$ depends on both the real data and synthetic data; (ii) our synthetic data are generated from a different distribution from the true distribution.

Below we analyze the sample complexity to achieve desired robust accuracy. Recall that, under the mild assumption $\boldsymbol{\mu}_{j}>\varepsilon, \forall j \in \{1,2, \ldots, d\}$, the closed form of the robust error of $f_{\boldsymbol{\theta}}$ is $Q(\frac{(\boldsymbol{\mu}-\varepsilon \mathbf{1}_d) ^{\top} \boldsymbol{\theta}}{\sigma\|\boldsymbol{\theta}\|})$.
Since $\hat{\boldsymbol{\theta}}_{\text {final }} = (1-\alpha + c\gamma\alpha)\boldsymbol{\mu} + \tilde{\boldsymbol{\delta}}$, we have the term inside $Q(\cdot)$ function is:
\begin{equation}\label{eq:tmp1}
 \frac{(\boldsymbol{\mu}-\varepsilon \mathbf{1}_d) ^{\top} \hat{\boldsymbol{\theta}}_{\text {final }} }{\sigma\|\hat{\boldsymbol{\theta}}_{\text {final }}\|}  = \frac{(1-\alpha + c\gamma\alpha)\|\boldsymbol{\mu}\|^2 + \boldsymbol{\mu}^\top\tilde{\boldsymbol{\delta}} }{\sigma\|\boldsymbol{\theta}\|} - \frac{(\varepsilon \mathbf{1}_d) ^{\top} \hat{\boldsymbol{\theta}}_{\text {final }} }{\sigma\|\hat{\boldsymbol{\theta}}_{\text {final }}\|}. 
\end{equation}

We consider the parameter setting:
\begin{equation}\label{eq:para}
  \epsilon < \frac12, \ \sigma = (nd)^{1/4}, \ \|\boldsymbol{\mu}\|^2 = d.
\end{equation}
Under such a setting and under the regime that $d/n \gg 1$, we have the classifier $f_{\boldsymbol{\mu}}$ achieves almost optimal performance in both robust and standard accuracy. Thus in the following, we mainly focus on the problem of finding the minimum number of synthetic samples $\tilde n$ needed in order to ensure the estimate $\hat{\boldsymbol{\theta}}_{\text {final }}$ is close (in direction) to $\boldsymbol{\mu}$.

For the squared inverse of the first term in \eqref{eq:tmp1}, we have

\begin{align}       \frac{\|\hat{\boldsymbol{\theta}}_{\text {final }}\|^{2}}{\left(\boldsymbol{\mu}^{\top} \hat{\boldsymbol{\theta}}_{\text {final }}\right)^{2}}
&= 
\frac{ \| \tilde{\boldsymbol{\delta}} + (1-\alpha + c\gamma\alpha)\boldsymbol{\mu} \|^2 }{ \left((1-\alpha+c\gamma\alpha) \|\boldsymbol{\mu}\|^2 + \boldsymbol{\mu}^\top \tilde{\boldsymbol{\delta}} \right)} \nonumber\\
&=
  \frac{1}{\|\boldsymbol{\mu}\|^{2}}+\frac{\|\tilde{\boldsymbol{\delta}}\|^{2}-\frac{1}{\|\boldsymbol{\mu}\|^{2}}\left( \boldsymbol{\mu}^{\top} \tilde{\boldsymbol{\delta}}\right)^{2}}{\left((1-\alpha+c\gamma\alpha)\|\boldsymbol{\mu}\|^{2}+\boldsymbol{\mu}^{\top} \tilde{\boldsymbol{\delta}}\right)^2}\nonumber\\
 &\leq \frac{1}{\|\boldsymbol{\mu}\|^{2}}+\frac{\|\tilde{\boldsymbol{\delta}}\|^{2}}{\left((1-\alpha+c\gamma\alpha)\|\boldsymbol{\mu}\|^{2}+\boldsymbol{\mu}^{\top} \tilde{\boldsymbol{\delta}}\right)^2} \nonumber\\
 & = \frac{1}{\|\boldsymbol{\mu}\|^{2}}+ 
 \frac{1}{\|\boldsymbol{\mu}\|^{4}}
 \frac{\|\tilde{\boldsymbol{\delta}}\|^{2}}{\left((1-\alpha+c\gamma\alpha)+ \frac{1}{\|\boldsymbol{\mu}\|^{2}}\boldsymbol{\mu}^{\top} \tilde{\boldsymbol{\delta}}\right)^2} 
 \label{eq:theorem1rob}
 \end{align}

Note that due to the dependence between $\tilde{y}_i$ and $\tilde{\boldsymbol{\varepsilon}}_i$, the random variable $\tilde{\boldsymbol{\delta}}$ is non-Gaussian. To obtain the concentration bounds for $\|\tilde{\boldsymbol{\delta}}\|^{2}$ and $\boldsymbol{\mu}^{\top} \tilde{\boldsymbol{\delta}}$, we follow the approach used in \cite{Carmon2019UnlabeledDI} as follows. Recall $\hat{\boldsymbol{\theta}}_{\text{inter}}=\frac{1}{n}\sum_{i=1}^{n} y_i\boldsymbol{x}_i$, $\tilde{y}_i=\operatorname{sign}(\hat{\boldsymbol{\theta}}_{\text{inter}}^\top\tilde{\boldsymbol{x}}_i)$, and $\tilde{\boldsymbol{\delta}}= \frac{1}{n+\tilde{n}}(\sum_{i=1}^{n} y_i\boldsymbol{\varepsilon}_i+\sum_{i=1}^{\tilde{n}} \tilde{y}_i\tilde{\boldsymbol{\varepsilon}}_i)$.
Find a coordinate system such that the first coordinate is in the direction of $\hat{\boldsymbol{\theta}}_{\text {inter }}$, and let $v^{(i)}$ denote the $i$ th entry of vector $v$ in this coordinate system. Then
$$
\tilde{y}_i=\operatorname{sign}\left(\tilde{\boldsymbol{x}}_i^{(1)}\right)=\operatorname{sign}\left(\boldsymbol{\mu}^{(1)}+\tilde{\boldsymbol{\varepsilon}}_i^{(1)}\right) .
$$
Consequently, $\tilde{\boldsymbol{\varepsilon}}_i^{(j)}$ is independent of $\tilde{y}_i$ for all $i$ and $j \geq 2$, so that $\tilde{y}_i \tilde{\boldsymbol{\varepsilon}}_i^{(j)} \sim \mathcal{N}\left(0, \sigma^2\right)$ and $\frac{1}{n+\tilde{n}} (\sum_{i=1}^{n} y_i \boldsymbol{\varepsilon}_i^{(j)}+\sum_{i=1}^{\tilde{n}} \tilde{y}_i \tilde{\boldsymbol{\varepsilon}}_i^{(j)}) \sim \mathcal{N}\left(0, \sigma^2 / (n+\tilde{n})\right)$ and
$$
\sum_{j=2}^d\left(\frac{1}{n+\tilde{n}} (\sum_{i=1}^{n} y_i \boldsymbol{\varepsilon}_i^{(j)}+\sum_{i=1}^{\tilde{n}} \tilde{y}_i \tilde{\boldsymbol{\varepsilon}}_i^{(j)})\right)^2 \sim \frac{\sigma^2}{n+\tilde{n}} \chi_{d-1}^2 .
$$

By the Cauchy-Schwarz inequality, we have:
$$
\begin{aligned}
  &\left(\frac{1}{n+\tilde{n}} (\sum_{i=1}^{n} y_i \boldsymbol{\varepsilon}_i^{(1)}+\sum_{i=1}^{\tilde{n}} \tilde{y}_i \tilde{\boldsymbol{\varepsilon}}_i^{(1)})\right)^2 \\
  & \leq \frac{1}{(n+\tilde{n})^2}\left\{ \left(\sum_{i=1}^{n} y_i^2 + \sum_{i=1}^{\tilde{n}} \tilde{y}_i^2\right) \left(\sum_{i=1}^{n}\left[\boldsymbol{\varepsilon}_i^{(1)}\right]^2 + \sum_{i=1}^{\tilde{n}}\left[\tilde{\boldsymbol{\varepsilon}}_i^{(1)}\right]^2
  \right)\right\}\\
  & = \frac{1}{n+\tilde{n}} (\sum_{i=1}^{n}\left[\boldsymbol{\varepsilon}_i^{(1)}\right]^2+\sum_{i=1}^{\tilde{n}}\left[\tilde{\boldsymbol{\varepsilon}}_i^{(1)}\right]^2) \sim \frac{\sigma^2}{n+\tilde{n}} \chi_{n+\tilde{n}}^2 .
\end{aligned}
$$
Since $\|\tilde{\boldsymbol{\delta}}\|^2=\sum_{j=1}^d\left(\frac{1}{n+\tilde{n}} (\sum_{i=1}^{n} y_i \boldsymbol{\varepsilon}_i^{(j)}+\sum_{i=1}^{\tilde{n}} \tilde{y}_i \tilde{\boldsymbol{\varepsilon}}_i^{(j)})\right)^2$, we have by the union bound
$$
\begin{aligned}
\mathbb{P}\left(\|\tilde{\boldsymbol{\delta}}\|^2 \geq 2 \frac{\sigma^2}{n+\tilde{n}}(d-1+n+\tilde{n})\right) & \leq \mathbb{P}\left(\chi_{n+\tilde{n}}^2 \geq 2 (n+\tilde{n})\right)+\mathbb{P}\left(\chi_{d-1}^2 \geq 2(d-1)\right) 
\\
& \leq e^{-(\tilde{n} + n) / 8}+e^{-(d-1) / 8} .
\end{aligned}
$$
Similarly applying the Cauchy-Schwarz inequality to  
$\boldsymbol{\mu}^{\top} \tilde{\boldsymbol{\delta}}=\frac{1}{\tilde{n} + n} \big( \sum_{i=1}^{n} y_i \boldsymbol{\mu}^{\top} \boldsymbol{\varepsilon}_i + \sum_{i=1}^{\tilde{n}} \tilde{y}_i \boldsymbol{\mu}^{\top} \tilde{\boldsymbol{\varepsilon}}_i \big)$, we have
$$
\begin{aligned}
&\left(\boldsymbol{\mu}^{\top} \tilde{\boldsymbol{\delta}}\right)^2 \leq \frac{1}{(n+\tilde{n})^2}\left\{\left( \sum_{i=1}^{n} y_i^2 + \sum_{i=1}^{\tilde{n}} \tilde{y}_i^2 \right)
\left(\sum_{i=1}^{n}\left(\boldsymbol{\mu}^{\top} \boldsymbol{\varepsilon}_i\right)^2 + 
\sum_{i=1}^{\tilde{n}}\left(\boldsymbol{\mu}^{\top} \tilde{\boldsymbol{\varepsilon}}_i\right)^2
\right)\right\}\\
&=\frac{1}{n+\tilde{n}} (\sum_{i=1}^{n}\left(\boldsymbol{\mu}^{\top} \boldsymbol{\varepsilon}_i\right)^2+\sum_{i=1}^{\tilde{n}}\left(\boldsymbol{\mu}^{\top} \tilde{\boldsymbol{\varepsilon}}_i\right)^2) \sim \frac{\sigma^2\|\boldsymbol{\mu}\|^2}{n+\tilde{n}} \chi_{n+\tilde{n}}^2.
\end{aligned}
$$
Therefore we have 
$$
\mathbb{P}\left(|\boldsymbol{\mu}^{\top} \tilde{\boldsymbol{\delta}} | \geq \sqrt{2} \sigma\|\boldsymbol{\mu}\|\right)=
\mathbb{P}\left((\boldsymbol{\mu}^{\top} \tilde{\boldsymbol{\delta}})^2 \geq 2 \sigma^2\|\boldsymbol{\mu}\|^2\right) \leq e^{-(\tilde{n} +n) / 8} .
$$

Finally, we look at the random variable $\gamma$. By definition $\gamma=\frac{1}{\tilde{n}} \sum_{i=1}^{\tilde{n}}\left(1-2 b_i\right)$, where $b_i$ is the indicator that $\tilde{y}_i$ is incorrect for the feature $\tilde{\boldsymbol{x}}_i$. Denote $\tilde{y}_i^\circ \in \{-1,1\}$ as the true label for $\tilde{\boldsymbol{x}}_i$, thus we have $ \tilde{\boldsymbol{x}}_i \sim N(c \tilde{y}_i^\circ \boldsymbol{\mu}, \sigma^2)$. 
Therefore
\[
\mathbb{E}[b_i] = \mathbb{P}[b_i=1] = \mathbb{P}\left(\tilde{y}_i^\circ \cdot \tilde{\boldsymbol{x}}_i^{\top} \hat{\boldsymbol{\theta}}_{\text {inter }}<0\right)=
\mathbb{P}\left(\mathcal{N}\left(\frac{c\boldsymbol{\mu}^{\top} \hat{\boldsymbol{\theta}}_{\text {inter }}}{\sigma\|\hat{\boldsymbol{\theta}}_{\text {inter }}\|}, 1\right)<0\right)= Q\left(\frac{c\boldsymbol{\mu}^{\top} \hat{\boldsymbol{\theta}}_{\text {inter }}}{\sigma\|\hat{\boldsymbol{\theta}}_{\text {inter }}\|}\right). 
\]
Moreover since $b_i$ are Bernoulli random variables, we have $\mathsf{Var}(b_i)=\mathbb{E}[b_i](1-\mathbb{E}[b_i])\leq \mathbb{E}[b_i]$.

By definition of $Q(\cdot)$ we clearly have $Q\left(\frac{c\boldsymbol{\mu}^{\top} \hat{\boldsymbol{\theta}}_{\text {inter }}}{\sigma\|\hat{\boldsymbol{\theta}}_{\text {inter }}\|}\right) \leq Q\left(\frac{\boldsymbol{\mu}^{\top} \hat{\boldsymbol{\theta}}_{\text {inter }}}{\sigma\|\hat{\boldsymbol{\theta}}_{\text {inter }}\|}\right)$ when $c\geq 1$ and $\boldsymbol{\mu}^{\top} \hat{\boldsymbol{\theta}}_{\text {inter }}>0$ (which happens with high probability as shown below).
Thus
$$
\mathbb{E}\left[\gamma \mid \hat{\boldsymbol{\theta}}_{\text {inter }}\right]=1- 2 Q\left(\frac{c\boldsymbol{\mu}^{\top} \hat{\boldsymbol{\theta}}_{\text {inter }}}{\sigma\|\hat{\boldsymbol{\theta}}_{\text {inter }}\|}\right) \geq 1-2 \operatorname{err}_{\text {standard }}\big(f_{\hat{\boldsymbol{\theta}}_{\text {inter }}}\big),
$$
where $\operatorname{err}_{\text {standard }}$ is given in \eqref{eq-standard-error}. 

Therefore, we expect $\gamma$ to be reasonably large as long as $\operatorname{err}_{\text {standard }}(f_{\hat{\boldsymbol{\theta}}_{\text {inter }}})<\frac{1}{2}$. Similar to \cite{Carmon2019UnlabeledDI}, we have
$$
\begin{aligned}
\mathbb{P}\left(\gamma<\frac{1}{6}\right) & =\mathbb{P}\left(\frac{1}{\tilde{n}} \sum_{i=1}^{\tilde{n}}\left(1-2 b_i\right)<\frac{1}{6}\right) \\
& = \mathbb{P}\left( \frac{1}{\tilde{n}} \sum_{i=1}^{\tilde{n}}\left(1-2 b_i\right)<\frac{1}{6} | \operatorname{err}_{\text {standard }}(f_{\hat{\boldsymbol{\theta}}_{\text {inter }}}) >\frac{1}{3}\right) \cdot \mathbb{P}\left(\text { err }_{\text {standard }}(f_{\hat{\boldsymbol{\theta}}_{\text {inter }}}) >\frac{1}{3} \right) \\
& \quad +\mathbb{P}\left( \frac{1}{\tilde{n}} \sum_{i=1}^{\tilde{n}}\left(1-2 b_i\right) < \frac16 \mid \operatorname{err}_{\text {standard }}(f_{\hat{\boldsymbol{\theta}}_{\text {inter }}}) \leq \frac{1}{3}\right) \cdot \mathbb{P}\left(\text { err }_{\text {standard }}(f_{\hat{\boldsymbol{\theta}}_{\text {inter }}}) \leq \frac{1}{3}\right) \\
& \leq \mathbb{P}\left(\text { err }_{\text {standard }}(f_{\hat{\boldsymbol{\theta}}_{\text {inter }}})>\frac{1}{3}\right)+\mathbb{P}\left(\frac{1}{\tilde{n}} \sum_{i=1}^{\tilde{n}} b_i > \frac{5}{12} \mid \text { err }_{\text {standard }}\left(f_{\hat{\boldsymbol{\theta}}_{\text {inter }}}\right) \leq \frac{1}{3}\right) .
\end{aligned}
$$
For the first probability, note that
$$
\frac{1}{3} \geq Q\left(\frac{1}{2}\right) \geq Q\left(\left[2\left(1+\sqrt{n / d}\right)\right]^{-1 / 2}\right)
$$
Therefore, by Lemma 1 in \cite{Carmon2019UnlabeledDI}, for sufficiently large $d / n$,
$$
\mathbb{P}\left(\operatorname{err}_{\text {standard }}\left(f_{\hat{\boldsymbol{\theta}}_{\text {inter }}}\right)>\frac{1}{3}\right) \leq e^{-c_1 \cdot \min \left\{\sqrt{d / n}, n\left(d / n\right)^{1 / 4}\right\}}
$$
for some constant $c_1$. 

For the second probability, note that $b_i$ are i.i.d. Bernoulli random variables with mean value $Q\left(\frac{c\boldsymbol{\mu}^{\top} \hat{\boldsymbol{\theta}}_{\text {inter }}}{\sigma\|\hat{\boldsymbol{\theta}}_{\text {inter }}\|}\right) \leq \operatorname{err}_{\text {standard }}(f_{\hat{\boldsymbol{\theta}}_{\text {inter }}} )$. Therefore, by Hoeffding's inequality we have
$$
\mathbb{P}\left(\frac{1}{\tilde{n}} \sum_{i=1}^{\tilde{n}} b_i > \frac{5}{12} \mid \operatorname{err}_{\text {standard }}\left(f_{\hat{\boldsymbol{\theta}}_{\text {inter }}}\right) \leq \frac{1}{3}\right) \leq e^{-2 \tilde{n}\left(\frac{5}{12}-\frac{1}{3}\right)^2}=e^{-\tilde{n} / 72} .
$$
Define the event,
$$
\mathcal{E}:=\left\{
\|\tilde{\boldsymbol{\delta}}\|^2 \leq 2 \frac{\sigma^2}{n+\tilde{n}}(d+n+\tilde{n}),\ 
\left|\boldsymbol{\mu}^{\top} \tilde{\boldsymbol{\delta}}\right| \leq \sqrt{2} \sigma\|\boldsymbol{\mu}\|, \text { and } \gamma \geq \frac{1}{6}\right\},
$$
thus by the previous concentration bounds, we have 
\[
\mathbb{P}[\mathcal{E}^C] \leq 
e^{-(\tilde{n} + n) / 8}+e^{-(d-1) / 8}+ e^{-c_1 \cdot \min \left\{\sqrt{d / n}, n\left(d / n\right)^{1 / 4}\right\}}
+ e^{-\tilde{n} / 72} \leq e^{-c_2 \min \left\{\tilde n, \ \sqrt{d / n}, \ n\left(d / n\right)^{1 / 4}\right\} }.
\]

Suppose the event $\mathcal{E}$ holds, then for the formula in \eqref{eq:theorem1rob} we have:
$$
\frac{\left\|\hat{\boldsymbol{\theta}}_{\text {final }}\right\|^2}{\left(\boldsymbol{\mu}^{\top} \hat{\boldsymbol{\theta}}_{\text {final }}\right)^2} \leq \frac{1}{\|\boldsymbol{\mu}\|^2}+\frac{2 \sigma^2(d+n+\tilde{n})}{(n+\tilde{n})\|\boldsymbol{\mu}\|^4\left( (1-\alpha+\frac16 c\alpha) -\frac{\sqrt{2} \sigma}{\|\boldsymbol{\mu}\|}\right)^2},
$$
which, after substituting the parameter setting \eqref{eq:para}, translates into:
$$
\begin{aligned}
\frac{\sigma^2\left\|\hat{\boldsymbol{\theta}}_{\text {final }}\right\|^2}{\left(\boldsymbol{\mu}^{\top} \hat{\boldsymbol{\theta}}_{\text {final }}\right)^2} & \leq \sqrt{\frac{n}{d}}+\frac{2 n d (d+n+\tilde{n})}{(n+\tilde{n}) d^2\left( (1-\alpha+\frac16 c\alpha) -\sqrt{2}\left(\frac{n}{d}\right)^{1 / 4}\right)^2} \\
& \leq \sqrt{\frac{n}{d}}+\frac{2 n d (d+n+\tilde{n})}{(n+\tilde{n}) d^2\left( \frac16 c\alpha -\sqrt{2}\left(\frac{n}{d}\right)^{1 / 4}\right)^2} \\
& \leq \sqrt{\frac{n}{d}}+\frac{72 n}{c\tilde{n}}\left(1+\tilde{c}_1\left(\frac{n}{d}\right)^{1 / 4}\right),
\end{aligned}
$$
where $\tilde{c}_1$ is some positive constant, and above we also implicitly assumed that $d/n$ is sufficiently large. 

Combining the above results together, following the analysis in \cite{Carmon2019UnlabeledDI}, we conclude that there exists a universal constant $\tilde C$ such that for $\epsilon^2 \sqrt{d/n} \geq \tilde C$, where $n$ is the number of labeled real data used to construct the intermediate classifier, and additional $\tilde n$ synthetic feature generated with mean vector $\pm c\boldsymbol{\mu}$ and pseudo labels, we have if 
\[
\tilde n \geq \frac{288n}{c} \epsilon^2 \sqrt{\frac{d}{n}}, 
\]
we have
$$
\mathbb{E}_{\hat{\boldsymbol{\theta}}_{\text {final }}} \text { err }_{\text {robust }}^{\infty, \epsilon}\left(f_{\hat{\boldsymbol{\theta}}_{\text {final }}}\right) \leq Q\left([\sqrt{2}-1] \epsilon\left(d / n\right)^{1 / 4}\right)+e^{-\epsilon^2 c_2 \sqrt{d / n}} \leq 10^{-3}.
$$
for sufficiently large $\tilde C$.

\end{proof}

\section{More simulation results under Gaussian setting in Section \ref{sec:theory}}
\label{simu-gaussian}

In this section, we present more detailed simulation results under the Gaussian setting in Section \ref{sec:theory} to demonstrate different scenarios in Proposition \ref{theorem3}. 

\paragraph{Experimental setting}
\label{simu:setting}
The experimental pipeline is as follow: 1) learning a intermediate classifier $\hat{\boldsymbol{\theta}}_{\text{inter}}$ by $n$ label data $\mathcal{D}=\{(\boldsymbol{x}_1,y_1),\ldots, (\boldsymbol{x}_{n},y_{n})\}$, 2) generating $\tilde{n}$ synthetic data $\tilde{\boldsymbol{x}} \sim \widetilde{\mathcal{D}}_\mathcal{X}=0.5\mathcal{N}(\tilde{\boldsymbol{\mu}},\sigma^2 \mathbb{I})+0.5 \mathcal{N}(-\tilde{\boldsymbol{\mu}},\sigma^2 \mathbb{I})$, with $\tilde{\boldsymbol{\mu}}=c\boldsymbol{\mu}$, 3)assigning pseudo label for synthetic data using the intermediate classifier $\hat{\boldsymbol{\theta}}_{\text{inter}}$, 4) learning $\hat{\boldsymbol{\theta}}_{\text{final}}$ by adversarial training on $\tilde{n}$ synthetic data, 5) testing on 10k extra real data to obtain the clean accuracy and robust accuracy. For data dimension $d=100$, we set $\|\boldsymbol{\mu}\|^2=2$, $\varepsilon=0.5$, and for $d=100$, we set $\|\boldsymbol{\mu}\|^2=4$, $\varepsilon=0.1$.

Table \ref{t1} and Table \ref{t2} show the clean and robust accuracy learned on synthetic distribution $\tilde{\boldsymbol{\mu}}=c\boldsymbol{\mu}$ with different angles between $\boldsymbol{\mu}$ and $\mathbf{\epsilon}\mathbf{1}_d$. Table \ref{t4} shows the clean and robust accuracy learned on synthetic distribution $\tilde{\boldsymbol{\mu}}=c (\boldsymbol{\mu} - \varepsilon \mathbf{1}_d)$ with different angles between $\boldsymbol{\mu}$ and $\mathbf{\epsilon}\mathbf{1}_d$. Recall that $\boldsymbol{\mu}$ is (one of) the optimal linear classifier that maximizes the clean accuracy under the true distribution considered in Section \ref{sec:theory}, similarly $\boldsymbol{\mu} - \mathbf{\epsilon} \mathbf{1}_d$ is the optimal solution for robust accuracy. Therefore, different angles between $\boldsymbol{\mu}$ and $\mathbf{\epsilon}\mathbf{1}_d$ represent different trade-offs between the clean and robust accuracy. For example, when the angle between $\boldsymbol{\mu}$ and $\mathbf{\epsilon}\mathbf{1}_d$ is 0 degrees, i.e., $\boldsymbol{\mu} = c \mathbf{1}_d$, we have that the optimal solution for clean accuracy and robust accuracy are the same. In most cases, the classifier learned from the synthetic distribution that is most separable achieves better performance even than the iid samples, which verifies Proposition \ref{theorem3}.

\begin{table}[h]
\caption{The clean and robust accuracy learned on synthetic distribution $\tilde{\boldsymbol{\mu}}=c\boldsymbol{\mu}$ when $d=2$ and the angle between $\boldsymbol{\mu}$ and $\mathbf{\epsilon}$ is 0 degrees and 90 degrees. ``Real'' denotes the real data distribution, and $n$ denotes the number of data from the real distribution, while we use ``$c$'' to denote different synthetic distributions and use $\tilde{n}$ to denote the number of synthetic data. The results and the standard deviation in the bracket are obtained from 50 repetitions.}
\label{t1}
\begin{center}
\begin{small}

\begin{tabular}{l|lcc|cc}
\toprule & & \multicolumn{2}{c}{0 degree} & \multicolumn{2}{c}{90 degree} \\
&  & acc (std) & rob acc (std) & acc (std) & rob acc (std) \\
\midrule
   \multirow{4}[0]{*}{Real} & $n=10$ & 0.9201 (0.0012) & \textbf{0.7593} (0.0020) & 0.9171 (0.0046) & 0.7552 (0.0040) \\
  & $n=20$ & 0.9204 (0.0007) & \textbf{0.7598} (0.0016) & 0.9186 (0.0017) & 0.7563 (0.0012) \\
  & $n=50$ &0.9206 (0.0004) & \textbf{0.7605} (0.0007) & 0.9196 (0.0009) & 0.7566 (0.0006) \\
  & $n=100$ & 0.9205 (0.0004) & \textbf{0.7608} (0.0006) & 0.9199 (0.0006) & 0.7565 (0.0007) \\
  \midrule
  \multirow{4}[0]{*}{$c=0.5$} & $\tilde{n}=10$ & 0.9159 (0.0099 ) & 0.7541 (0.0096) & 0.9104 (0.0121) & 0.7492 (0.0122) \\
  & $\tilde{n}=20$ & 0.9179 (0.0047) & 0.7562 (0.0050) & 0.9161 (0.0052) & 0.7546 (0.0054) \\
  & $\tilde{n}=50$ & 0.9200 (0.0023) & 0.7586 (0.0024) &  0.9183 (0.0022) & 0.7570 (0.0022) \\
  & $\tilde{n}=100$ & 0.9213 (0.0011) & 0.7601 (0.0009) & 0.9193 (0.0012) & 0.7576 (0.0010) \\
   \midrule
  \multirow{4}[0]{*}{$c=1$} & $\tilde{n}=10$ &  0.9133 (0.0066) & 0.7502 (0.0061) & 0.9161 (0.0048) & \textbf{0.7598} (0.0048) \\
  & $\tilde{n}=20$ &0.9155 (0.0020) & 0.7516 (0.0019) & 0.9180 (0.0017) & \textbf{0.7612} (0.0020) \\
  & $\tilde{n}=50$ &0.9161 (0.0009) & 0.7525 (0.0006) & 0.9186 (0.0010) & \textbf{0.7620} (0.0006) \\
  & $\tilde{n}=100$ &0.9165 (0.0005) & 0.7528 (0.0006) & 0.9189 (0.0005) & \textbf{0.7622} (0.0003) \\
  \midrule
  \multirow{4}[0]{*}{$c=1.5$}& $\tilde{n}=10$ &  \textbf{0.9209} (0.0038) & 0.7523 (0.0025) & \textbf{0.9221} (0.0017) & 0.7583 (0.0015) \\
  & $\tilde{n}=20$ & \textbf{0.9228} (0.0010) & 0.7536 (0.0006 ) & \textbf{0.9226} (0.0013) & 0.7588 (0.0013) \\
  & $\tilde{n}=50$ & \textbf{0.9229} (0.0008) & 0.7538 (0.0005) & \textbf{0.9232} (0.0005) & 0.7594 (0.0006) \\
  & $\tilde{n}=100$ & \textbf{0.9232} (0.0003) & 0.7538 (0.0005) & \textbf{0.9233} (0.0005) & 0.7595 (0.0005) \\

\bottomrule
\end{tabular}

\end{small}
\end{center}
% \vskip -0.1in
\end{table}

\begin{table}[h]
\caption{The clean and robust accuracy learned on synthetic distribution $\tilde{\boldsymbol{\mu}}=c\boldsymbol{\mu}$ when $d=2$ and the angle between $\boldsymbol{\mu}$ and $\mathbf{\epsilon}$ is 30 degrees and 60 degrees. ``Real'' denotes the real data distribution, and $n$ denotes the number of data from the real distribution, while we use ``$c$'' to denote different synthetic distributions and use $\tilde{n}$ to denote the number of synthetic data. The results and the standard deviation in the bracket are obtained from 50 repetitions.}
\label{t2}
% \vskip 0.15in
\begin{center}
\begin{small}

\begin{tabular}{l|lcc|cc}
\toprule & & \multicolumn{2}{c}{30 degree} & \multicolumn{2}{c}{60 degree} \\
 &&  acc (std) & rob acc (std) & acc (std) & rob acc (std) \\
 \midrule
  \multirow{4}[0]{*}{Real} & $n=10$ &  0.8307 (0.0123) & 0.6343 (0.0283) & 0.8348 (0.0117) & 0.6378 (0.0293) \\
  & $n=20$ & 0.8353 (0.0055) & 0.6404 (0.0234) & 0.8391 (0.005) & 0.6433 (0.0222) \\
  & $n=50$ &0.8371 (0.0022) & 0.6450 (0.0168) & 0.8410 (0.0017) & 0.6494 (0.0134) \\
  & $n=100$ & 0.8385 (0.0010) & 0.6461 (0.0097) & 0.8413 (0.0013) & 0.6522 (0.0102) \\
 \midrule
  \multirow{4}[0]{*}{$c=0.5$}& $\tilde{n}=10$ &  0.8265 (0.0184 ) & 0.6282 (0.0418 ) & 0.8338 (0.0132 ) & 0.6303 (0.0335 ) \\
  & $\tilde{n}=20$ & 0.8299 (0.0129) & 0.6352 (0.0325) & 0.8365 (0.0132) & 0.6393 (0.0316) \\
  & $\tilde{n}=50$ & 0.8372 (0.0046) & 0.6483 (0.0215) & 0.8414 (0.0034) & 0.6489 (0.0199) \\
  & $\tilde{n}=100$ & 0.8402 (0.0015) & 0.6466 (0.0110) & 0.8431 (0.0012) & 0.6510 (0.0135) \\
 \midrule
  \multirow{4}[0]{*}{$c=1$} & $\tilde{n}=10$ &  0.8383 (0.0158) & 0.6439 (0.0319) & \textbf{0.8377} (0.0074) & 0.6396 (0.0267) \\
  & $\tilde{n}=20$ & 0.8425 (0.0060) & 0.6480 (0.0218) & \textbf{0.8416} (0.0034) & \textbf{0.6513} (0.0178) \\
  & $\tilde{n}=50$ & 0.8455 (0.0023) & 0.6553 (0.0128) & \textbf{0.8432} (0.0020) & \textbf{0.6503} (0.0122) \\
  & $\tilde{n}=100$ & 0.8457 (0.0021) & 0.6535 (0.0100) & \textbf{0.8435} (0.0014) & \textbf{0.6501}1 (0.0096) \\
 \midrule
  \multirow{4}[0]{*}{$c=1.5$} & $\tilde{n}=10$ &   \textbf{0.8431} (0.0045) & \textbf{0.6542} (0.0173) & 0.8368 (0.0073) & \textbf{0.6446} (0.0213) \\
  & $\tilde{n}=20$ & \textbf{0.8447} (0.0021) & \textbf{0.6542} (0.0142) & 0.8393 (0.0022) & 0.6479 (0.0150) \\
  & $\tilde{n}=50$ & \textbf{0.8455} (0.0006) & \textbf{0.6556} (0.0082) & 0.8404 (0.0005) & 0.6488 (0.0089) \\
  & $\tilde{n}=100$ & \textbf{0.8457} (0.0004) & \textbf{0.6547} (0.0057) & 0.8404 (0.0007) & 0.6486 (0.0082) \\

\bottomrule
\end{tabular}

\end{small}
\end{center}
% \vskip -0.1in
\end{table}

\begin{table}[h]
\caption{The clean and robust accuracy learned on synthetic distribution $\tilde{\boldsymbol{\mu}}=c\boldsymbol{\mu}$ when $d=100$ and the angle between $\boldsymbol{\mu}$ and $\mathbf{\epsilon}$ is 0 degrees. ``Real'' denotes the real data distribution, and $n$ denotes the number of data from the real distribution, while we use ``$c$'' to denote different synthetic distributions and use $\tilde{n}$ to denote the number of synthetic data. The results and the standard deviation in the bracket are obtained from 50 repetitions}
\label{t3}
% \vskip 0.15in
\begin{center}
\begin{small}

\begin{tabular}{l|lcc}
\toprule & &  acc (std) & rob acc (std) \\
\midrule
 \multirow{4}[0]{*}{Real} & $n=10$ & 0.9023 (0.0192) & 0.6843 (0.0359) \\
  & $n=20$ & 0.9341 (0.0128) & 0.7519 (0.0267) \\
  & $n=50$ & 0.9599 (0.0028) & 0.8078 (0.0061) \\
  & $n=100$ & 0.9682 (0.0014) & 0.8239 (0.0028) \\
 \midrule
  \multirow{4}[0]{*}{$c=0.5$} & $\tilde{n}=10$ & 0.7562 (0.0564) & 0.4611 (0.0694) \\
  & $\tilde{n}=20$ & 0.8566 (0.0307) & 0.6047 (0.0491) \\
  & $\tilde{n}=50$ & 0.9261 (0.0117) & 0.7328 (0.0227) \\
  & $\tilde{n}=100$ & 0.9505 (0.0047) & 0.7848 (0.0111) \\
 \midrule
  \multirow{4}[0]{*}{$c=1$} & $\tilde{n}=10$ & 0.8866 (0.0273) & 0.6557 (0.0487) \\
  & $\tilde{n}=20$ & 0.9371 (0.0091) & 0.7555 (0.0201) \\
  & $\tilde{n}=50$ & 0.9620 (0.0028) & 0.8085 (0.0060) \\
  & $\tilde{n}=100$ & 0.9695 (0.0012) & 0.8239 (0.0031) \\
 \midrule
  \multirow{4}[0]{*}{$c=1.5$} & $\tilde{n}=10$ & \textbf{0.9400} (0.0100) &\textbf{ 0.7603} (0.0233) \\
  & $\tilde{n}=20$ & \textbf{0.9591} (0.0037) & \textbf{0.8031} (0.0080) \\
  & $\tilde{n}=50$ & \textbf{0.9710} (0.0013) & \textbf{0.8280} (0.0028) \\
  & $\tilde{n}=100$ & \textbf{0.9743} (0.0008) & \textbf{0.8343} (0.0011) \\
\bottomrule
\end{tabular}
\end{small}
\end{center}
% \vskip -0.1in
\end{table}

\begin{table}[h]
\caption{The clean and robust accuracy learned on synthetic distribution $\tilde{\boldsymbol{\mu}}=c (\boldsymbol{\mu} - \varepsilon \mathbf{1}_d)$ when $d=2$ and the angle between $\boldsymbol{\mu}$ and $\mathbf{\epsilon}$ is 30 degrees and 60 degrees. ``Real'' denotes the real data distribution, and $n$ denotes the number of data from the real distribution, while we use ``$c$'' to denote different synthetic distributions and use $\tilde{n}$ to denote the number of synthetic data. The results and the standard deviation in the bracket are obtained from 50 repetitions.}
\label{t4}
% \vskip 0.15in
\begin{center}
\begin{small}

\begin{tabular}{l|lcc|cc}
\toprule & & \multicolumn{2}{c}{30 degree} & \multicolumn{2}{c}{60 degree} \\
 &&  acc (std) & rob acc (std) & acc (std) & rob acc (std) \\
 \midrule
  \multirow{4}[0]{*}{Real} & $n=10$ &  0.9152 (0.0049) & 0.7633 (0.0111) &\textbf{0.9211} (0.0034) & 0.7702 (0.0094)\\
  & $n=20$ & 0.9170 (0.0030) & 0.7642 (0.0075) &\textbf{0.9225} (0.0020) & 0.7714 (0.0068) \\
  & $n=50$ &0.9183 (0.0009) & 0.7653 (0.0040) & \textbf{0.9232} (0.0011) & 0.7711 (0.0050)\\
  & $n=100$ & 0.9185 (0.0006) & 0.7658 (0.0027) & \textbf{0.9235} (0.0009) & 0.7724 (0.0027)\\
 \midrule
  \multirow{4}[0]{*}{$c=0.5$}& $\tilde{n}=10$ &  0.9089 (0.0183) & 0.7563 (0.0310) & 0.9111 (0.0114) & 0.7638 (0.0172)\\
  & $\tilde{n}=20$ & 0.9144 (0.0068) & 0.7659 (0.0107) &0.9138 (0.0068) & 0.7694 (0.0066) \\
  & $\tilde{n}=50$ & 0.9174 (0.0029) & 0.7680 (0.0068) & 0.9161 (0.0038) & 0.7714 (0.0033) \\
  & $\tilde{n}=100$ & 0.9183 (0.0016) & 0.7681 (0.0053) & 0.9165 (0.0031) & 0.7727 (0.0014)\\
 \midrule
  \multirow{4}[0]{*}{$c=1$} & $\tilde{n}=10$ &  0.9135 (0.0116) & 0.7642 (0.0194) & 0.9069 (0.0111) & 0.7677 (0.0109)\\
  & $\tilde{n}=20$ & 0.9178 (0.0046) & 0.7710 (0.0073) & 0.9042 (0.0098) & 0.7676 (0.0072)\\
  & $\tilde{n}=50$ & 0.9183 (0.0042) & 0.7728 (0.0042) & 0.9073 (0.0047) & 0.7702 (0.0017)\\
  & $\tilde{n}=100$ & 0.9196 (0.0017) & 0.7733 (0.0036) & 0.9059 (0.0039) & 0.7698 (0.0016)\\
 \midrule
  \multirow{4}[0]{*}{$c=1.5$} & $\tilde{n}=10$ &   \textbf{0.9181} (0.0079) & \textbf{0.7747} (0.0104) & 0.9034 (0.0079) & \textbf{0.7704} (0.0053)\\
  & $\tilde{n}=20$ & \textbf{0.9209} (0.0053) & \textbf{0.7770} (0.0052) & 0.9077 (0.0059) & \textbf{0.7716} (0.0056)\\
  & $\tilde{n}=50$ & \textbf{0.9218} (0.0029) & \textbf{0.7788} (0.0028) & 0.9073 (0.0030) & \textbf{0.7722} (0.0014)\\
  & $\tilde{n}=100$ & \textbf{0.9222} (0.0017) & \textbf{0.7793} (0.0023) & 0.9077 (0.0024) & \textbf{0.7729} (0.0011)\\
\bottomrule
\end{tabular}

\end{small}
\end{center}
% \vskip -0.1in
\end{table}

% \begin{table}[h]
% \caption{The clean and robust accuracy learned on synthetic distribution $\tilde{\boldsymbol{\mu}}=c (\boldsymbol{\mu} - \varepsilon \mathbf{1}_d)$ when $d=2$ and the angle between $\boldsymbol{\mu}$ and $\mathbf{\epsilon}$ is 90 degrees. ``Real'' denotes the real data distribution, and $n$ denotes the number of data from the real distribution, while we use ``$c$'' to denote different synthetic distributions and use $\tilde{n}$ to denote the number of synthetic data. The results and the standard deviation in the bracket are the results of 50 independent trials.}
% \label{t90}
% \vskip 0.15in
% \begin{center}
% \begin{small}

% \begin{tabular}{l|lcc}
% \toprule & &  acc (std) & rob acc (std) \\
% \midrule

% \bottomrule
% \end{tabular}
% \end{small}
% \end{center}
% \vskip -0.1in
% \end{table}

\clearpage
\section{The detailed construction of the contrastive loss}\label{app:contra}

In this section, we first give a detailed description of several possible ways to design contrastive loss, especially in constructing positive and negative pairs. Then, we give a visualization of the synthetic data distributions generated under different contrastive losses.

\subsection{Positive and negative pair selection strategy.} 
\label{positive-pair}
In this subsection, we give several possible ways to construct positive and negative pairs.
\begin{enumerate}[leftmargin=*]
\item Vanilla version: Using all the samples in the minibatch is the common strategy for contrastive learning. In the diffusion process, since for each time step $t$, we want to distinguish each image from other images in the minibatch at the same time step, a straight-forward strategy is to use all the samples in the minibatch other than $\boldsymbol{x}^i_t$ at time step $t$ to be the negative pairs. For the positive pairs, we can simply adopt $\boldsymbol{x}_{t+1}^i$ to be the positive pairs rather than augmentation of $\boldsymbol{x}_{t}^i$.
\item Real data as positive pairs: A possible improvement upon the vanilla version is considering we aim to generate images similar to real data. Therefore, we can directly adopt the real data as positive pairs.
\item Real data as negative pairs: Another improvement upon the vanilla version is considering the other images in time step $t$ in the minibatch is not as high quality as the real data. Therefore, we can directly adopt the real data as the negative pairs.
\item Class conditional version: When we use conditional diffusion, and the class label of $\boldsymbol{x}_t$ in the minibatch is available, a further improvement can be adopted is to use all the samples with different class label $y$ in the minibatch at time step $t$ to be the negative pairs.
\end{enumerate}

\subsection{Visualization of the synthetic data distribution generated by different designs of the contrastive loss}
\label{2d-simulation}

In this subsection, we demonstrate the synthetic distributions generated by different designs of the contrastive loss mentioned in Section \ref{positive-pair} on the Gaussian setting mentioned in Section \ref{sec:theo_setup}. Figure \ref{contrastive-simulation1} shows the synthetic distribution generated by using $\mathcal{N}(\boldsymbol{0}, \mathbb{I})$ as initialization, while Figure \ref{contrastive-simulation2} shows the synthetic distribution generated by using $\mathcal{N}(\boldsymbol{0}, 4 \mathbb{I})$ as initialization. In all figures, all of the contrastive loss except for conditional hard negative mining form a circle within each class, which means these algorithms cannot explicitly distinguish the data within the same class and thus maximize the distance within each class, while the guidance from conditional hard negative mining can generate samples that are more distinguishable.

\begin{figure}[h]
	\centering
	\subfigure[InfoNCE]{	\includegraphics[width=.35\textwidth]{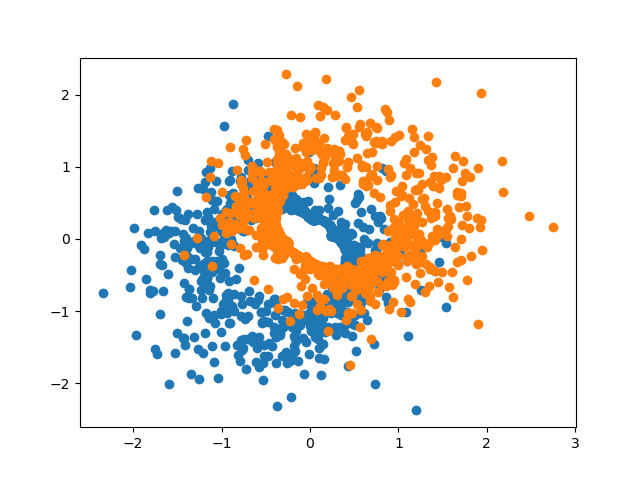}}
\subfigure[Hard negative mining]{\includegraphics[width=.35\textwidth]{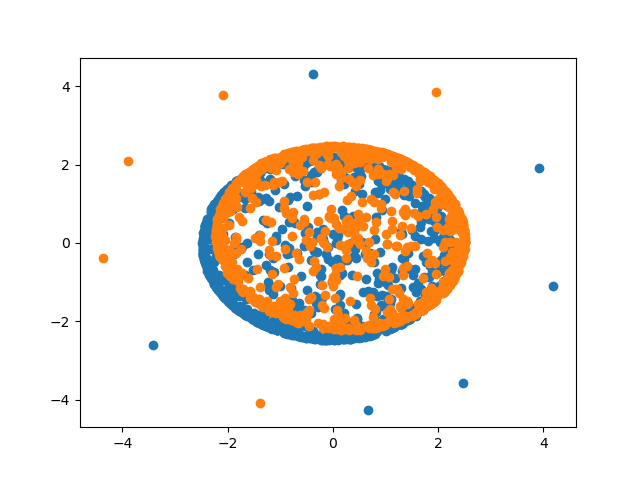}}
\\
\subfigure[Hard negative mining (real data as positive pair)]{\includegraphics[width=.35\textwidth]{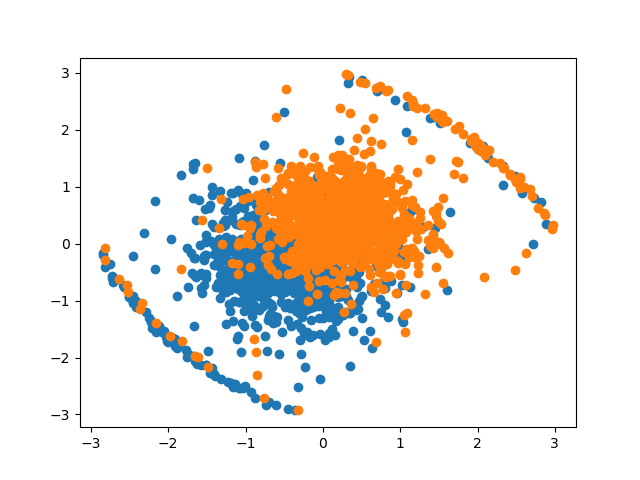}}
\subfigure[Hard negative mining (real data as negative pair)]{	\includegraphics[width=.35\textwidth]{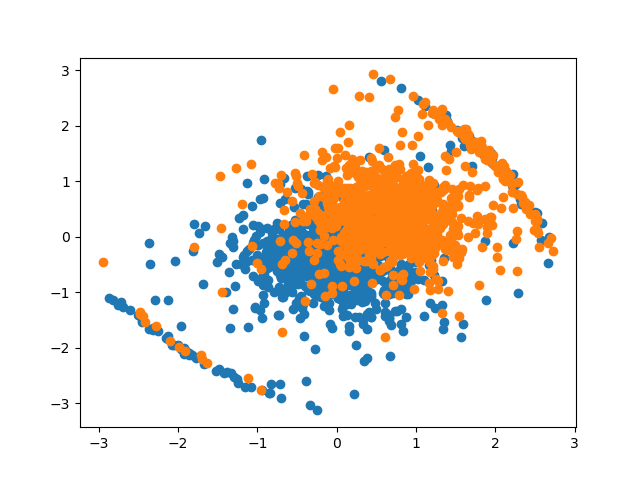}}
\\
\subfigure[Conditional inforNCE]{	\includegraphics[width=.35\textwidth]{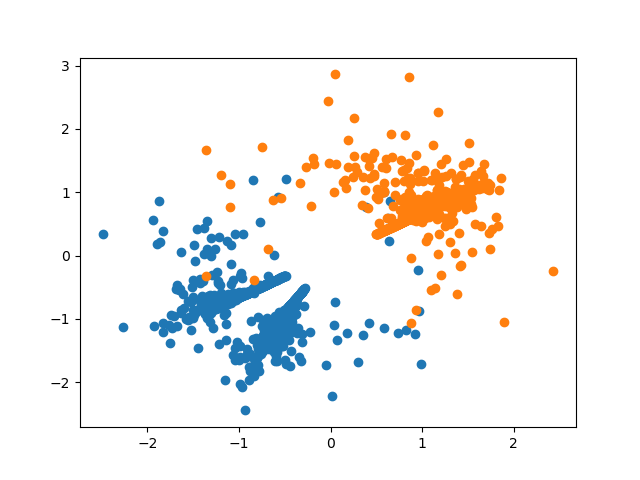}}
\subfigure[Conditional hard negative mining]{	\includegraphics[width=.35\textwidth]{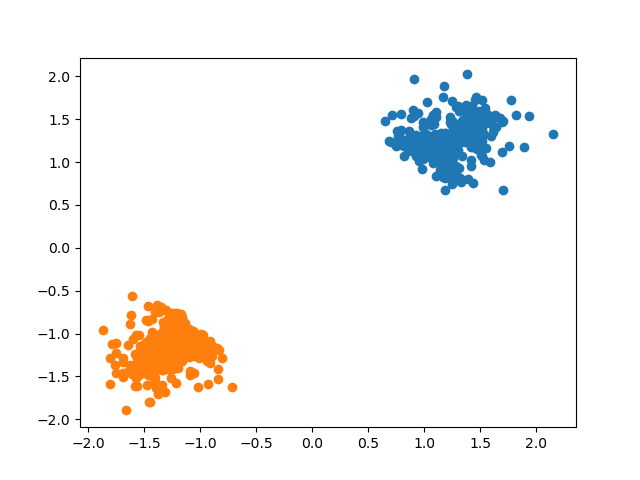}}

	\caption{A comparison of the synthetic distribution guided by different contrastive loss with initialization $\mathcal{N}(\boldsymbol{0},\mathbb{I})$. Real data as positive pair means using the mixture of oracle distribution $\mathcal{N}(\pm\mathbf{1}_d,\mathbb{I})$ and the data in the same batch as negative pair, while real data as negative pair means using the data in the same batch as positive pair and using the mixture of oracle distribution as negative pair.}
	\label{contrastive-simulation1}
\end{figure}

\begin{figure}[h]
	\centering
\subfigure[Diffusion]{	\includegraphics[width=.35\textwidth]{fig/diffusion_variance=4_0001.png}}
\subfigure[InfoNCE ]{	\includegraphics[width=.35\textwidth]{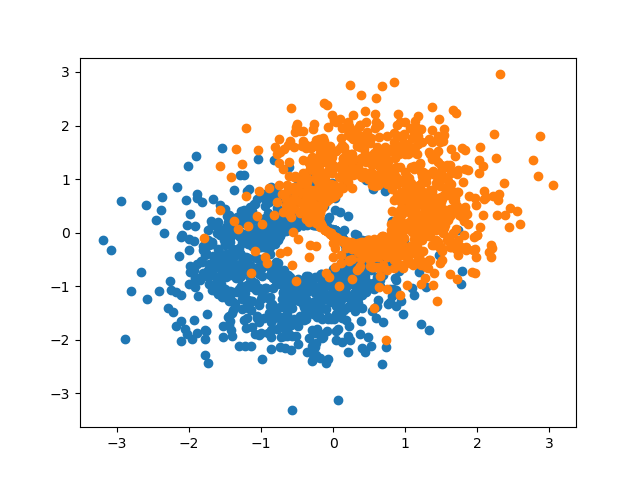}}
\\
\subfigure[Hard negative mining]{\includegraphics[width=.35\textwidth]{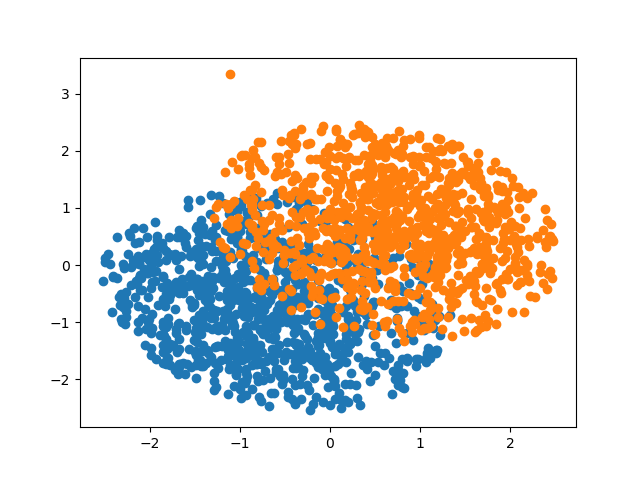}}
\subfigure[Hard negative mining (true data as positive pair)]{\includegraphics[width=.35\textwidth]{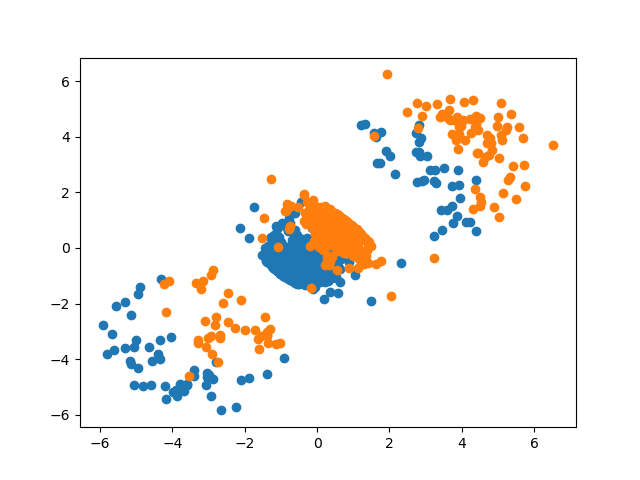}}
\\
\subfigure[Hard negative mining (true data as negative pair)]{	\includegraphics[width=.35\textwidth]{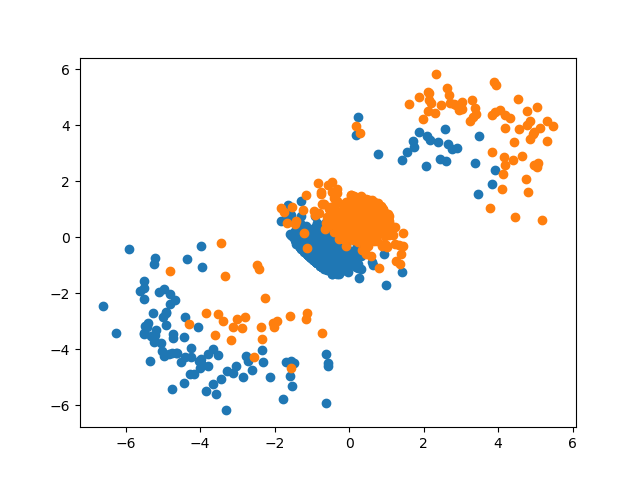}}
\subfigure[Conditional hard negative mining]{	\includegraphics[width=.35\textwidth]{fig/conditinoal_hardmining_minus_true_100_variance=4_0001.png}}

	\caption{A comparison of the synthetic distribution guided by different contrastive loss with initialization $\mathcal{N}(\boldsymbol{0}, 4 \mathbb{I})$.}
	\label{contrastive-simulation2}
\end{figure}

\clearpage
\section{The experimental results for the real datasets}
\label{experimental-setup}
\subsection{Experimental setup for MNIST dataset}
We describe the pipeline of synthetic data generation for adversarial robustness and a corresponding setting for the MNIST dataset in this subsection.

\paragraph{Dataset.} MNIST dataset \cite{LeCun1998GradientbasedLA} contains 60k 28$\cdot$28 pixel grayscale handwritten digits between 0 to 9 for training and 10k digits for testing.

\paragraph{Synthetic data generation by the diffusion model.} To utilize the pre-trained diffusion model \footnote{\href{https://github.com/VSehwag/minimal-diffusion\#how-useful-is-synthetic-data-from-diffusion-models-}{https://github.com/VSehwag/minimal-diffusion}}, we use a conditional DDPM for generating samples for MNIST dataset. We adopt the hard negative mining loss with $\tau=10$, the strength of guidance of the contrastive loss $\lambda=5k$, the probability of the same class in the minibatch $\tau^+=0.1$ and the hardness of negative mining $\beta=1$. We also use the pre-trained four layers Convolutional Neural Network model (removing the last fully connected layer) to get the representation for applying the contrastive loss and use a 2-layer feed-forward neural network to encode the representation after the pre-trained model.

\paragraph{Adversarial Training.} 
Since grayscale handwritten digits can be easily classified, we adopt four layers Convolutional Neural Network as the classifier instead of using the Wide ResNet-28-10 model. We adopt stochastic weight averaging \cite{Izmailov2018AveragingWL} with the decay rate 0.995 and use TRADES \cite{Zhang2019TheoreticallyPT} with 10 Projected Gradient Descent steps and $\varepsilon_\infty=0.3$ for 150 epochs with batch size 1024.

\subsection{Experimental setup for CIFAR-10 dataset}
We describe the pipeline of synthetic data generation for adversarial robustness and a corresponding setting for the CIFAR-10 dataset in this subsection.

\paragraph{Dataset.} CIFAR-10 dataset \cite{Krizhevsky2009LearningML} contains 50K 32$\cdot$32 color training images in 10 classes and 10K images for testing.

\paragraph{Overall training pipeline}
We follow the same training pipeline as \cite{gowal2021improving}, i.e., synthesizing data by using the diffusion model, assigning pseudo-label for synthetic data and aggregating the original data and the synthetic data for adversarial training. We give a careful explanation of these three components as follow.

\paragraph{Synthetic data generation by the diffusion model.} Considering the advantage of DDIM on generation speed, we base on the official implementation of the DDIM model \citep{Song2021DenoisingDI} and add the guidance of the contrastive loss. We generate images with 200 steps with batchsize=512, and use the quadratic version of sub-sequence selection \footnote{We refer to Appendix D.2 for a detailed explanation of the quadratic version.}. For the guidance of the contrastive loss, we try different designs of the contrastive loss mentioned in Section~\ref{contr-loss}. We set the temperature $\tau=0.1$ and the strength of guidance of the contrastive loss $\lambda=20k$ in the InfoNCE loss, while $\tau=10$, the strength of guidance of the contrastive loss $\lambda=100k$, the probability of the same class in the minibatch $\tau^+=0.1$ and the hardness of negative mining $\beta=1$ in hard negative mining loss. These corresponding hyperparameters are chosen based on some preliminary experiments on image generation. The detailed ablation studies can be found in Section~\ref{sec:ablation}. Moreover, we also delve into the representation used by contrastive loss. The default setting is to use the pre-trained Wide ResNet-28-10 model \cite{gowal2021improving} to get the representation for applying the contrastive loss, which is named as (without embedding) in Section~\ref{sec:ablation}. A further improvement is to apply a 2-layer feed-forward neural network to encode the representation after the pre-trained model, which is named as (with embedding). The advantage of the latter design is we can adopt the contrastive loss to optimize the encoding network rather than a fixed encoder.

\paragraph{LaNet for assigning pseudo-label.} Since the DDIM is an unconditional generator, we need to assign the pseudo-label to the generated sample. We follow the same choice adopted by \cite{sehwag2021improving}, i.e., using state-of-the-art LaNet \cite{Wang2019SampleEfficientNA} network for assigning the pseudo-label for the synthetic data.  

\paragraph{Adversarial Training.} 
We follow the same setting as \cite{gowal2021improving}, i.e., we use Wide ResNet-28-10 \cite{Zagoruyko2016WideRN} with Swish activation function \cite{Hendrycks2016GaussianEL}, adopt stochastic weight averaging \cite{Izmailov2018AveragingWL} with decay rate 0.995 and use TRADES \cite{Zhang2019TheoreticallyPT} with 10 Projected Gradient Descent steps and $\varepsilon_\infty=8/255$ for 400 epochs with batch size 1024\footnote{For Table \ref{tab:ablation-contrastive} in the ablation studies subsection, we use batch size with 256.}.

\paragraph{Evaluation setup}
For each trained model, we adopt AUTOATTACK \cite{Croce2020ReliableEO} with $\epsilon_\infty=8/255$.

\subsection{Experimental setup for Traffic Signs dataset}\label{experimental-setup}
We describe the pipeline of synthetic data generation for adversarial robustness and a corresponding setting for the Traffic Signs dataset in this subsection.

\paragraph{Dataset.} Traffic Signs dataset \cite{Houben-IJCNN-2013} contains 39252 training images in 43 classes and 12629 images for testing, and the image sizes vary between 15x15 to 250x250 pixels. 

\paragraph{Synthetic data generation by the diffusion model.} To utilize the pre-trained diffusion model \footnote{\href{https://github.com/VSehwag/minimal-diffusion\#how-useful-is-synthetic-data-from-diffusion-models-}{https://github.com/VSehwag/minimal-diffusion}}, we use a conditional DDPM for generating samples for Traffic Signs dataset. We adopt the hard negative mining loss with $\tau=10$, the strength of guidance of the contrastive loss $\lambda=5k$, the probability of the same class in the minibatch $\tau^+=0.1$ and the hardness of negative mining $\beta=1$. We also use the pre-trained Wide ResNet-28-10 model to get the representation for applying the contrastive loss and use a 2-layer feed-forward neural network to encode the representation after the pre-trained model.

\paragraph{Adversarial Training.} 
We follow the same setting as the CIFAR-10 dataset, except the training epochs are reduced to 50. We also extend the training epochs to 400 but do not find significant improvement.

\section{Ablation study\protect\footnote{In this section, the robust accuracy is reported by the worst accuracy obtained by either AUTOATTACK \cite{Croce2020ReliableEO} or AA+MT \cite{Gowal2020UncoveringTL}}.}
\subsection{The effectiveness of different contrastive losses.}\label{ab-contrastive}
% \vspace{-0.1in}

Table \ref{tab:ablation-contrastive} demonstrates the performance of different designs of the contrastive loss. We find out that applying the hard negative mining together with the embedding network achieves better clean and robust accuracy when the additional data is small (50K and 200K setting), while the infoNCE loss achieves better clean and robust accuracy when the additional data is large (1M setting). This result shows that we can improve the sample efficiency of the generative model by carefully designing the contrastive loss.

% \vspace{-0.18in}
\begin{table*}[ht!]
 \centering
 \caption{The performance of Contrastive-DP under different contrastive loss: infoNCE and HNM losses, and w/wo embedding denote with/without an embedding network.} %50K, 200k, and 1M denote the number of synthetic used for adversarial training.
% \vskip 0.15in
  \resizebox{0.8\textwidth}{!}{\begin{tabular}{lrrrrrr}
\toprule          & \multicolumn{2}{c}{50K} & \multicolumn{2}{c}{200K} & \multicolumn{2}{c}{1M} \\
\midrule
     & \multicolumn{1}{l}{clean acc} & \multicolumn{1}{l}{rob acc} & \multicolumn{1}{l}{clean acc} & \multicolumn{1}{l}{rob acc} & \multicolumn{1}{l}{clean acc} & \multicolumn{1}{l}{rob acc} \\
%\midrule
  DDIM+infoNCE & 83.40\% & 52.74\% & 84.18\% & 54.75\% & \textbf{85.64\%} & \textbf{56.28\%} \\
  DDIM+HNM(w embedding) & \textbf{84.20\%} & \textbf{53.19\%} & \textbf{85.71\%} & \textbf{54.92\%} & 85.29\% & 56.12\% \\
  DDIM+HNM(wo embedding) & 83.97\% & 52.89\% & 85.65\% & 54.83\% & 85.38\% & 55.95\%\\
    \bottomrule
  \end{tabular}}%
 \label{tab:ablation-contrastive}%
\end{table*}%

\subsection{Sensitivity of the strength of the contrastive loss}\label{ab-strength}
% \vspace{-0.04in}
Table \ref{tab:ablation-lambda} shows the influence of the strength of the contrastive loss. $\lambda=100k$ gives consistently better results than a smaller $\lambda=50k$ or a larger $\lambda=200k$ on robust accuracy on all settings. Moreover, we find the larger the $\lambda$ is, the better performance we get on clean accuracy when the additional data is small (50K case), while the smaller the $\lambda$ is, the better performance we get on clean accuracy when the additional data is large (1M case). %This reflects adding the contrastive loss to the diffusion process can significantly improve the sample efficiency.

% \vspace{-0.1in}
\begin{table}[ht!]
 \centering
 \caption{The performance of Contrastive-DP under different $\lambda$ values.}
%  \vskip 0.15in
\resizebox{0.6\textwidth}{!}{
  \begin{tabular}{lrrrrrr}
\toprule          & \multicolumn{2}{c}{50K} & \multicolumn{2}{c}{200K} & \multicolumn{2}{c}{1M} \\
\midrule
     & \multicolumn{1}{l}{clean acc} & \multicolumn{1}{l}{rob acc} & \multicolumn{1}{l}{clean acc} & \multicolumn{1}{l}{rob acc} & \multicolumn{1}{l}{clean acc} & \multicolumn{1}{l}{rob acc} \\
%\midrule
  % Baseline (DDIM) 50K & \textbf{84.37\%(0.40\%)} & 53.10\%(0.90\%) \\
  % Baseline (DDIM) 200K & \textbf{85.21\%(0.93\%)} & 55.14\%(0.21\%) \\
  % Baseline (DDIM) 1M & 85.73\%(0.51\%) & 56.77\%(0.23\%) \\
   $\lambda=50k$ & 84.41\% & 53.78\% & 85.45\% & 55.24\% & \textbf{86.35\%} & 56.83\%\\
  $\lambda=100k$ & 83.66\% & \textbf{53.91\%} & \textbf{85.71\%} & \textbf{55.79\%} & 86.30\% & \textbf{56.84\%} \\
  $\lambda=200k$ & \textbf{84.51\%} & 53.55\% & \textbf{85.51}\% & 55.33\% & 85.98\% & 56.69\%\\
    \bottomrule
  \end{tabular}}%
 \label{tab:ablation-lambda}%
\end{table}%

\subsection{Data selection for synthetic data}\label{data-selection}
Data selection methods are worthy of study since, in practice, we would like to know whether we can achieve better performance by generating a large number of samples and applying some selection criteria to filter out some samples. Therefore, we propose several data selection criterion and evaluate corresponding effectiveness in Table \ref{tab:ablation-selection}. All of the selection methods on Contrastive-DP are higher than vanilla DDIM plus selection methods, which demonstrates the superiority of using the contrastive learning loss as the guidance rather than using selection methods on the images generated by the vanilla diffusion model.

% \vspace{-0.1in}
\begin{table*}[ht!]
 \centering
 % \vspace{-0.1in}
 \caption{Comparison of different data selection criteria. The detailed explanation of each selection method can be found in Append \ref{data-selection}.}.
%  \vskip 0.15in

 \resizebox{0.8\textwidth}{!}{
  \begin{tabular}{lrrrrrr}
\toprule          & \multicolumn{2}{c}{50K} & \multicolumn{2}{c}{200K} & \multicolumn{2}{c}{1M} \\
\midrule
     & \multicolumn{1}{l}{clean acc} & \multicolumn{1}{l}{rob acc} & \multicolumn{1}{l}{clean acc} & \multicolumn{1}{l}{rob acc} & \multicolumn{1}{l}{clean acc} & \multicolumn{1}{l}{rob acc} \\
%\midrule
  DDIM (Separability) & 79.93\% &	49.49\% & 85.09\% &	54.90\% & 84.87\% &	56.08\%\\
  Contrastive-DP (Gradient norm) & 80.41\% &	49.47\% & 84.64\% &	 55.17\% & \textbf{86.36}\%	& 57.11\%\\
  Contrastive-DP (Gradient norm-rob) & 83.91\% & \textbf{55.23\%} & 84.78\% & 55.42\% & 85.93\% & \textbf{57.18\%}\\
  Contrastive-DP (Entropy) & \textbf{83.66\%} & 53.91\% & \textbf{85.71\%} & \textbf{55.79\%} & 86.30\% & 56.84\%\\
    \bottomrule
  \end{tabular}}%
 \label{tab:ablation-selection}%
\end{table*}%

Below we summarize different data selection methods:
\begin{itemize}[leftmargin=*]
  \item DDIM (Separability): We adopt the separability of the data as a criterion to make the selection of the data generated by vanilla DDIM. For each data, we use a pre-trained WRN-28-10 model to encode them into the embedding space. Then, we compute the L2 distance between each sample and the centroid of all classes (which is easily computed as the mean of all samples in this class) and add them together. To select a subset of samples that are most distinguishable, we choose the top K samples that have the smallest distance in each class.
  \item Contrastive-DP (Gradient norm): We use the gradient norm with respect to a pre-trained WRN-28-10 model as a criterion to make the selection on the data generated by Contrastive-DP. The larger the gradient norm is, the more informative the sample is for learning a downstream model. Therefore, we select the top $K$ samples that have the largest gradient norm in each class.
  \item Contrastive-DP (Gradient norm-rob): Similar to Contrastive-DP (Gradient norm), we use the gradient norm of the robust loss rather than standard classification loss as a criterion to make the selection on the data generated by Contrastive-DP. Therefore, we select the top $K$ samples that have the largest gradient norm in each class.
  \item Contrastive-DP (Entropy): We use the entropy of each sample with respect to LaNet as a criterion to make the selection on the data generated by Contrastive-DP. The smaller the entropy is, the higher likelihood this image has good quality. Therefore, we select the top $K$ samples that have the smallest entropy in each class.

\end{itemize}

\section{Additional experiments}
% \subsection{Experiments on Cifar-100 dataset}

\subsection{Changing the base adversarial training algorithm}
We mainly adopt the TRADES \cite{Zhang2019TheoreticallyPT} for adversarial training on synthetic data together with real training data. A question is whether Contrastive-DP algorithm can also have good performance using vanilla adversarial training algorithm \citep{Madry2017TowardsDL}. Table \ref{tab:expAT} demonstrates Contrastive-DP also shows advantages against vanilla DDPM and DDIM by different base adversarial training algorithms.

\begin{table}[htbp]
 \centering
 \caption{The clean and adversarial accuracy on CIFAR-10 dataset. The robust accuracy is reported by AUTOATTACK \citep{Croce2020ReliableEO} with $\epsilon_\infty=8/255$ and WRN-28-10. 50k, 200k, and 1M denote the number of synthetic used for adversarial training.}
% \vskip 0.15in
  \begin{tabular}{l|rr|rr|rr}
\toprule 
      & \multicolumn{2}{c|}{50K} & \multicolumn{2}{c|}{200K} & \multicolumn{2}{c}{1M}  \\
      & \multicolumn{1}{l}{clean acc} & \multicolumn{1}{l|}{rob acc} & \multicolumn{1}{l}{clean acc} & \multicolumn{1}{l|}{rob acc} & \multicolumn{1}{l}{clean acc} & \multicolumn{1}{l}{rob acc}\\
\midrule
  DDIM  & 87.84\% & \textbf{54.97\%} & \textbf{89.19\%} & 53.79\% & 88.91\% & 55.10\%\\
  Contrastive-DP & \textbf{88.50\%} & 54.74\% & 88.26\% & \textbf{54.20\%} & \textbf{89.43\%} & \textbf{55.31\%} \\
  \midrule
  DDPM  & 88.19\% & 53.32\% & 89.21\% & 54.16\% & \textbf{89.98\%} & 54.17\% \\
  Contrastive-DP &\textbf{88.99\%} & \textbf{53.67\%} & \textbf{89.55\%} & \textbf{54.85\%} & 89.97\% & \textbf{55.82\%}  \\
   
    \bottomrule
  \end{tabular}
 \label{tab:expAT}%
\end{table}%

\subsection{Comparison with adversarial self-supervised learning}
In the main paper, we only give the comparison of Contrastive-DP with the state-of-the-art method of adversarial robustness \citep{gowal2021improving} by using the diffusion model to generate synthetic data. Since contrastive learning is also used in adversarial self-supervised learning literature \citep{Kim2020AdversarialSC,Fan2021WhenDC,Zhang2022DecoupledAC}, we give a detailed comparison with these methods in Table \ref{tab:advsef}, which also demonstrates the effectiveness of Contrastive-DP.

\begin{table}[htbp]
 \centering
 \caption{The clean and adversarial accuracy on CIFAR-10 dataset. The robust accuracy is reported by AUTOATTACK \citep{Croce2020ReliableEO} with $\epsilon_\infty=8/255$.}
% \vskip 0.15in
  \begin{tabular}{l|r}
\toprule 
& rob acc\\
\midrule
RoCL \citep{Kim2020AdversarialSC}& 47.88\% \\
AdvCL \citep{Fan2021WhenDC}& 49.77\% \\
DeACL \citep{Zhang2022DecoupledAC}& 50.39\% \\
Contrastive-DP& 59.99\% \\
    \bottomrule
  \end{tabular}
 \label{tab:advsef}%
\end{table}%

\end{document}